\documentclass{article}

\usepackage{microtype}
\usepackage{graphicx}
\usepackage{grffile}
\usepackage{array}
\usepackage{adjustbox}
\usepackage{subcaption}
\usepackage{booktabs} 

\usepackage{hyperref}
\usepackage{xspace}

\PassOptionsToPackage{numbers, compress}{natbib}
\usepackage[preprint]{neurips_2026}

\usepackage{mathtools}

\usepackage[capitalize,noabbrev]{cleveref}

\usepackage{amsmath,amsfonts,bm}

\def\eqref#1{equation~\ref{#1}}

\def\1{\bm{1}}

\DeclareMathAlphabet{\mathsfit}{\encodingdefault}{\sfdefault}{m}{sl}
\SetMathAlphabet{\mathsfit}{bold}{\encodingdefault}{\sfdefault}{bx}{n}

\usepackage[utf8]{inputenc} 
\usepackage[T1]{fontenc}    
\usepackage{url}            
\usepackage{multirow}
\usepackage{amsfonts}       
\usepackage{nicefrac}       
\usepackage{xcolor}         
\usepackage{tcolorbox}
\usepackage{amsmath,amssymb,amsthm,bm,enumitem}
\usepackage{tikz}
\usetikzlibrary{positioning,calc,arrows.meta}
\usepackage[normalem]{ulem} 
\usepackage{xfrac}
\usepackage{array}
\usepackage{adjustbox}
\usepackage{float}
\usepackage{comment}
\usepackage{MnSymbol}
\usepackage{tabularx}

\newcommand{\methodname}{O\textsuperscript{3}\xspace}

\newtcolorbox{ourtodo}[1][]{
  colback=yellow!10,
  colframe=red!50!black,
  title=TODO,
  #1
}

\theoremstyle{plain}
\newtheorem{theorem}{Theorem}[section]

\newtheorem{lemma}[theorem]{Lemma}
\newtheorem{corollary}[theorem]{Corollary}
\theoremstyle{definition}

\theoremstyle{remark}

\usepackage[disable,textsize=tiny]{todonotes}

\title{
Sample-Efficient Optimisation over the Outputs of Generative Models
}

\author{
Samuel Willis\textsuperscript{*, 1} \And
Paul Duckworth\textsuperscript{6} \And
Jack Simons\textsuperscript{6} \And
Aleksandra Kalisz\textsuperscript{6} \And
Krisztina Sinkovics\textsuperscript{6} \And
Noam Ghenassia\textsuperscript{6} \And
Shikha Surana\textsuperscript{6} \And
Henry T. Oldroyd\textsuperscript{2} \And
Alexandru I. Stere\textsuperscript{4} \And
Dragos D. Margineantu\textsuperscript{5} \And
Carl Henrik Ek\textsuperscript{1,3} \And
Henry Moss\textsuperscript{1, 2} \And
Erik Bodin\textsuperscript{*, 1,7}
}

\begin{document}
\maketitle

\renewcommand{\thefootnote}{\fnsymbol{footnote}}
\footnotetext[1]{Equal contribution.}
\renewcommand{\thefootnote}{\arabic{footnote}}
\setcounter{footnote}{0}

\vspace{-1.75em}
\hspace{3.5em}
\textsuperscript{1}University of Cambridge
\hspace{1em}
\textsuperscript{2}Lancaster University
\hspace{1em}
\textsuperscript{3}Karolinska Institutet
\hspace{1em}
\vspace{0.75em}
\\
\makebox[\textwidth][l]{\hspace{2.2em} 
\textsuperscript{4}Boeing Commercial Airplanes
\hspace{2em}
\textsuperscript{5}Boeing AI
\hspace{2em}
\textsuperscript{6}InstaDeep
\hspace{2em}
\textsuperscript{7}Monumo}
\vspace{2.5em}

\begin{abstract}
Modern generative AI models, such as diffusion and flow matching models, can sample from rich data distributions. However, many applications, especially in science and engineering, require more than drawing samples from the model distribution: they require searching within this distribution for samples that optimise task-specific criteria. In this work, we propose \methodname{} (Optimisation Over the Outputs of Generative Models), a method for sample-efficient black-box optimisation over continuous-variable diffusion and flow-matching models. \methodname{} is built around \emph{surrogate latent spaces}: low-dimensional Euclidean embeddings that can be extracted from a generative model without additional training. The resulting representations have controllable dimensionality and support the direct application of standard optimisation algorithms. We show, on image and protein design tasks, that surrogate-space optimisation finds substantially higher-scoring samples than standard sampling or optimisation in the original latent space. Our method is model- and optimiser-agnostic, incurs negligible additional cost over standard generation, and requires no retraining or fine-tuning of the generative model.
\end{abstract}

\section{Introduction}
\label{sec:introduction}

Black-box optimisation algorithms aim to find an optimum of a function whose explicit form is unknown, but which can be queried through evaluations. Their performance is strongly affected by the dimensionality of the search space and the structure of the objective function. In high-dimensional settings, one way to simplify the problem is to optimise over an alternative representation of the data, for example one that is lower-dimensional or in which the objective is smoother or more stationary.

One route to such representations is to embed the data in a low-dimensional space using projections. Examples include REMBO~\citep{wang2016bayesian}, which uses random projections, methods that select axis-aligned subspaces~\citep{eriksson2021high}, and one-dimensional search sequences~\citep{ngo2025boids}. Another route is to use representations learned by generative models. While many latent variable models are applicable in this setting, variational autoencoders (VAEs)~\citep{Kingma:2014:aevb} have been especially influential. VAEs can learn low-dimensional continuous representations amenable to optimisation; however, the local structure of the variational distribution can lead to latent regions with low data support, making it difficult to faithfully unravel the data manifold. Several works address this limitation by fine-tuning the model during optimisation~\citep{grosnit2021high,tripp2020sample,maus2022local,chu2024inversion} or by constraining the generation process directly~\citep{boyar2024latent,moss2025return}.

In recent years, generative models have undergone a step-change in performance. Modern stochastic interpolation-based generative models~\cite{albergo2025stochastic}, such as diffusion and flow-matching models, provide a more general framework and substantially improved generative capabilities. However, unlike VAEs, these models do not directly provide a low-dimensional latent space suitable for optimisation. Existing optimisation approaches for such models therefore typically modify or condition the generative process itself. This includes training task-specific models from scratch~\citep{krishnamoorthy2023diffusion}, fine-tuning pre-trained models for a given objective~\citep{fan2023reinforcement,denker2025iterative}, requiring gradients, or learning auxiliary models of the objective function~\citep{gruver2023protein,yun2025posterior,oliveira2025generative,steinberg2024variational,yuan2024paretoflow,yao2024proud}. While these approaches demonstrate that stochastic interpolant models can be used for optimisation, they are typically task-specific and require either training the generative model or learning an additional auxiliary model, limiting their general applicability.

In this paper, we introduce a general framework for extracting optimisation-suitable latent representations from diffusion and flow-matching models. The approach is both model- and optimiser-agnostic, supports expensive gradient-free objectives, and incurs negligible additional cost. These search spaces are defined by example latents and mapped to valid model outputs, allowing standard optimisers to operate directly over the outputs of stochastic interpolation-based generative models. This enables highly sample-efficient optimisation without retraining the generative model.

\begin{figure*}[ht]
  \centering
  
  \begin{minipage}[b]{\textwidth}
    \centering
    {\footnotesize \textit{``A photo of two astronauts on the moon playing badminton while drinking tea''}}
    \vspace{0.4em}
  \end{minipage}
  
  \begin{minipage}[b]{\textwidth}
    \hspace{1.45em}
    \begin{minipage}[b]{0.66\textwidth}
      \centering
      {\tiny Seeds defining $\mathcal{U}^2$}\par\vspace{-0.8em}
      \rule{\linewidth}{0.4pt}
    \end{minipage}
    \begin{minipage}[b]{0.33\textwidth}\centering\end{minipage}
  \end{minipage}
  
  \vspace{0.2em}
  \begin{minipage}[b]{0.245\textwidth}
    \centering
    {\small $\bm{x}_1$}\\[0.2em]
    \begin{tikzpicture}[remember picture]
      \node[inner sep=0] (seed1)
        {\includegraphics[width=0.9\textwidth]{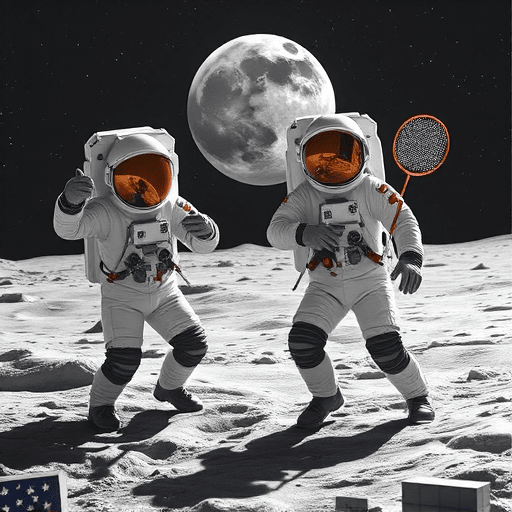}};
    \end{tikzpicture}\\[-0.2em]
    {\scriptsize Score: 1.31}
  \end{minipage}
  \begin{minipage}[b]{0.245\textwidth}
    \centering
    {\small $\bm{x}_2$}\\[0.2em]
    \begin{tikzpicture}[remember picture]
      \node[inner sep=0] (seed2)
        {\includegraphics[width=0.9\textwidth]{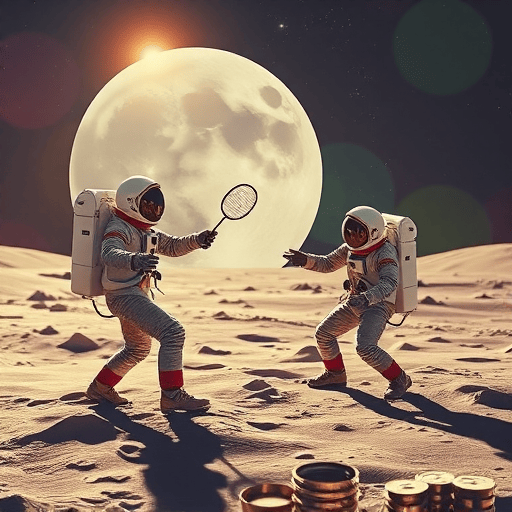}};
    \end{tikzpicture}\\[-0.2em]
    {\scriptsize Score: 1.61}
  \end{minipage}
  \begin{minipage}[b]{0.245\textwidth}
    \centering
    {\small $\bm{x}_3$}\\[0.2em]
    \begin{tikzpicture}[remember picture]
      \node[inner sep=0] (seed3)
        {\includegraphics[width=0.9\textwidth]{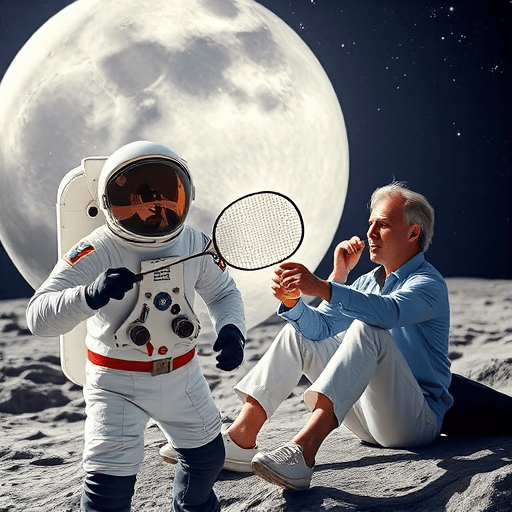}};
    \end{tikzpicture}\\[-0.2em]
    {\scriptsize Score: 1.64}
  \end{minipage}
  \begin{minipage}[b]{0.245\textwidth}
    \centering
    {\scriptsize Highest scoring image in $\mathcal{U}^2$}\\[0.2em]
    \begin{tikzpicture}[remember picture]
      \node[inner sep=0] (bestimg)
        {\includegraphics[width=0.9\textwidth]{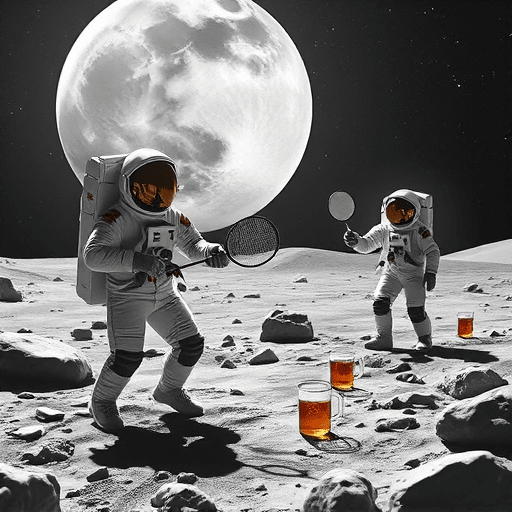}};
    \end{tikzpicture}\\[-0.2em]
    {\scriptsize Score: 1.96}
  \end{minipage}
  
  \vskip\baselineskip
  
  \def\axgap{4pt}        
  \def\axlblgapx{2pt}    
  \def\axlblgapy{6pt}    
  \def\arrowpad{2pt}
  \def\axlbl{\scriptsize}
  
  \hspace*{-2em}%
  \begin{minipage}[b]{0.395\textwidth}
    \centering
    \begin{tikzpicture}[baseline=(img.south)]
      \node[anchor=south west, inner sep=0, outer sep=0] (img) at (0,0)
        {\includegraphics[width=\textwidth]{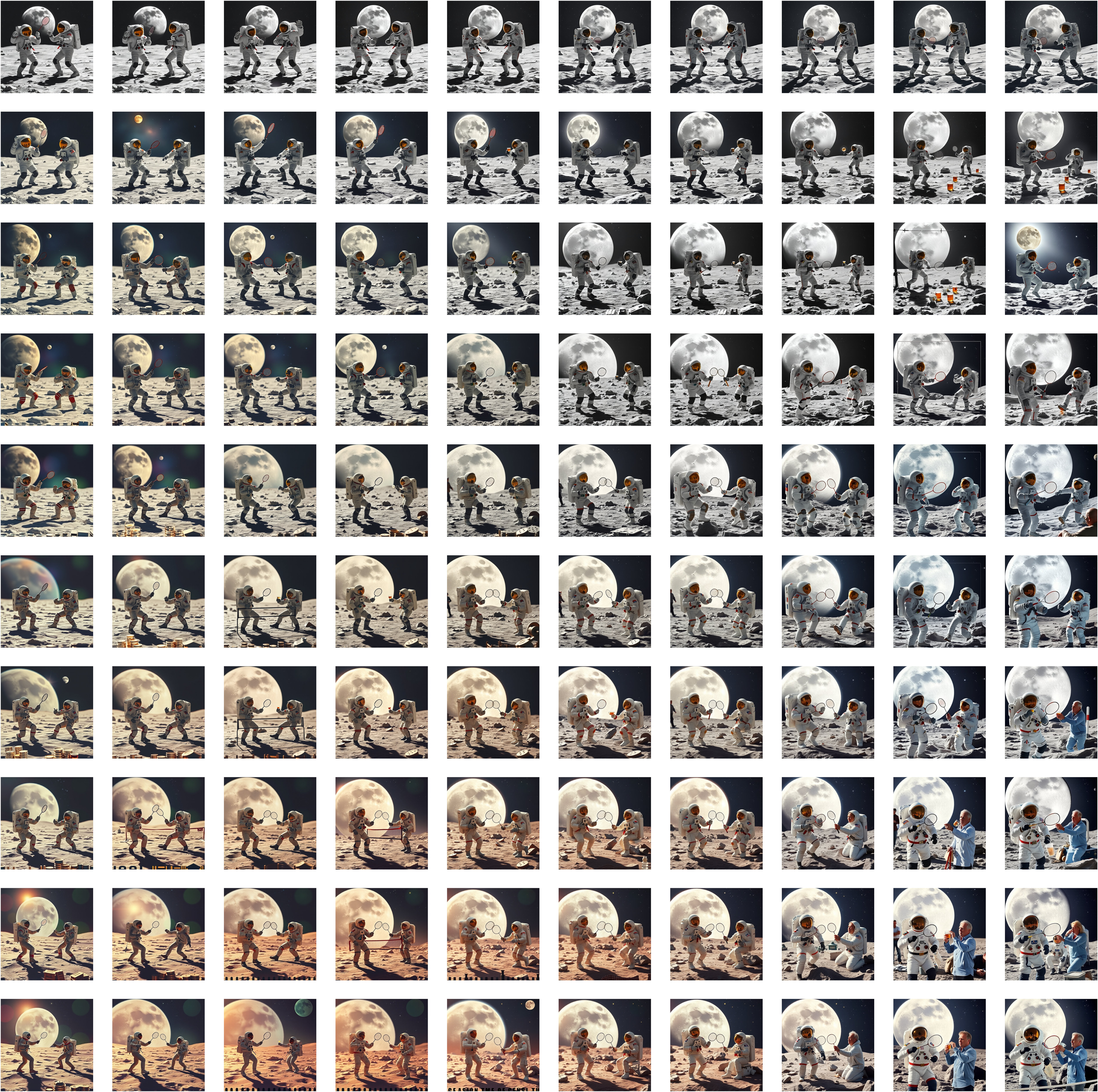}};
      \path let \p1=(img.south west), \p2=(img.north east) in
        coordinate (SW) at (\x1,\y1)
        coordinate (SE) at (\x2,\y1)
        coordinate (NW) at (\x1,\y2);
      \draw[-{Stealth}, line width=0.3pt]
        ($(SW)+(\arrowpad,-\axgap)$) -- ($(SE)+(-\arrowpad,-\axgap)$)
        node[midway, anchor=center, below=\axlblgapx] {\axlbl $u_1$};
      \draw[-{Stealth}, line width=0.3pt]
        ($(SW)+(-\axgap,\arrowpad)$) -- ($(NW)+(-\axgap,-\arrowpad)$)
        node[pos=0.55, anchor=center, left=\axlblgapy, rotate=90] {\axlbl $u_2$};
    \end{tikzpicture}
  \end{minipage}
  \hspace{1.5em}
  \begin{minipage}[b]{0.48\textwidth}
    \centering
    \begin{tikzpicture}[baseline=(img.south), remember picture]
      \node[anchor=south west, inner sep=0, outer sep=0] (img) at (0,0)
        {\includegraphics[width=\textwidth]{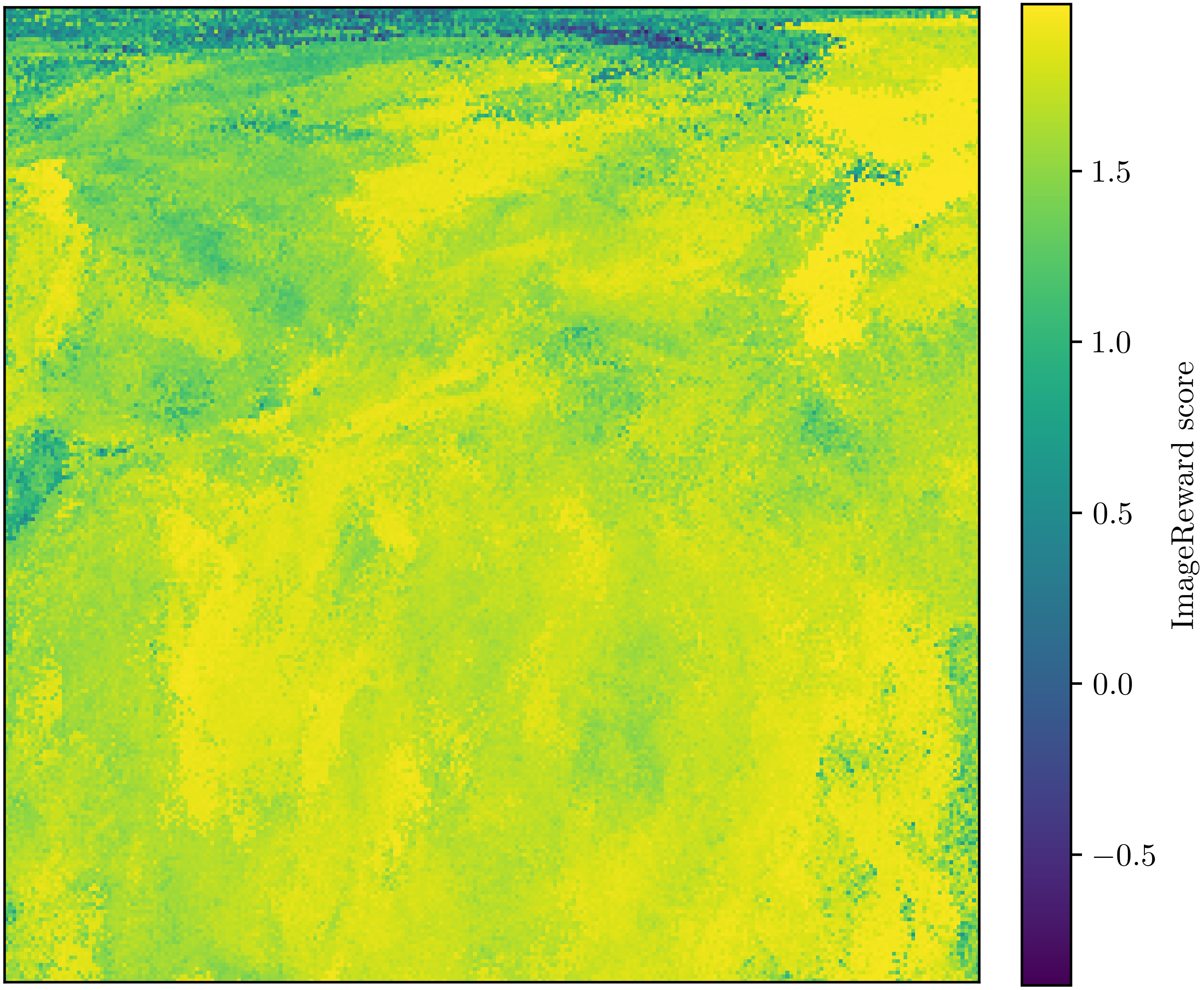}};
  
      \path let \p1=(img.south west), \p2=(img.north east) in
        coordinate (SW) at (\x1,\y1)
        coordinate (SE) at (\x2,\y1)
        coordinate (NW) at (\x1,\y2);
  
      \def\cbarfrac{0.18}
      \pgfmathsetmacro{\plotfrac}{1 - \cbarfrac}
      \coordinate (SEsq) at ($(SW)!\plotfrac!(SE)$);
  
      \draw[-{Stealth}, line width=0.3pt]
        ($(SW)+(\arrowpad,-\axgap)$) -- ($(SEsq)+(-\arrowpad,-\axgap)$)
        node[midway, below=\axlblgapx] {\axlbl $u_1$};
      \draw[-{Stealth}, line width=0.3pt]
        ($(SW)+(-\axgap,\arrowpad)$) -- ($(NW)+(-\axgap,-\arrowpad)$)
        node[pos=0.55, left=\axlblgapy, rotate=90] {\axlbl $u_2$};
  
       \newcommand{\PlaceUV}[3]{%
        \coordinate (#1@x) at ($(SW)!#2!(SEsq)$);
        \coordinate (#1)    at ($(#1@x)!#3!(#1@x|-NW)$);
      }
      \PlaceUV{u_seed1}{0.99}{0.995}  
      \PlaceUV{u_seed2}{0.002}{0.01}  
      \PlaceUV{u_seed3}{0.99}{0.01}  
      \PlaceUV{u_best}{0.93}{0.82}  
  
    \end{tikzpicture}
  \end{minipage}

  \begin{tikzpicture}[overlay, remember picture]
    \tikzset{seedarrow/.style={-{Stealth[length=4pt,width=5pt]}, line width=0.4pt}}
    \coordinate (seed1start) at ($ (seed1.south west)!0.9!(seed1.south east) + (0pt,-4pt) $);
    \coordinate (seed2start) at ($ (seed2.south west)!0.9!(seed2.south east) + (0pt,-4pt) $);
    \coordinate (seed3start) at ($ (seed3.south west)!0.9!(seed3.south east) + (0pt,-4pt) $);
    \coordinate (beststart)  at ($ (bestimg.south west)!0.9!(bestimg.south east) + (0pt,-4pt) $);
    \draw[seedarrow] (seed1start) .. controls +(-0.3,-1.0) and +(-1.0, 1.0) .. (u_seed1);
    \draw[seedarrow] (seed2start) .. controls +( 0.0,-4.0) and +(-0.6, 0.8) .. (u_seed2);
    \draw[seedarrow] (seed3start) .. controls +( 0.2,-1.0) and +(-0.8, 0.8) .. (u_seed3);
    \draw[seedarrow] (beststart)  .. controls +( 0.3,-1.0) and +(-0.4, 0.6) .. (u_best);
  \end{tikzpicture}
  
  \caption{
  \textbf{Surrogate latent spaces}. (\textit{bottom left}) Generations associated with a grid over a 2-dimensional surrogate latent space, formed using the latent vectors corresponding to the examples $\bm{x}_1$, $\bm{x}_2$, and $\bm{x}_3$ for the FLUX.1-schnell~\citep{flux2024} rectified flow model. (\textit{bottom right}) The ImageReward score for a target prompt (the objective function) over a dense grid ($256 \times 256$) of generations from our surrogate space show rich structure that can be exploited by standard optimisers. The example (`seed') images were obtained by sampling the model using the target prompt but they fail to follow it; by navigating our surrogate space, we can find images with better alignment. 
  }
  \label{fig:image_grid}
\end{figure*}

\section{Background}

Black-box optimisation targets problems where the objective function is explicitly unknown \textit{a-priori} but can be queried through function evaluations. These evaluations may be expensive, noisy, or derivative-free, and given that the function is not directly accessible such problems are referred to as ``black-box''. 
A typical black-box optimisation problem can be written as
$
\bm{x}^* = \arg\max_{\bm{x} \in \mathcal{X}} f(\bm{x}),
$
where \( f \colon \mathcal{X} \to \mathbb{R} \) is the objective function, and \( \mathcal{X} \) is the domain of interest. 

Many problems where we would like to apply black-box optimisation are characterised by high-dimensional and highly structured data, making it challenging to optimise \(f\) directly. One way to address this challenge is to optimise over a representation of the data that is better suited for search. This is commonly referred to as Latent Space Optimisation (LSO)~\citep{kusner2017grammar,gomez2018automatic,luo2018neural,lu2018structured}, and involves searching over a latent space $\mathcal{Z}$ instead of directly in $\mathcal{X}$:
\begin{equation}
    \label{eq:lso}
    \bm{z}^* = \arg\max_{\bm{z} \in \mathcal{Z}} f(g(\bm{z})),
\end{equation}
where \( g \colon \mathcal{Z} \to \mathcal{X} \) is a generative model. Traditionally, LSO has relied on low-dimensional generative latent-variable models, most prominently VAEs~\citep{Kingma:2014:aevb}.

Diffusion~\citep{ho2020denoising,song2020score,song2020denoising} and flow matching~\citep{lipman2022flow} models have been shown to learn latent representations capable of modelling highly structured, high-dimensional data. 
When sampled deterministically, a latent variable $\bm z \sim p(\bm z)$ fully specifies the generated sample $\bm x$, inducing a mapping $g:\mathcal Z\to\mathcal X$ that can in principle be used in Equation~\ref{eq:lso}.
However, the latent variables in these models typically have the same dimensionality as the data, making direct optimisation in $\mathcal Z$ challenging. A further challenge is that, until recently, it was unclear how to safely navigate these latent spaces beyond sampling.  
\citet{bodin2024linear} address this by showing that the latents are statistically structured: high-quality generations arise from latents whose statistics match typical draws from $p(\bm z)$, rather than from arbitrary points that may have high density under $p(\bm z)$ but are atypical in high dimensions. Consequently, optimisation algorithms querying the latent space must preserve these statistics to remain within the input regime where the generative model operates reliably. While~\citet{bodin2024linear} provide a framework for preserving these statistics, they do not address optimisation or how to construct effective search spaces. This is the focus of our paper.


\section{Surrogate latent spaces}
\label{sec:surrogate_latent_spaces}

Here we begin with a simplified derivation of our construction, restricted to cases where the latent distribution of the generative model is of the form $\bm{z} \sim \mathcal{N}(\bm{0}, \bm{I})$. Full details of how to apply to more general latent distributions are included in Appendix~\ref{appendix:non_gaussian}. We start from existing random subspace methods, then proceed to identify and address desirable properties required for valid latent manipulation and effective search in the black-box setting.

A standard approach in high-dimensional optimisation is to restrict the search to a low-dimensional subspace. REMBO~\citep{wang2016bayesian} follows this approach by introducing a random Gaussian embedding mapping low-dimensional ($K \ll D$) proposals $\bm{w}$ to the full space
\begin{equation}
    \bm{z}^* = \bm{Z}\bm{w}, \qquad \bm{w} \in \mathbb{R}^K,
\end{equation}
where $\bm{Z} \in \mathbb{R}^{D \times K}$ has i.i.d.\ Gaussian entries, $\bm{Z}_{ij} \sim \mathcal{N}(0,1)$, and $\bm{z}^* \in \mathbb{R}^D$ denotes the resulting proposal in the full $D$-dimensional space.

\textbf{Ensuring high-quality generations.}
For general $\bm w$, the REMBO embedding induces
\[
\bm z^*=\bm Z\bm w \sim \mathcal N(\bm 0,\|\bm w\|_2^2\bm I_D),
\]
which does not, in general, match the latent distribution $\bm z\sim\mathcal N(\bm 0,\bm I_D)$ expected by the generative model $g$ in Equation~\ref{eq:lso}. Consequently, the resulting latents may fall outside the typical input regime on which diffusion and flow-matching models operate reliably. As shown by~\citet{bodin2024linear}, such mismatched latent statistics can lead to poor or implausible generations. To avoid this, we use Latent Optimal Linear combinations (LOL)~\citep{bodin2024linear}. In the Gaussian setting, LOL reduces to,
\begin{equation}
    \bm z^*
    =
    \frac{\bm Z\bm w}{\|\bm w\|_2},
    \qquad
    \bm w\in\mathbb R^K.
\end{equation}
This construction ensures that $\bm z^*\sim\mathcal N(\bm 0,\bm I_D)$, matching the sampling statistics expected by the generative model and thereby enabling reliable generation~\citep{bodin2024linear}.

\begin{figure*}[t]
\centering

\newlength{\tile} \setlength{\tile}{0.036\textheight}
\newlength{\vgap} \setlength{\vgap}{0.002\textheight}
\newlength{\stackH}\setlength{\stackH}{\dimexpr 4\tile + 3\vgap\relax}

\newcommand{\colWide}{0.32\textwidth}   
\newcommand{\colNarrow}{0.16\textwidth} 

\def\imgGrid{compressed_figures/illustrative_figure/u_grid.png}
\def\imgOrth{compressed_figures/illustrative_figure/u_grid_kr_3d.png}
\def\noiseA{compressed_figures/illustrative_figure/new_noise.jpg}
\def\noiseB{compressed_figures/illustrative_figure/new_noise.jpg}
\def\noiseC{compressed_figures/illustrative_figure/new_noise.jpg}
\def\noiseD{compressed_figures/illustrative_figure/new_noise.jpg}
\def\astroA{compressed_figures/astronauts/seeds/0.png}
\def\astroB{compressed_figures/astronauts/seeds/1.png}
\def\astroC{compressed_figures/astronauts/seeds/2.png}
\def\astroD{compressed_figures/astronauts/best_of_10_by_10.png}

\tikzset{arr/.style={-stealth, line width=0.6pt}}
\tikzset{tight node/.style={inner sep=0,outer sep=0}}

\hspace{-0.5em}
\begin{minipage}[b]{\colWide}
  \centering
  \begin{tikzpicture}[remember picture,baseline=(img.south)]
    \node[tight node] (img) at (0,0)
      {\includegraphics[height=\stackH,width=\linewidth,keepaspectratio]{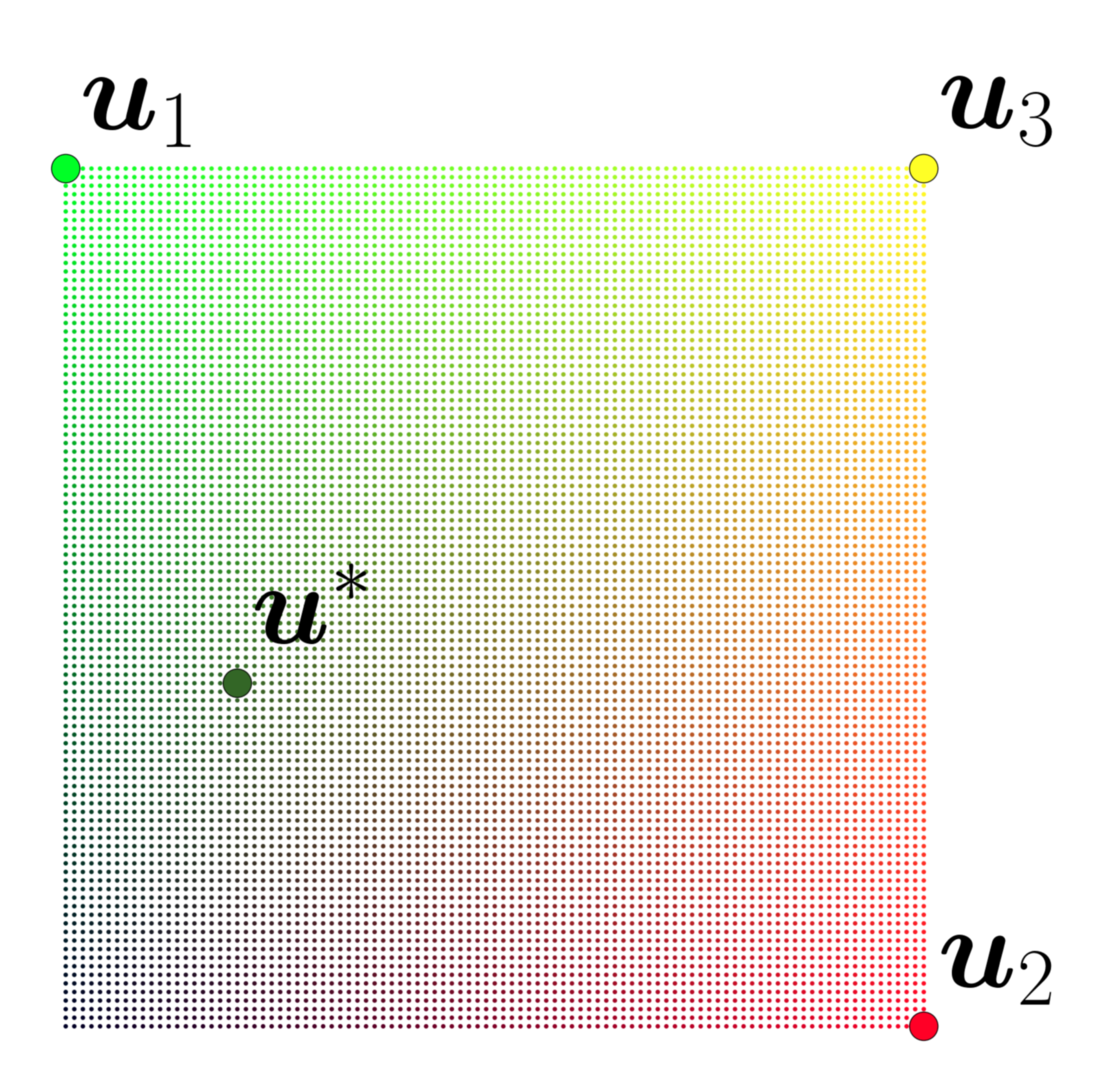}};
    \coordinate (col1E) at ([xshift=-10pt]img.east);
  \end{tikzpicture}
  \\[0.25em]  
  {\hspace{-1em}$\mathcal{U}$}
\end{minipage}\hfill
%
\hspace{-1.5em}
\begin{minipage}[b]{\colWide}
  \centering
  \begin{tikzpicture}[remember picture,baseline=(img.south)]
    \node[tight node] (img) at (0,0)
      {\includegraphics[height=\stackH,width=\linewidth,keepaspectratio]{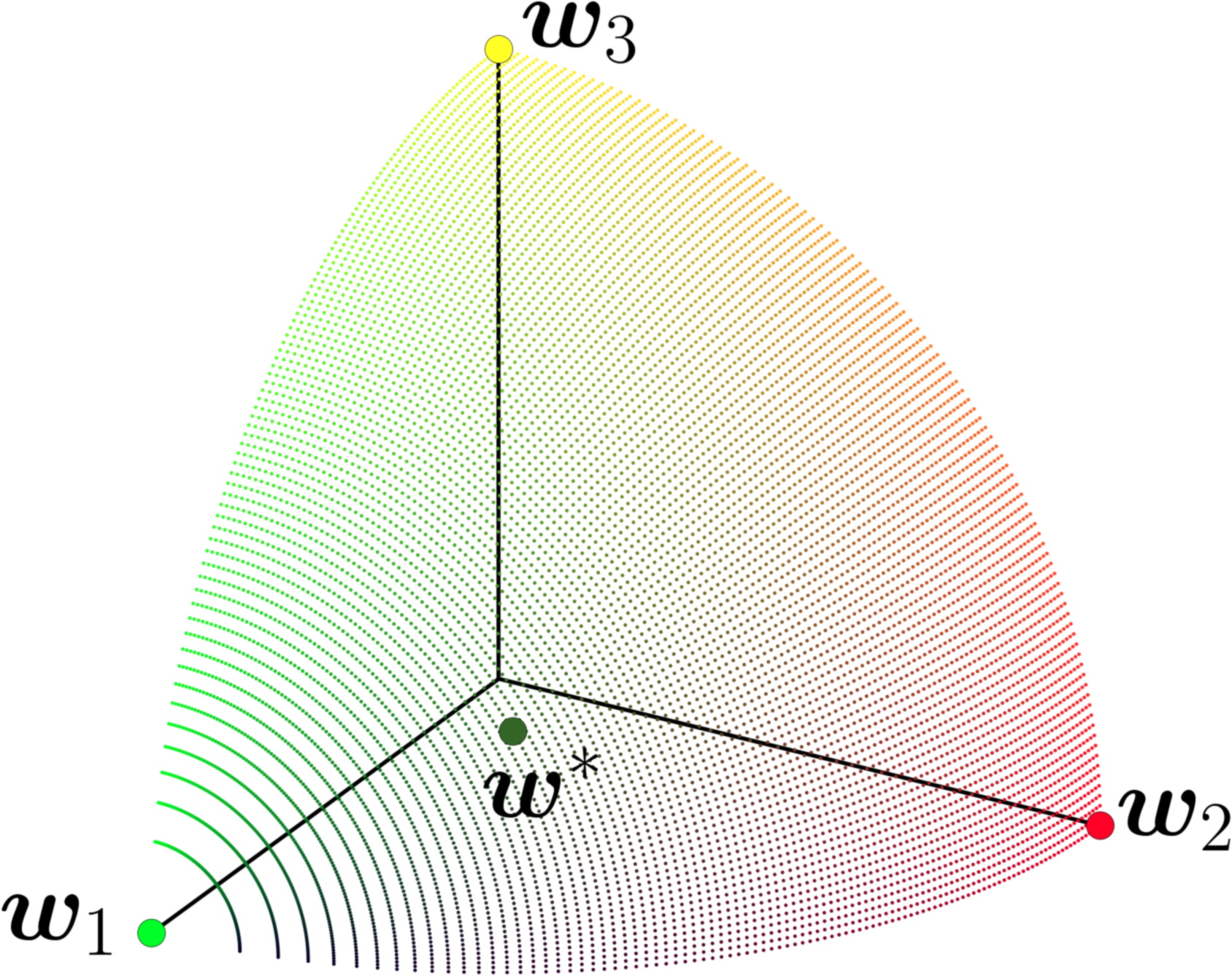}};
    \coordinate (col2W) at ([xshift=13pt]img.west);
    \coordinate (col2E) at ([xshift=-13pt]img.east);
  \end{tikzpicture}
  \\[0.25em]  
  {\hspace{0.5em}$\mathbb{S}_{+}^{K-1}$}
\end{minipage}\hfill
%
\hspace{-0.5em}
\begin{minipage}[b]{\colNarrow}
  \centering
  \begin{tikzpicture}[remember picture,baseline=(n1.south)]
    \node[tight node] (n1) at (0,0)                              {\includegraphics[height=\tile]{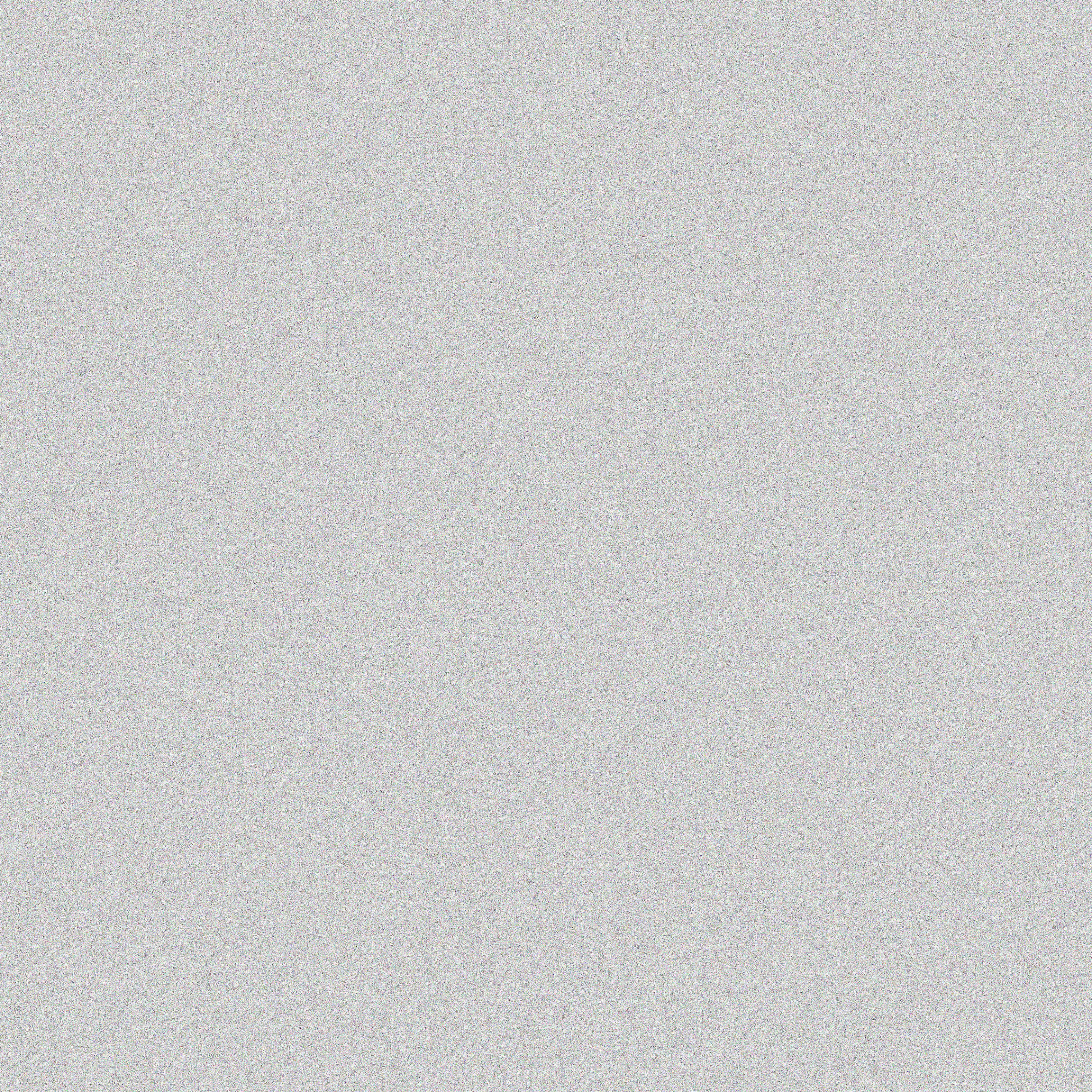}};
    \node at (n1.south) [xshift=1pt, yshift=13pt] {\fontsize{11}{9}\selectfont $\bm{z}^*$};
    \node[tight node] (n2) at (0,\dimexpr \tile+\vgap\relax)      {\includegraphics[height=\tile]{\noiseB}};
    \node at (n2.south) [xshift=0pt, yshift=10pt] {\fontsize{11}{9}\selectfont $\bm{z}_3$};
    \node[tight node] (n3) at (0,\dimexpr 2\tile+2\vgap\relax)    {\includegraphics[height=\tile]{\noiseC}};
    \node at (n3.south) [xshift=0pt, yshift=10pt] {\fontsize{11}{9}\selectfont $\bm{z}_2$};
    \node[tight node] (n4) at (0,\dimexpr 3\tile+3\vgap\relax)    {\includegraphics[height=\tile]{\noiseD}};
    \node at (n4.south) [xshift=0pt, yshift=10pt] {\fontsize{11}{9}\selectfont $\bm{z}_1$};
    \coordinate (col3W) at ($ (n1.west)!0.5!(n4.west) $);
    \coordinate (col3E) at ($ (n1.east)!0.5!(n4.east) $);
    \coordinate (col3Woff) at ([xshift=-5pt]col3W);
    \coordinate (col3Eoff) at ([xshift= 5pt]col3E);
  \end{tikzpicture}
  \\[0.25em]  
  {$\mathcal{Z}$}
\end{minipage}\hfill
%
\hspace{0.5em}
\begin{minipage}[b]{\colNarrow}
  \centering
  \begin{tikzpicture}[remember picture,baseline=(a1.south)]
    \node[tight node] (a1) at (0,0)                              {\includegraphics[height=\tile]{\astroA}};
    \node at (a1.south) [xshift=24pt, yshift=13pt] {\fontsize{11}{9}\selectfont $\bm{x}^*$};
    \node[tight node] (a2) at (0,\dimexpr \tile+\vgap\relax)      {\includegraphics[height=\tile]{\astroB}};
    \node at (a2.south) [xshift=23pt, yshift=10pt] {\fontsize{11}{9}\selectfont $\bm{x}_3$};
    \node[tight node] (a3) at (0,\dimexpr 2\tile+2\vgap\relax)    {\includegraphics[height=\tile]{\astroC}};
    \node at (a3.south) [xshift=23pt, yshift=10pt] {\fontsize{11}{9}\selectfont $\bm{x}_2$};
    \node[tight node] (a4) at (0,\dimexpr 3\tile+3\vgap\relax)    {\includegraphics[height=\tile]{\astroD}};
    \node at (a4.south) [xshift=23pt, yshift=10pt] {\fontsize{11}{9}\selectfont $\bm{x}_1$};
    \coordinate (col4W) at ($ (a1.west)!0.5!(a4.west) $);
    \coordinate (col4E) at ($ (a1.east)!0.5!(a4.east) $);
    \coordinate (col4Woff) at ([xshift=-5pt]col4W);
    \coordinate (col4Eoff) at ([xshift= 5pt]col4E);
  \end{tikzpicture}
  \\[0.25em]  
  {\hspace{-2em}$\mathcal{X}$}
\end{minipage}

\begin{tikzpicture}[remember picture,overlay]
  \draw[arr] ([yshift= 6pt]col1E) -- node[midway,above] {$\phi_w(\bm{u})$}   ([yshift= 6pt]col2W);
  \draw[arr] ([yshift=-6pt]col2W) -- node[midway,below] {$\phi_w^{-1}(\bm{w})$} ([yshift=-6pt]col1E);

  \draw[arr] ([yshift= 6pt]col2E) -- node[midway,above] {$\ell(\bm{w})$}     ([yshift= 6pt]col3Woff);
  \draw[arr] ([yshift=-6pt]col3Woff) -- node[midway,below] {$\ell^{-1}(\bm{z})$} ([yshift=-6pt]col2E);

  \draw[arr] ([yshift= 6pt]col3Eoff) -- node[midway,above] {$g(\bm{z})$}      ([yshift= 6pt]col4Woff);
  \draw[arr] ([yshift=-6pt]col4Woff) -- node[midway,below] {$g^{-1}(\bm{x})$}   ([yshift=-6pt]col3Eoff);
\end{tikzpicture}

\caption{\textbf{Illustration of a surrogate latent space.} \textit{(left)} 
A low-dimensional surrogate space $\mathcal{U}$ is mapped via the surrogate chart $\phi_w$ onto the positive orthant of the unit hypersphere \textit{(mid left)}, which is then mapped to latent variables $\bm{z}$ via $\ell$ (in the Gaussian case $\bm{z} \sim \mathcal{N}(\bm{0}, \bm{I})$, $\ell(\bm{w}) = \bm{Z}\bm{w}$; see Appendix~\ref{appendix:non_gaussian} for the general definition of $\ell$). \textit{(mid right)}, before finally being decoded by the generative model $g$ into objects $\bm x$ \textit{(right)}.}
\label{fig:one_fig_to_rule_them_all}
\end{figure*}

\textbf{Avoiding redundancy}. While this construction ensures that the latent representation have the right statistical properties for the model, it introduces redundancy in the parametrisation. In particular, all weight vectors that differ by a positive scalar factor map to the same latent,
\begin{equation}
    \frac{\bm{Z}\bm{w}}{\|\bm{w}\|_2} = \frac{\bm{Z}(\alpha \bm{w})}{\|\alpha \bm{w}\|_2}, \qquad \alpha > 0.
\end{equation}
Thus, the mapping is many-to-one, and the search space contains multiple representations of the same solution. This redundancy is undesirable for optimisation, as it introduces flat directions along which the objective remains constant, 
causing the optimiser to waste evaluations exploring equivalent parametrisations. To remove this redundancy, we constrain the weights to have unit norm,
\begin{equation}
    \bm{z}^* = \bm{Z}\bm{w},
    \qquad \bm{w} \in \left\{ \bm{w} \in \mathbb{R}^K \,:\, \|\bm{w}\|_2 = 1 \right\}
    \qquad \big(\text{i.e. } \bm{w} \in \mathbb{S}^{K-1}\big).
\end{equation}
If $\bm Z$ has full column rank, as occurs almost surely with independent Gaussian draws and $K\leq D$, then its null space is trivial. Consequently, distinct weight vectors induce distinct latents. This unit-norm constraint therefore removes the scaling redundancy while preserving the latent distribution, yielding a one-to-one, distribution-preserving parametrisation.

\textbf{Removing curvature}. While restricting $\bm{w}$ to the unit sphere removes redundancy, it introduces a new challenge: the search space is now a curved manifold $\mathbb{S}^{K-1}$. This curvature complicates optimisation, as standard black-box methods
do not naturally handle manifold constraints. Optimising directly in this space would therefore require either constrained optimisation in $\mathbb{R}^K$ or methods adapted to spherical geometry.

We are thus motivated to construct a bounded, unconstrained Euclidean space $\mathcal{U}$ in which optimisation can be performed, together with a bijection,
\begin{equation}
    \phi_w : \mathcal{U} \rightarrow \mathbb{S}^{K-1}, \qquad \mathcal{U} = [0,1]^{K-1}.
\end{equation}
This allows standard optimisation algorithms to operate in $\mathcal U$, while $\phi_w$ maps their proposals to unit-norm weight vectors $\bm{w}$; bijectivity is discussed in Appendix~\ref{appendix:non_gaussian}.

An overview of the resulting pipeline is shown in Figure~\ref{fig:one_fig_to_rule_them_all}. Here, $\ell$ denotes the map from weights $\bm w$ to latent variables $\bm z$. In the Gaussian setting considered in this section, this map reduces to $\ell(\bm w)=\bm Z\bm w$. For more general latent distributions, $\ell$ may additionally involve transport maps; see Appendix~\ref{appendix:non_gaussian} for the full derivation.


The map $\phi_w$ can be interpreted as flattening the hypersphere, analogous to constructing a map of the globe. Such a mapping cannot preserve all geometric properties and
some distortion is unavoidable~\citep{gauss1827}. The key question is therefore: how do we design $\phi_w$ to minimise this distortion?

We address this challenge in two ways. First, we restrict the domain to the positive orthant of the hypersphere, $\mathbb S_+^{K-1} \subset \mathbb S^{K-1}$. By restricting the chart to this smaller region, the resulting Euclidean parametrisation incurs less severe distortion in practice. The second part of our approach, discussed in the next section, is to choose $\phi_w$ so that Euclidean distances in $\mathcal U$ approximately preserve the similarity structure induced by the hypersphere.

\newcommand{\scoretitle}[1]{%
    \makebox[\linewidth][c]{\hspace{0.15\linewidth}{\footnotesize\scshape #1}}\\[2pt]
}

\begin{figure*}[ht]
    \centering

    \begin{minipage}[t]{0.49\textwidth}
        \centering
        \textbf{Images (Flux)}\\[3pt]
        \begin{minipage}[t]{0.49\linewidth}
            \centering
            \scoretitle{ImageReward}
            \includegraphics[width=\linewidth]{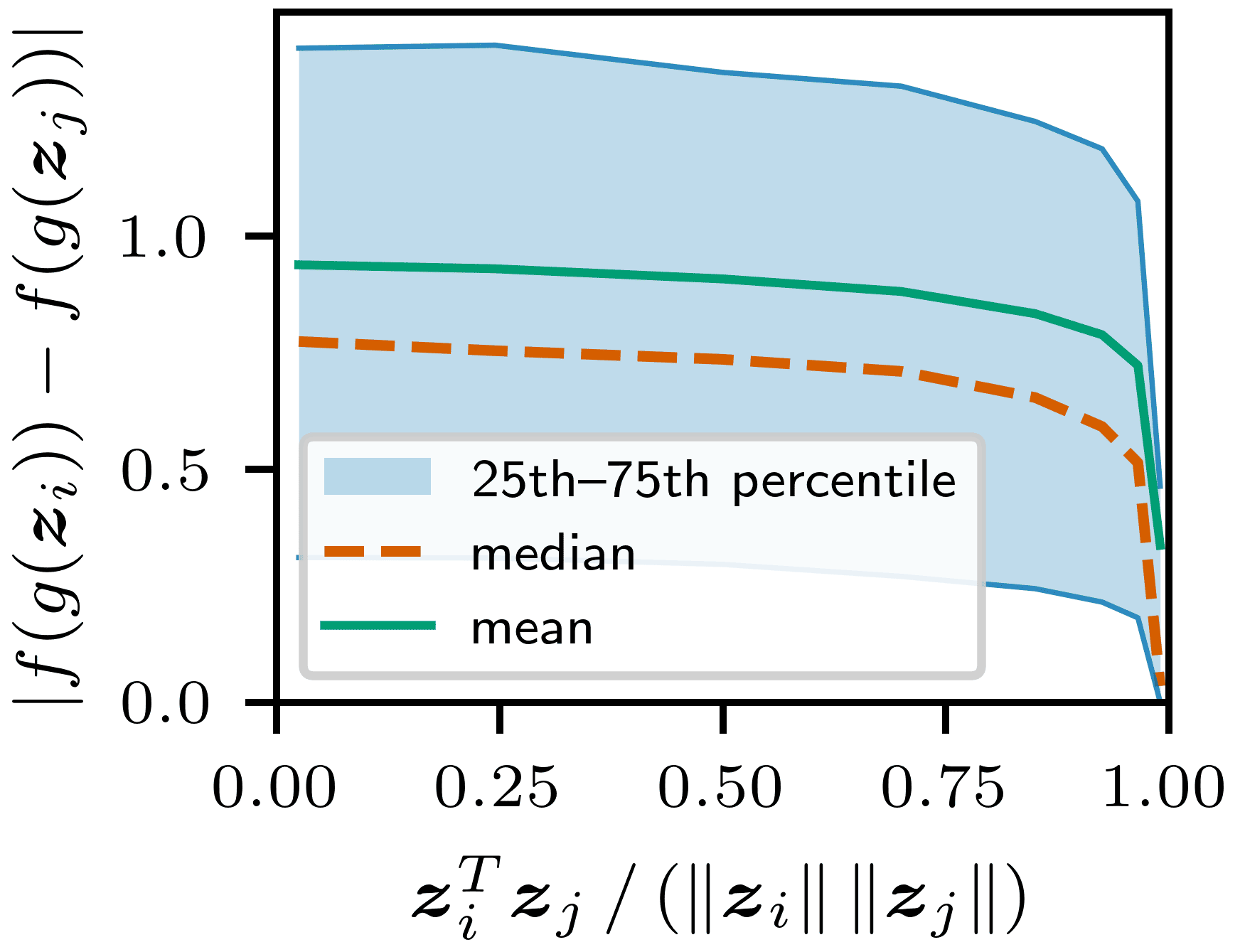}
        \end{minipage}
        \hfill
        \begin{minipage}[t]{0.49\linewidth}
            \centering
            \scoretitle{PickScore}
            \includegraphics[width=\linewidth]{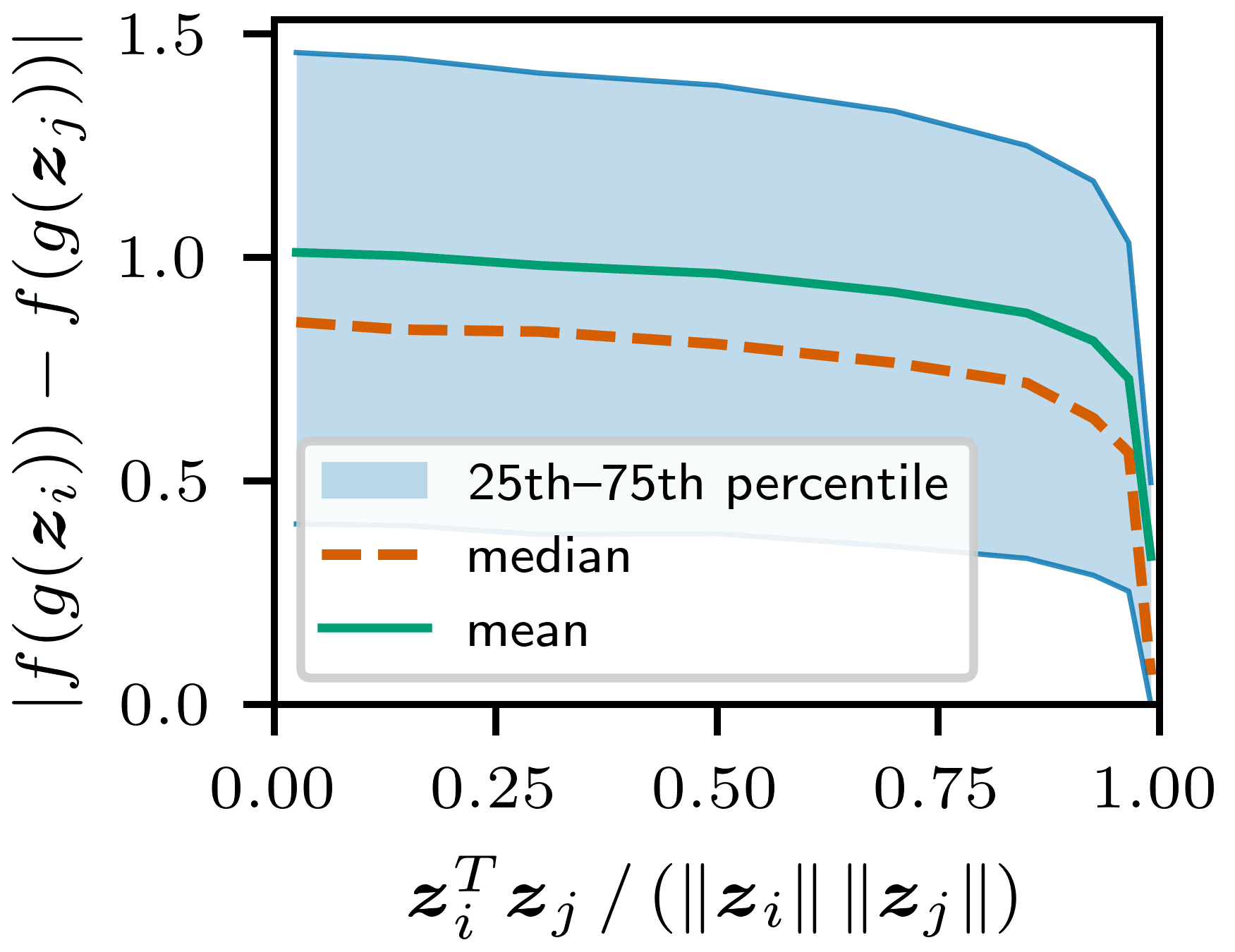}
        \end{minipage}
    \end{minipage}
    \hfill
    \begin{minipage}[t]{0.49\textwidth}
        \centering
        \textbf{Proteins (\textsc{RFdiffusion})}\\[3pt]
        \begin{minipage}[t]{0.49\linewidth}
            \centering
            \scoretitle{RMSE}
            \includegraphics[width=\linewidth]{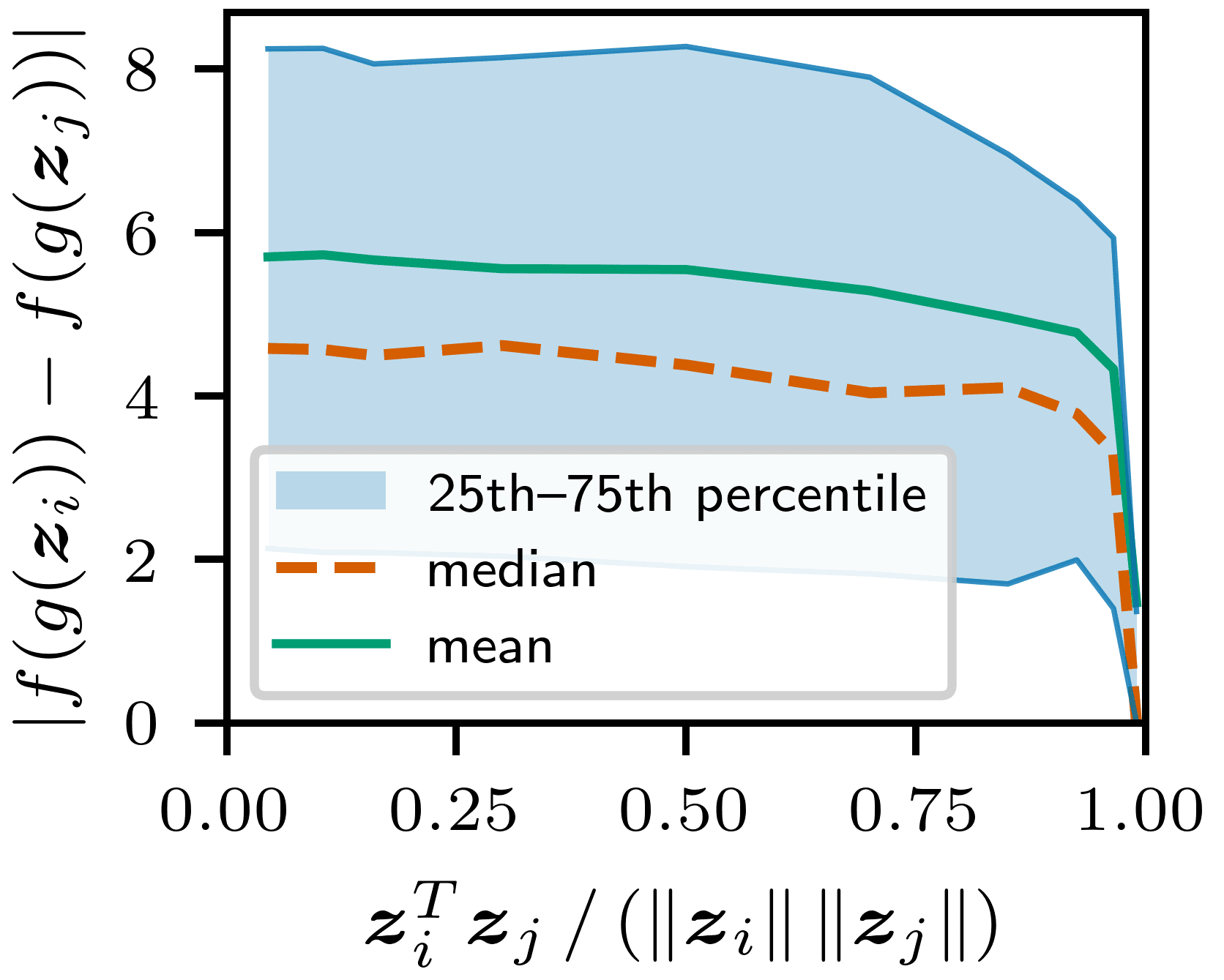}
        \end{minipage}
        \hfill
        \begin{minipage}[t]{0.49\linewidth}
            \centering
            \scoretitle{TM-Score}
            \includegraphics[width=\linewidth]{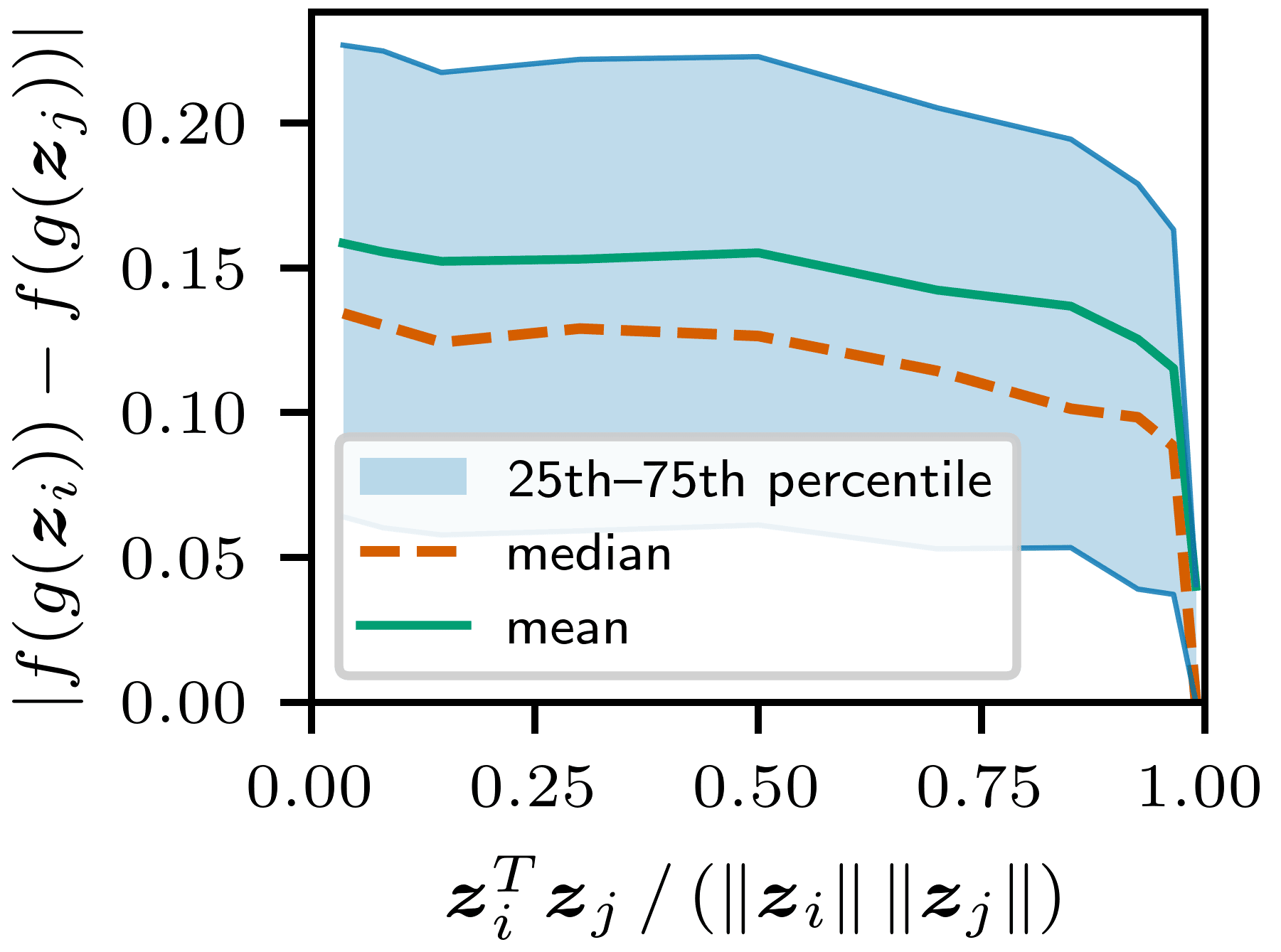}
        \end{minipage}
    \end{minipage}
    \caption{
Relationship between latent cosine similarity and objective similarity across image and protein generation tasks. The first two panels show Flux~\citep{flux2024} image generations evaluated with ImageReward and PickScore, while the last two panels show \textsc{RFdiffusion}~\citep{watson2023rfdiffusion} protein generations evaluated with RMSE and TM-score. For each metric, pairs of generated samples are grouped by the cosine similarity of their corresponding latents, and the plot shows the absolute difference in objective value between the paired samples. The black curve denotes the median, the green curve denotes the mean, and the blue shaded region denotes the 25th--75th percentile interval. Across modalities and metrics, higher latent cosine similarity corresponds to smaller differences in objective value, indicating that latent direction provides a useful notion of similarity for optimisation.
}
\label{fig:cosim_vs_objectives}
\end{figure*}

\subsection{
Hyperspherical geometry and cosine-preserving search spaces
}
The norm of samples from a high-dimensional unit Gaussian tightly concentrates around \(\sqrt{D}\).
Consequently, the variations in the latent representation is mainly encoded in the \emph{direction} of the sample.
Since the mapping \(g\) is a smooth neural network, it produces similar outputs for nearby directions. While smoothness guarantees this locally, it is not clear a priori that it extends to larger angles. However, as we show empirically in Figure~\ref{fig:cosim_vs_objectives}, cosine similarity remains smoothly related to different objectives even at larger angular separations. We demonstrate this for images and prompt-following via ImageReward~\citep{xu2023imagereward} and PickScore~\citep{kirstain2023pick} in Flux~\citep{flux2024}, and for proteins using RMSE to a reference~\citep{watson2023rfdiffusion}. 

Motivated by this, we seek mappings $\phi_w$ such that Euclidean distance in $\mathcal U$ has a monotonic relationship with the cosine similarity of the corresponding points in latent space $\mathcal Z$,
\begin{equation}
v\!\left(\|\bm u_i-\bm u_j\|_2\right)
\approx
\frac{\bm z_i^\top \bm z_j}{\|\bm z_i\|_2\|\bm z_j\|_2},
\qquad
\bm z_i=\bm Z\phi_w(\bm u_i),
\label{eq:uz_stationarity}
\end{equation}
where $v:\mathbb R_+\to[0,1]$ is a monotonically decreasing function. This induces a more regular and easier-to-optimise objective landscape over $\mathcal U$ when $f(\cdot)$ varies smoothly with cosine similarity.

In the Gaussian construction, there is a simple relationship between cosine similarity in latent space and inner products of the corresponding weight vectors. Let $\bm z_i=\bm Z\bm w_i$ and $\bm z_j=\bm Z\bm w_j$, where $\bm Z\in\mathbb R^{D\times K}$ has i.i.d.\ standard Gaussian entries and $\bm w_i,\bm w_j\in\mathbb S^{K-1}$. Then each marginal latent is distributed as $\mathcal N(\bm 0,\bm I_D)$,
\[
\mathbb E[\bm z_i^\top\bm z_j]
=
D\,\bm w_i^\top\bm w_j,
\quad\text{so}\quad
\mathbb E[D^{-1}\bm z_i^\top\bm z_j]
=
\bm w_i^\top\bm w_j.
\]
Moreover, since $D^{-1}\bm z_i^\top\bm z_j$ concentrates around its expectation and $\|\bm z_i\|_2,\|\bm z_j\|_2$ concentrate around $\sqrt D$ in high dimensions, the latent cosine similarity satisfies,
\[
\frac{\bm z_i^\top\bm z_j}{\|\bm z_i\|_2\|\bm z_j\|_2}
\approx
\bm w_i^\top\bm w_j.
\]
Thus, \eqref{eq:uz_stationarity} simplifies to
\begin{equation}
v\!\left(\|\bm u_i-\bm u_j\|_2\right)
\approx
\bm w_i^\top\bm w_j,
\qquad
\bm w_i=\phi_w(\bm u_i),
\end{equation}
for some monotonically decreasing function $v:\mathbb R_+\to[0,1]$.

In Appendix~\ref{appendix:charts}, we consider several choices for $\phi_w$ and empirically verify their suitability in Appendix~\ref{appendix:empirical_stationarity}. 
We adopt a variant of the Knothe--Rosenblatt (KR) chart, which gives a smooth bijection from $\mathcal U=[0,1]^{K-1}$ to $\mathbb S^{K-1}_+$, pushes $\bm u\sim\mathrm{Unif}(\mathcal U)$ to the uniform distribution on the hypersphere, and approximately preserves stationarity. By this, we mean that the induced weight-space similarity $\bm w_i^\top\bm w_j$ maintains an approximately monotonic relationship with the distance $\|\bm u_i-\bm u_j\|_2$.

\subsection{Seed latents}
\label{sec:seeds}

\begin{figure*}[ht]
\centering
\def\axgap{4pt}        
\def\axlblgapx{2pt}    
\def\axlblgapy{6pt}    
\def\arrowpad{2pt}
\def\axlbl{\scriptsize}

\hspace*{-2em}%
\begin{minipage}[b]{0.32\textwidth}
  \centering
  \begin{tikzpicture}[baseline=(img.south)]
    \node[anchor=south west, inner sep=0, outer sep=0] (img) at (0,0)
      {\includegraphics[width=\textwidth]{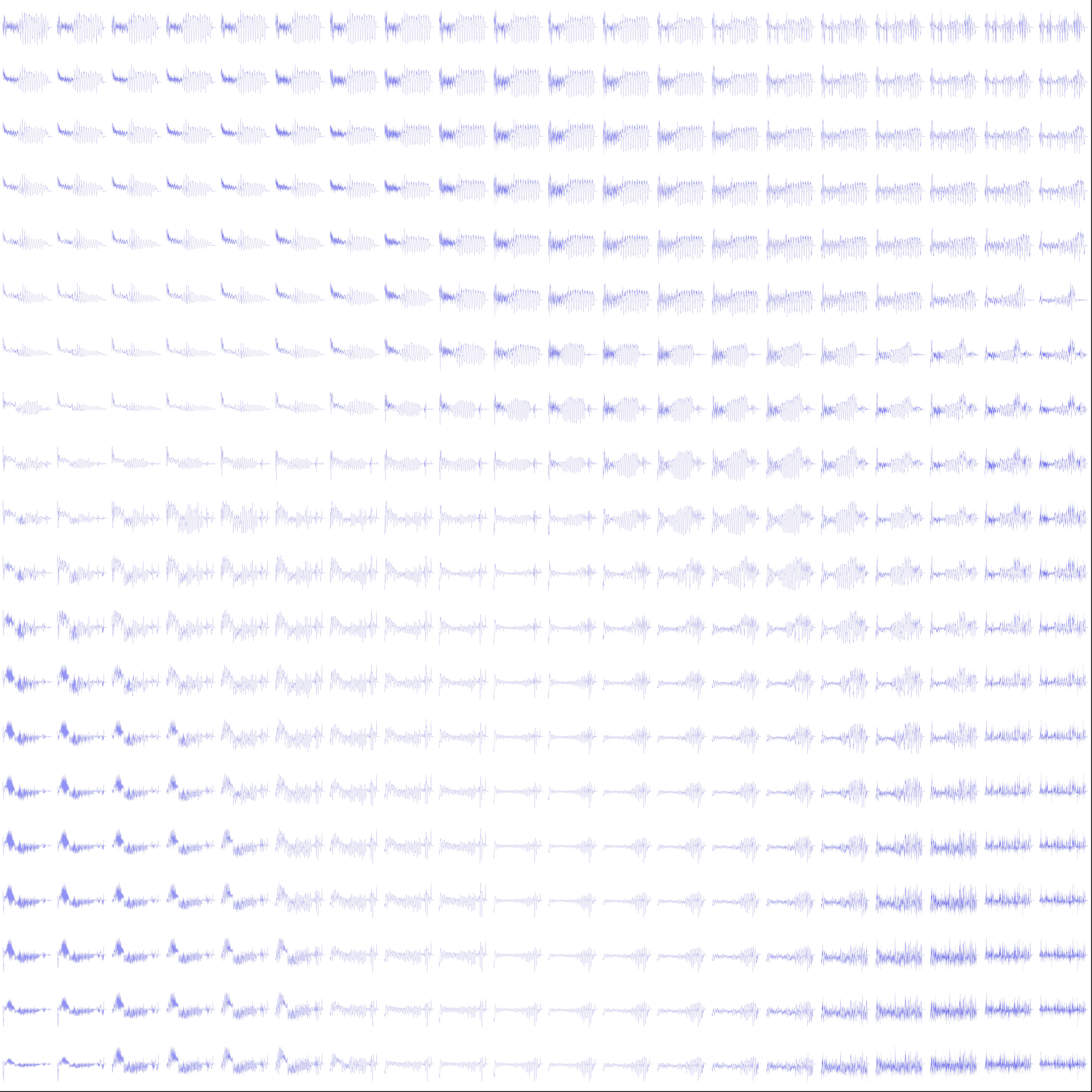}};
    \path let \p1=(img.south west), \p2=(img.north east) in
      coordinate (SW) at (\x1,\y1)
      coordinate (SE) at (\x2,\y1)
      coordinate (NW) at (\x1,\y2);
    \draw[-{Stealth}, line width=0.3pt]
      ($(SW)+(\arrowpad,-\axgap)$) -- ($(SE)+(-\arrowpad,-\axgap)$)
      node[midway, anchor=center, below=\axlblgapx] {\axlbl $u_1$};
    \draw[-{Stealth}, line width=0.3pt]
      ($(SW)+(-\axgap,\arrowpad)$) -- ($(NW)+(-\axgap,-\arrowpad)$)
      node[pos=0.55, anchor=center, left=\axlblgapy, rotate=90] {\axlbl $u_2$};
  \end{tikzpicture}
\end{minipage}
\hspace{2em}
\begin{minipage}[b]{0.18\textwidth}
  \centering
  \begin{tikzpicture}[baseline=(img.south), remember picture]
    \node[anchor=south west, inner sep=0, outer sep=0] (img) at (0,0)
      {\includegraphics[width=\textwidth]{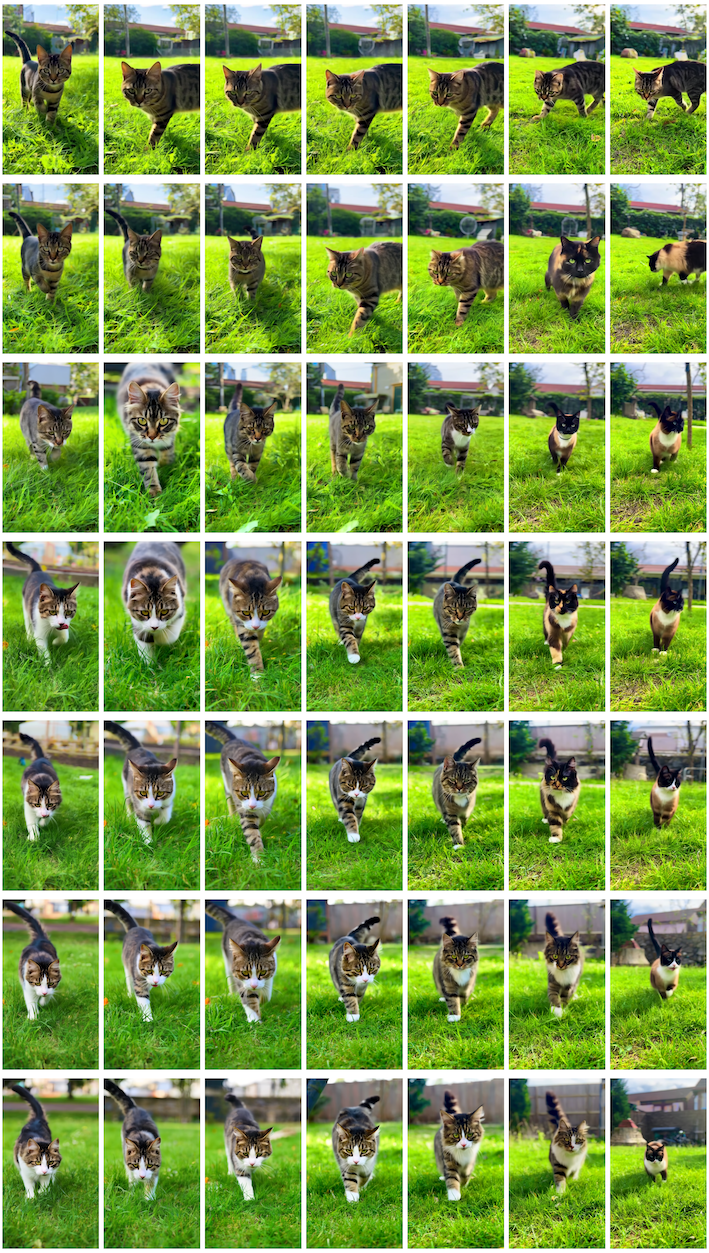}};
        \path let \p1=(img.south west), \p2=(img.north east) in
      coordinate (SW) at (\x1,\y1)
      coordinate (SE) at (\x2,\y1)
      coordinate (NW) at (\x1,\y2);
    \draw[-{Stealth}, line width=0.3pt]
      ($(SW)+(\arrowpad,-\axgap)$) -- ($(SE)+(-\arrowpad,-\axgap)$)
      node[midway, anchor=center, below=\axlblgapx] {\axlbl $u_1$};
    \draw[-{Stealth}, line width=0.3pt]
      ($(SW)+(-\axgap,\arrowpad)$) -- ($(NW)+(-\axgap,-\arrowpad)$)
      node[pos=0.55, anchor=center, left=\axlblgapy, rotate=90] {\axlbl $u_2$};
  \end{tikzpicture}
\end{minipage}
\hspace{2em}
\begin{minipage}[b]{0.32\textwidth}
  \centering
  \begin{tikzpicture}[baseline=(img.south), remember picture]
    \node[anchor=south west, inner sep=0, outer sep=0] (img) at (0,0)
      {\includegraphics[width=\textwidth]{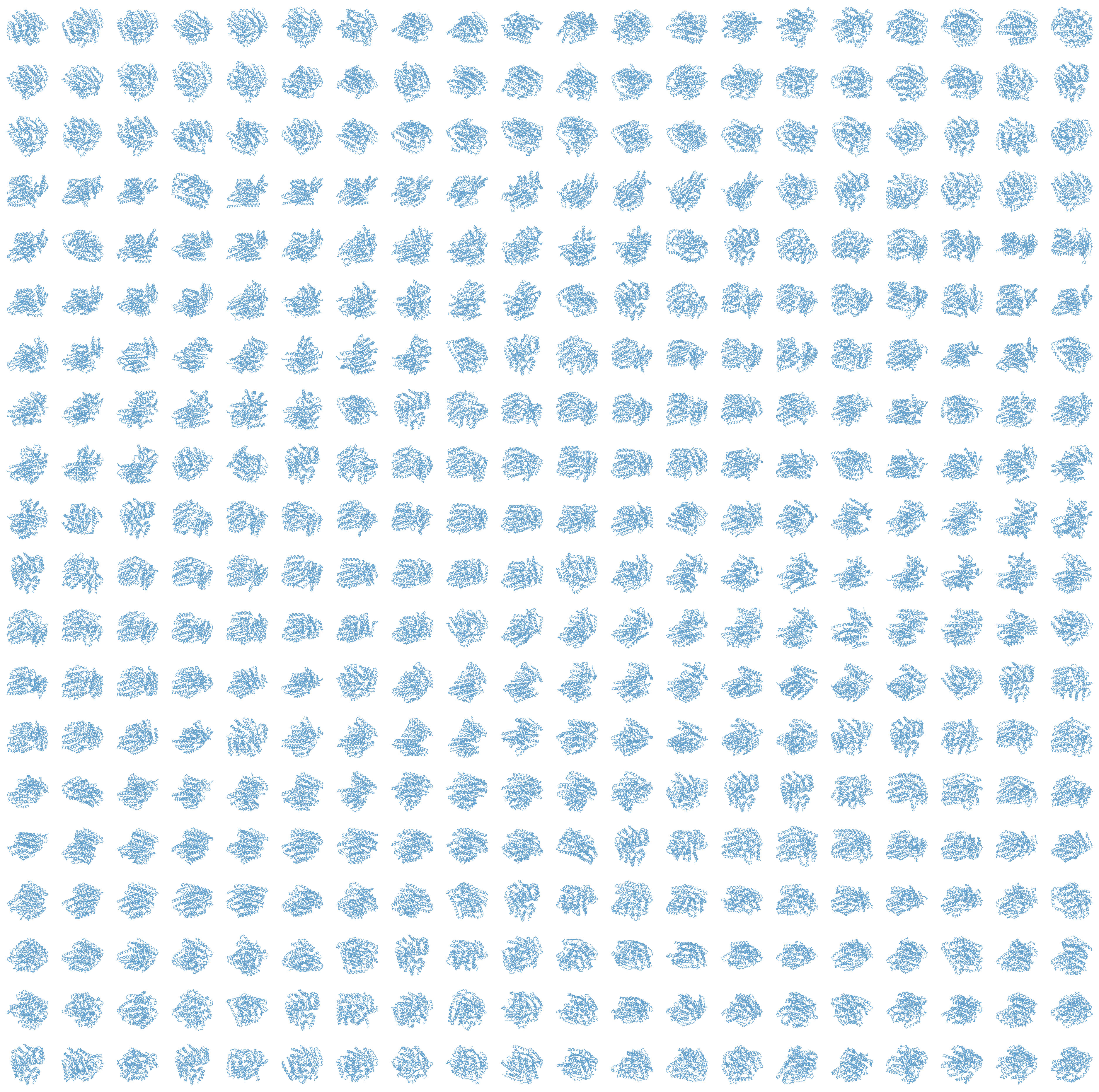}};
        \path let \p1=(img.south west), \p2=(img.north east) in
      coordinate (SW) at (\x1,\y1)
      coordinate (SE) at (\x2,\y1)
      coordinate (NW) at (\x1,\y2);
    \draw[-{Stealth}, line width=0.3pt]
      ($(SW)+(\arrowpad,-\axgap)$) -- ($(SE)+(-\arrowpad,-\axgap)$)
      node[midway, anchor=center, below=\axlblgapx] {\axlbl $u_1$};
    \draw[-{Stealth}, line width=0.3pt]
      ($(SW)+(-\axgap,\arrowpad)$) -- ($(NW)+(-\axgap,-\arrowpad)$)
      node[pos=0.55, anchor=center, left=\axlblgapy, rotate=90] {\axlbl $u_2$};
  \end{tikzpicture}
\end{minipage}

\caption{
\textbf{We can build smooth surrogate spaces for any generative model.}
(\textit{left}) Waveform generations over a grid of a 2D slice of a 7D surrogate space formed from 8 seed latents and the 8256-dimensional StableAudio2.0 text-to-audio generation model~\citep{evans2025stable}. (\textit{middle}) The first frames of a similarly constructed grid of videos from the 4,308,480-dimensional~HunyuanVideo text-to-video generation model~\citep{kong2024hunyuanvideo}. 
(\textit{right}) A grid of proteins over a 2D surrogate space formed from 3 seed latents corresponding to 3 proteins using \textsc{RFdiffusion}~\citep{watson2023rfdiffusion}.  
}
\label{fig:audio/video}
\end{figure*}

Building on the framework introduced by \cite{bodin2024linear}, we treat the columns of $\bm{Z}$ as \textit{seed latents}, where $\bm{Z} = [\bm{z}_1, \bm{z}_2, \dots, \bm{z}_K]$ are valid samples from the model. As the generative model is a deterministic map each $\bm{z}_k$ corresponds to a generated object, which means our choice of $\bm{Z}$ defines a subspace formed of $k$ generations from the model. New latent locations can now be formed as linear combinations of these seed latents.
Crucially, the framework does not restrict the choice of seed latents beyond requiring that they exhibit sampling statistics consistent with the latent distribution~\citep{bodin2024linear}. This flexibility has two important consequences. First, the number of seed latents determines the dimensionality of the optimisation problem, and thus the expressivity of the search space. Second, seed latents can be chosen to correspond to favourable samples, inducing a subspace that interpolates between desirable regions of the data distribution. 



We consider two practical strategies for selecting seed latents. The first is \textbf{filtering}, where we sample from the latent distribution and retain latents whose corresponding generations score highly under a given objective or scoring oracle. The second is \textbf{inversion}~\citep{song2020denoising,song2020score}, where we obtain latents corresponding to specific, known data points via the inverse of the generative model, $\bm{z}_i = g^{-1}(\bm{x}_i)$. Our method is agnostic to how these inversions are obtained, and can therefore make use of recent fast and accurate inversion methods, such as ExactDPM~\citep{hong2024exact} and LightningInversionEdit~\citep{samuel2023lightning}.

The choice of seed latents therefore provides a high degree of control over the induced latent span, and consequently over the surrogate search space $\mathcal{U}$.

\subsection{Summary}

In summary, for $\bm{Z} \in \mathbb{R}^{D \times K}$ formed from i.i.d.\ Gaussian entries, the constructions considered in this work are:
\[
\begin{array}{c@{\hspace{3em}}c@{\hspace{3em}}c}
\text{REMBO} & \text{LOL} & \text{\methodname} \\[4pt]
\bm{z}^* = \bm{Z}\bm{w},\; \bm{w}\in\mathbb{R}^K
&
\bm{z}^* = \frac{\bm{Z}\bm{w}}{\|\bm{w}\|_2},\; \bm{w}\in\mathbb{R}^K
&
\bm{z}^* = \bm{Z}\phi_w(\bm{u}),\; \bm{u}\in[0,1]^{K-1}
\end{array}
\]

All three differ only in how the weights $\bm{w}$ are constrained or normalised. REMBO uses unconstrained weights, leading to a mismatch with the latent distribution. LOL corrects this by normalising the linear combination weights $\bm{w}$, but induces a many-to-one parametrisation. In contrast, we directly parametrise unit-norm weights via $\phi_w$, yielding a one-to-one, distribution-preserving construction. Moreover, the map $\phi_w$ provides a bounded Euclidean search space $\mathcal{U} = [0,1]^{K-1}$, enabling the use of standard black-box optimisation algorithms. 
We refer to this framework as \methodname. 

\section{Experiments}
\label{sec:experiments}

\begin{figure*}[ht]
    \vspace{-1em}
    \centering
    \includegraphics[width=\linewidth]{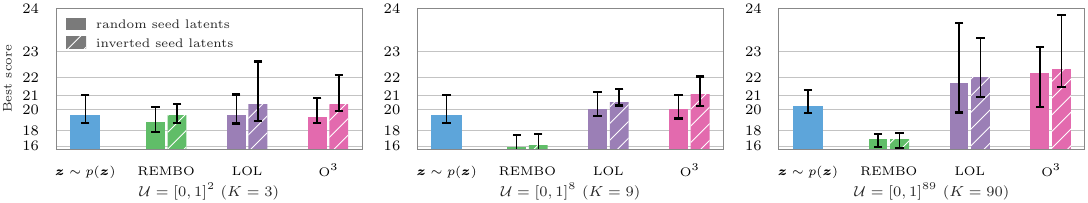}

    \caption{
    \textbf{Image optimisation with random and inversion-based seeds.}
    The three panels correspond to optimisation in surrogate spaces induced by 3, 9, and 90 seed latents, giving 2D, 8D, and 89D surrogate spaces, with optimisation budgets of $100$, $100$, and $1000$ function evaluations respectively. Each bar shows the median and 90\% confidence interval of the best PickScore achieved across $10$ target prompts, comparing Bayesian Optimisation in search spaces formed by REMBO, LOL, and \methodname{} against Best-of-$N$ --- denoted as $\bm{z} \sim p(\bm{z})$ --- from the base model.
    }
    \label{fig:ours_vs_random}
\end{figure*}

The experiments in Sections~\ref{sec:exp_data_modalities}~–~\ref{sec:exp_optimisation} are designed to verify basic properties: generality across modalities, sample quality and diversity within surrogate spaces, and the advantage of \methodname{}'s surrogate space over alternative search space constructions when navigated by standard optimisers. Sections~\ref{sec:instadeep}~–~\ref{sec:protein_opt} then benchmark \methodname{} in more realistic protein-design settings: against existing guidance methods under matched oracle budgets, and in a regime where prior optimisation approaches are infeasible.

\subsection{\methodname{} is model and data-modality agnostic}
\label{sec:exp_data_modalities}

Our methodology acts on latent variables only and applies to any continuous-variable generative model with a latent distribution amenable to interpolation (Appendix~\ref{appendix:non_gaussian}). We illustrate this on images and proteins (Figures~\ref{fig:image_grid},~\ref{fig:protein_optimisation}) and on audio and video in Figure~\ref{fig:audio/video} (see Appendix~\ref{appendix:modalities}). The formed spaces are smooth, low-dimensional, and retain full generative quality.


\subsection{Good examples define spaces with better solutions}
\label{sec:exp_dense_with_solutions}

\textbf{Goal.} This experiment isolates a key requirement of surrogate search spaces: they must preserve both sample quality and diversity in order to support effective optimisation. 

\textbf{Setup.}
We follow the benchmark of~\citet{denker2025iterative}, which evaluates generation from a diffusion model under an image prompt-following objective while maintaining sample diversity. This experiment tests whether high-scoring seed latents can define surrogate spaces that contain both high-scoring and diverse solutions (generations). We first sample a pool of random latents, select the highest-scoring latents as seeds, and construct surrogate spaces from these seed latents. We then evaluate the score and diversity of generations obtained from the resulting surrogate spaces, thereby assessing what regions of output space are made accessible by the chosen seeds. As reference points, we compare against state-of-the-art methods that explicitly train or fine-tune models for the same scoring function.

\textbf{Results.}
Surrogate spaces constructed from high-scoring seed latents consistently yield generations with high mean scores while preserving substantial diversity. Their performance is comparable to methods that explicitly train or fine-tune models on the target objective, despite requiring no task-specific generative-model training. This suggests that high-scoring seed latents can define search spaces that retain access to high-quality and diverse solutions. Quantitative results and confidence intervals, together with comparisons to Importance Fine-tuning~\citep{denker2025iterative}, DPOK~\citep{fan2023reinforcement}, and Adjoint Matching~\citep{domingo2024adjoint}, are reported in Table~\ref{table:sampling_seeds} together with experimental details in Appendix~\ref{appendix:good_seeds_better_solutions_details}.

\begin{figure*}[t]
      \centering
      \includegraphics[width=0.8\textwidth]{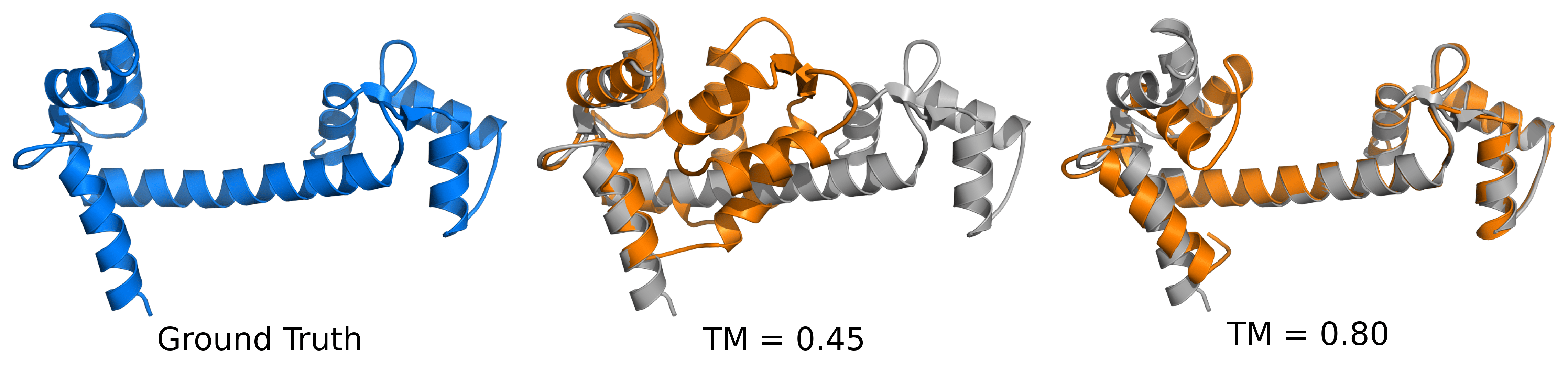}
      \includegraphics[width=\textwidth]{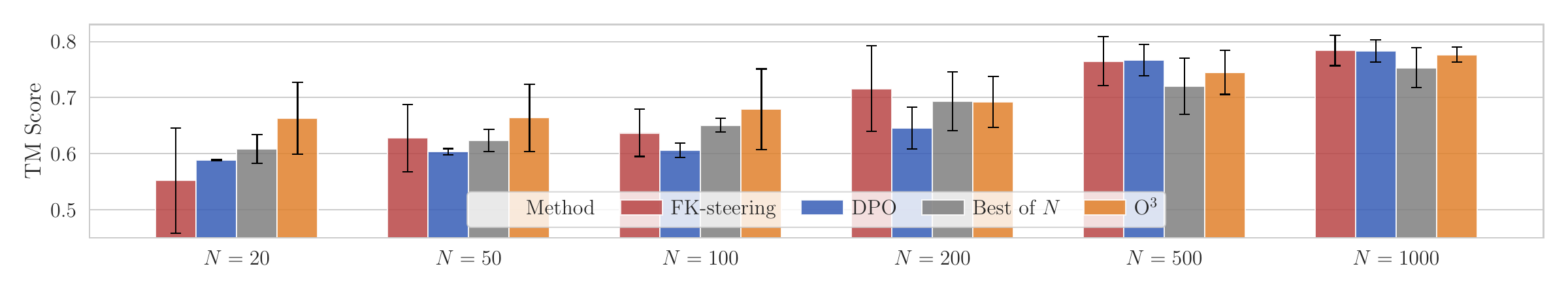}
      \caption{\textbf{Biasing protein structures of \textsc{Boltz-2} toward a desired conformation.} 
\emph{Top:} the desired 1CLL calmodulin conformation (left), a representative vanilla \textsc{Boltz-2} generation in a different conformation (middle), and a \methodname{}-guided generation better matching the desired conformation (right); ground truth shown in grey for reference in the middle and right panels. 
\emph{Bottom:} 
Comparison of \methodname{}+BO to FK-steering, DPO, and Best-of-$N$ across six budgets of $N$ oracle calls. 
Maximum TM-score to the 1CLL calmodulin structure is reported on the y-axis. Bars show the mean over $5$ repetitions, and error bars $\pm 1$ standard deviation. Higher is better. 
      }
      \label{fig:boltz_method_comparison}
\end{figure*}

\subsection{Optimisation with \methodname{} versus existing search spaces}
\label{sec:exp_optimisation}

\textbf{Goal.}
We construct a synthetic benchmark designed to be hard to solve by sampling alone: the model receives a generic prompt while being scored against a much more specific hidden target. The model is conditioned on a generic prompt (\texttt{`A vehicle'}), while the objective rewards prompt-following against a target sampled from a grammar of the form \texttt{`A <attribute> <vehicle type> <environment>'}. Target prompts are not provided to the methods directly; they are revealed only through the prompt-following score used as the objective.

\textbf{Setup.}
We measure prompt-following with PickScore~\citep{kirstain2023pick}. PickScore is prompt-dependent and on this task occupies a roughly 16--24 range, so absolute values are not directly meaningful; we therefore compare methods within a prompt and average across 10 target prompts. To aid interpretation, scores less than 19 typically correspond to monochrome or noise-like outputs, 20 to coherent images that miss the target, and 23 to images matching two of three target attributes. Small numeric differences in this range can correspond to substantive visual differences; see Figure~\ref{fig:image_optimisation_best_images} for examples, as well as Appendix~\ref{appendix:image_opt_details} for a full description of the experimental setup.

We compare Bayesian Optimisation in search spaces formed by REMBO, LOL, and \methodname{} against Best-of-$N$ sampling from the base model. We evaluate at three search space dimensionalities (2D, 8D, and 89D, corresponding to 3, 9, and 90 seed latents respectively). For each method, the search space is formed either from random seed latents or from seed latents obtained by inverting an image that matches exactly one property of the target prompt. We report filtered-seed results to isolate optimiser performance from seed-acquisition cost; see Section~\ref{sec:instadeep} for matched-budget comparisons.

\textbf{Results.}
Figure~\ref{fig:ours_vs_random} shows three main findings. \emph{First}, search spaces that produce samples matching the latent distribution of the model (LOL and \methodname{}) substantially outperform Best-of-$N$ sampling from the base model, while those that do not (REMBO) fail to do so --- an example of the best REMBO generation is included in Figure~\ref{fig:image_optimisation_best_images}. 
This supports the view that preserving the latent sampling statistics is what enables optimisation to succeed. We further corroborate this with CMA-ES in the native latent space $\mathcal Z$ (see Appendix~\ref{appendix:image_opt_details}), whose proposals violate the necessary statistical properties identified by~\citet{bodin2024linear}, and produce only black images. \emph{Second}, \methodname{} matches LOL in low-dimensional settings (2D, 8D) and pulls clearly ahead in the 89D search space. \emph{Third}, inversion-based seeds improve performance for both LOL and \methodname{} across all settings, though the gain shrinks as the search space grows. Since our benchmark is constructed such that suitable seeds are readily available, this represents a best case for inversion; the diminishing returns at higher dimensions suggest the search-space structure increasingly carries the optimisation, with seed latent quality contributing less.

We report several ablations in the appendices. Appendix~\ref{appendix:image_opt_details} extends the comparison to CMA-ES, both in a search space formed with \methodname{} and in the native space $\mathcal{Z}$. Appendix~\ref{appendix:weight_chart_comparison} reports results with an alternative choice of $\phi_{w}$, alongside all combinations of optimisers --- including random search in surrogate spaces. We find that in very low-dimensional surrogate spaces, random search in $\mathcal{U}$ is nearly as effective as BO and CMA-ES, since these spaces are simple to search but contain limited solutions. As dimensionality increases, BO and CMA-ES exploit the structure of the surrogate space and pull substantially ahead.

\subsection{\textsc{Boltz-2} guidance under matched oracle-call budgets}
\label{sec:instadeep}
\textbf{Goal.} Protein structure prediction models such as \textsc{Boltz-2}~\citep{passaro2025boltz} predict 3D atomic structures from amino-acid sequences, but their outputs are not always aligned with task-specific objectives, for example when a particular conformation is preferred among several stable ones. In this experiment, we bias generation toward a specific reference conformation, with TM-score~\citep{zhang2004tmscore}, a structural similarity score against a chosen reference, as our oracle.\footnote{TM-score is in range of $(0,1]$ with higher values indicating higher structural similarity. See Appendix~\ref{app:tm_score} for more details.} While TM-score is cheap to compute, the matched-budget protocol is intended as a proxy for the expensive, often non-differentiable oracles common in protein design.

\begin{figure*}[ht]
    \centering
    \vspace{-1.0em}
    \newcommand{\proteinimgwidth}{0.21\linewidth}
    
    \begin{minipage}[t]{0.98\textwidth}
    \centering
    \includegraphics[width=\proteinimgwidth]{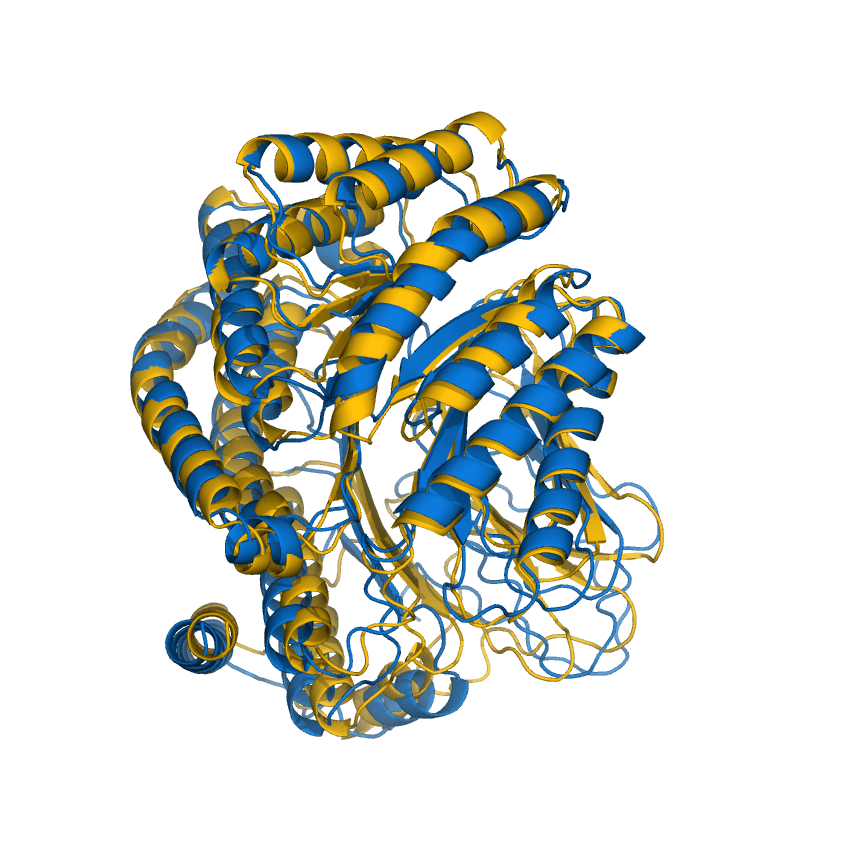}%
    \hfill
    \includegraphics[width=\proteinimgwidth]{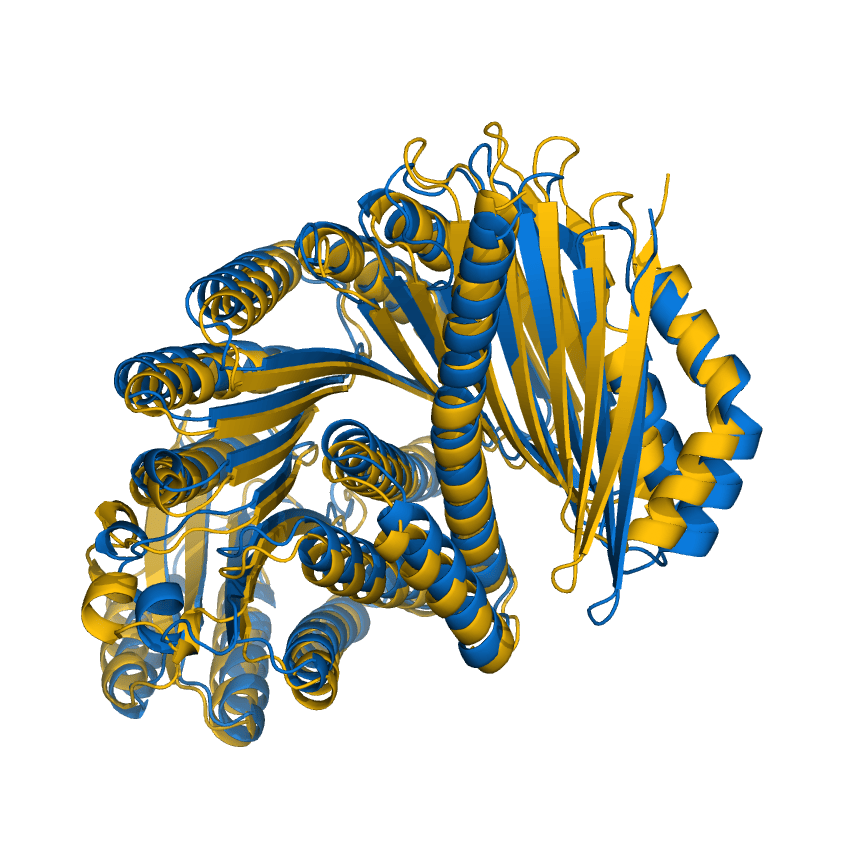}%
    \hfill
    \includegraphics[width=\proteinimgwidth]{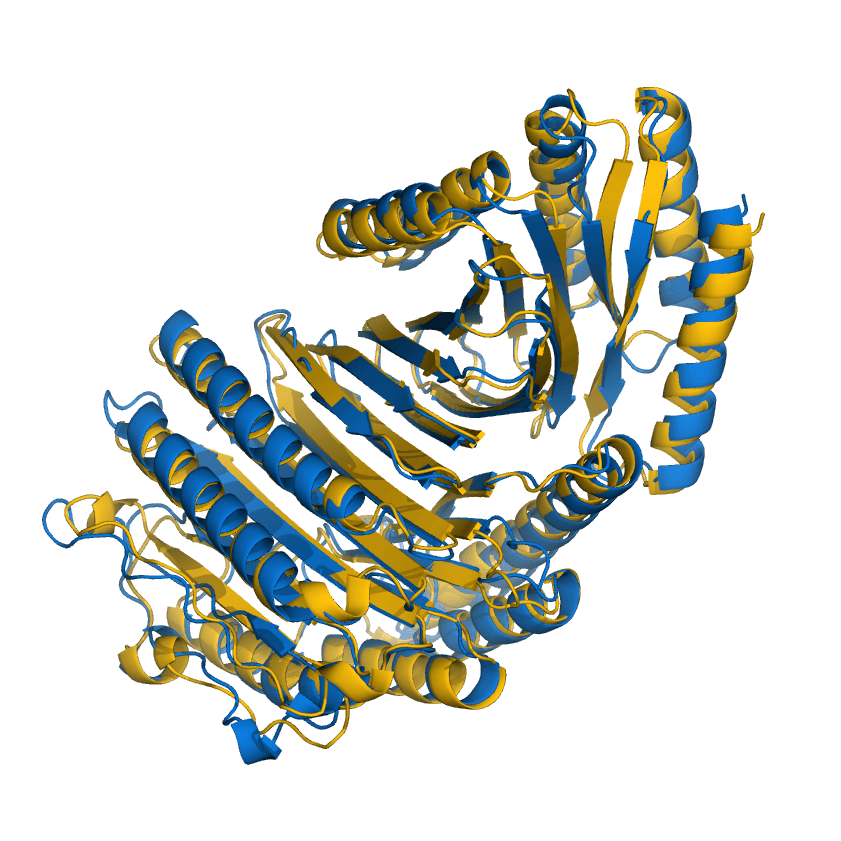}%
    \hfill
    \includegraphics[width=\proteinimgwidth]{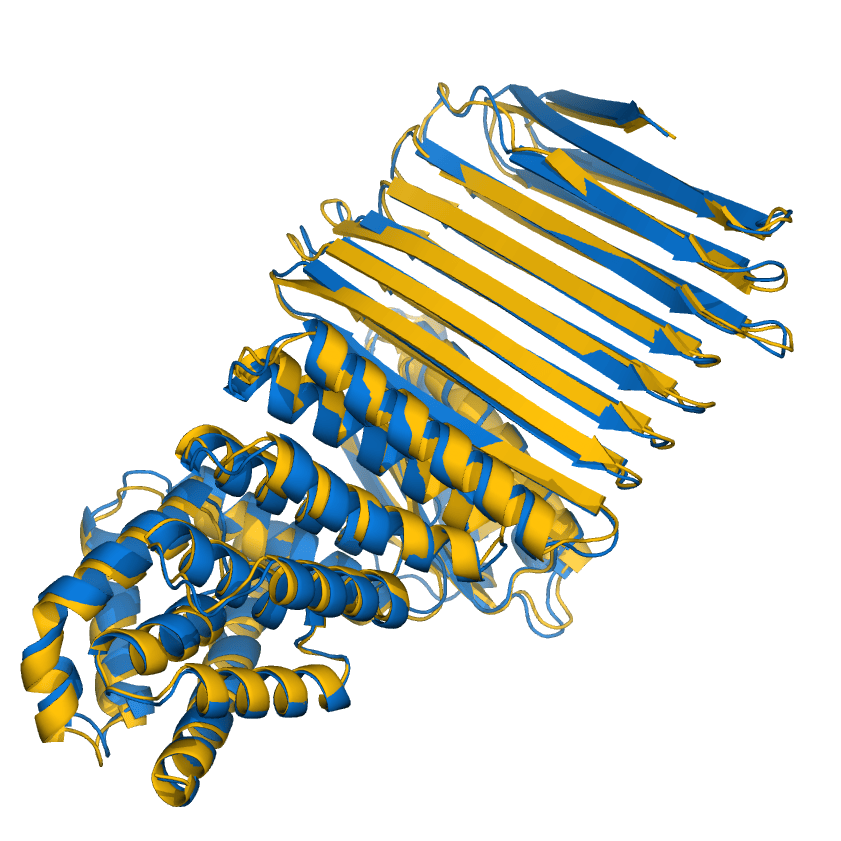}\\[-0.4em]
    
    \begin{minipage}[t]{\proteinimgwidth}\centering
        \scriptsize 5.00 RMSE
    \end{minipage}%
    \hfill
    \begin{minipage}[t]{\proteinimgwidth}\centering
        \scriptsize 2.54 RMSE
    \end{minipage}%
    \hfill
    \begin{minipage}[t]{\proteinimgwidth}\centering
        \scriptsize 1.83 RMSE
    \end{minipage}%
    \hfill
    \begin{minipage}[t]{\proteinimgwidth}\centering
        \scriptsize 1.10 RMSE
    \end{minipage}
    \end{minipage}
  \hfill
    {
     \begin{subfigure}{\textwidth}
        \centering
        \vspace{0.5em}
        \includegraphics[width=\linewidth]{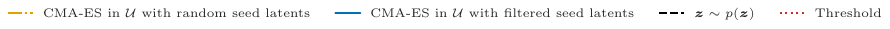}
     \end{subfigure}
    }
    \begin{subfigure}[t]{0.99\textwidth}
        \centering
        \includegraphics[width=\linewidth]{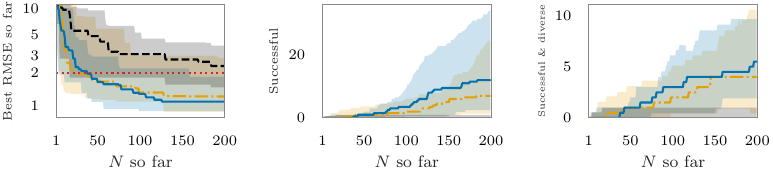}
        \label{fig:plot1}
    \end{subfigure}
    \vspace{-1em}
    \caption{
    \textbf{Protein design with surrogate latent spaces.} \textit{Top:} Representative generations showing the RMSE discrepancy between the \textsc{RFdiffusion} backbone \textit{(yellow)} and their \textsc{AlphaFold2} regeneration \textit{(blue)}. \textit{Bottom:} comparison of sample-and-filter in $\mathcal{Z}$ (Best of $N$) versus optimisation with \methodname{} using CMA–ES. 
    Plots report the median and 90\% confidence interval of the best RMSE per step as well as the number of successful and diverse designs. 
    }
    \label{fig:protein_optimisation}
\end{figure*}

\textbf{Setup.} Our reference is 1CLL~\citep{chattopadhyaya1992calmodulin}, a known conformation of calmodulin. 
We adapt \textsc{Boltz-2}'s stochastic sampler to its corresponding probability-flow ODE so that latents map deterministically to structures, as required by the \methodname{} methodology.
\methodname{} first samples a pool of $m$ structures from \textsc{Boltz-2}, scores them with the oracle, and keeps the top $K$ as seed latents defining the surrogate space $\mathcal{U}$. Bayesian optimisation then runs over $\mathcal{U}$ with the remaining $N - m$ oracle calls.
The three baselines we compare against are: Best-of-$N$, which draws $N$ samples and returns the highest-scoring one; FK-steering~\citep{singhal2025general}, which runs Sequential Monte Carlo and reweights particles by oracle scores during sampling for $N$ total oracle calls; and DPO~\citep{rafailov2023direct, wallace2023diffusion}, which fine-tunes \textsc{Boltz-2} iteratively on preference pairs formed from $N$ oracle-scored structures. Full implementation details, including the probability-flow ODE conversion, BO setup, baseline configurations, and ablations, are reported in Appendix~\ref{app:boltz}.

\textbf{Results.} \cref{fig:boltz_method_comparison} reports max TM-score against the 1CLL reference across budgets. At low budgets ($N \leq 100$), \methodname{} outperforms all baselines: at $N=100$ it achieves $\sim0.68$, against $\sim 0.65$ for Best-of-$N$ and lower for FK-steering and DPO, both of which underperform \methodname{} at these budgets and also fail to beat Best-of-$N$. At intermediate budgets, FK-steering performs best at $N=200$ and is on par with DPO at $N=500$; DPO, however, requires fine-tuning of the base generative model, which is costly. At $N=1000$, \methodname{}, FK-steering, and DPO all reach $\sim0.78$, with \methodname{} operating without any fine-tuning. Across the full budget range, \methodname{} is competitive with the strongest baseline at every $N$ and outperforms all other methods at low oracle budgets.

\subsection{Protein optimisation with \textsc{RFdiffusion}}
\label{sec:protein_opt}

\textbf{Goal.}
Finally, we apply our method to a challenging setting in protein design: generating $600$-residue proteins with \textsc{RFdiffusion}~\citep{watson2023rfdiffusion}. At this length, sampling directly from the model almost never produces designs meeting the RMSE recoverability criterion of~\citet{watson2023rfdiffusion}. The objective is expensive and non-differentiable, and the sample-and-filter strategy of~\citet{watson2023rfdiffusion} remains the only practical baseline.

\textbf{Setup.}
We compare sample-and-filter in $\mathcal{Z}$ with CMA-ES in our surrogate space $\mathcal{U}$, using 24 seeds, either random or selected by lowest RMSE from $100$ samples. The random-seed condition tests whether gains require informative seeds or follow from the search-space structure alone.

We optimise for RMSE only. The full criterion of~\citet{watson2023rfdiffusion} also requires PAE $<5$; we relax this to avoid multi-objective optimisation. Appendix~\ref{appendix:rfdiff} reports PAE, where \methodname{} substantially improves over baseline despite not optimising it. Throughout, ``successful'' means passing the RMSE filter.

\textbf{Results.}
Figure~\ref{fig:protein_optimisation} shows that after $200$ evaluations, sample-and-filter in $\mathcal Z$ often fails (median run: RMSE $=2.33$, no successful designs). CMA-ES in $\mathcal U$ reliably finds successful designs with random seed latents (median run: RMSE $=1.19$, $7$ successes across $4$ clusters), improving further with filtered seed latents (median run: RMSE $=1.08$, $12$ successes across $5.5$ clusters). Substantial gains even from random seed latents indicate that surrogate-space structure is the dominant factor.

\methodname{} also reduces PAE substantially relative to baseline (Appendix~\ref{appendix:rfdiff}), though without targeting it directly we do not reach the PAE $< 5$ threshold; multi-objective optimisation is a natural extension. 
Unlike the matched-budget protocol of Section~\ref{sec:instadeep}, here we do not charge seed-acquisition cost to the optimisation budget in the filtered-seed condition: this isolates the contribution of search-space structure from seed quality, and is often more representative of practice, where seeds are typically available from prior experiments or known designs; when they are not, the random-seed results show the search space alone delivers most of the benefit.
Full experimental details are in Appendix~\ref{appendix:rfdiff}.

\section{Conclusion}

In this paper, we introduced \emph{surrogate latent spaces}: a simple, general construction for expressive low-dimensional search spaces from high-dimensional generative models. These spaces define structured, deterministic manifolds over model outputs, making optimisation tractable for expensive or gradient-free black-box objectives. Our experiments show that surrogate latent spaces condition generative models through example-defined coordinates, with geometry suited to standard optimisation across modalities. Future work includes predicting seed-latent utility, adaptive seed selection, and human-in-the-loop objectives for creative and scientific applications.

\bibliographystyle{plainnat}
\bibliography{references}

\newpage
\appendix

\section*{Impact Statement}

\methodname{} is a general, model-agnostic optimisation layer that operates on top of pretrained continuous-variable diffusion and flow-matching models without modifying or fine-tuning them. Its societal impact is therefore largely inherited from the underlying generative models and from the objective functions that practitioners choose to optimise.

On the positive side, by enabling sample-efficient black-box optimisation without retraining, \methodname{} can lower the compute and data costs of steering generative models toward task-specific objectives. This is particularly relevant in scientific and engineering applications such as protein design, where evaluation is expensive and gradients are typically unavailable, and where reducing the number of oracle calls translates directly into reduced experimental cost. More broadly, decoupling search from model training broadens access to controllable generation for users without the resources to fine-tune large generative models.

On the negative side, any method that improves the controllability of generative models can in principle be misused to more reliably elicit harmful outputs, for example to amplify biases, generate disinformation or non-consensual content from image, audio, or video models, or to design biological sequences with harmful function. \methodname{} does not introduce new generative capabilities beyond those already present in the base model, and does not bypass safety filters or alignment mechanisms built into those models; the appropriate locus for safeguards therefore remains the base model and the choice of objective. We encourage practitioners to apply \methodname{} only on top of generative models whose intended use and safety properties are appropriate for the task, and, in high-stakes domains such as biological design, to pair it with task-specific review and screening procedures.

\section{Non-Gaussian latents and bijectivity of the surrogate chart}
\label{appendix:non_gaussian}

The construction in Section~\ref{sec:surrogate_latent_spaces} was presented for standard Gaussian latents, $\bm z \sim \mathcal N(\bm 0,\bm I)$, for which unit-norm linear combinations of independent latent samples remain distributed according to the same law. This property makes it possible to form valid proposals directly by combining seed latents with weights on the unit sphere. In this appendix, we extend the construction to latent distributions beyond the standard Gaussian case. The central requirement is not Gaussianity itself, but the existence of an ``inner'' latent representation in which unit-$\ell_2$ linear combinations can be transformed back into valid samples from the target latent distribution. When the original latent distribution does not satisfy this property directly, we introduce transport maps to and from such an inner representation, perform the linear combination there, and map the result back to the original latent space. This recovers the general Latent Optimal Linear combinations (LOL) construction of~\citet{bodin2024linear}, which we further extend with additional transport maps to broaden the class of supported latent distributions. In the Gaussian setting used throughout the main text, these transports reduce to the identity.

We also show that the resulting chart is bijective \emph{on its image}: every latent generated by the chart has a unique coordinate $\bm u \in \mathcal U$, and this coordinate can be recovered exactly. Equivalently, the surrogate space introduces no redundant parametrisations within the transported span induced by the chosen seed latents: distinct surrogate coordinates map to distinct latent proposals, and any proposal generated by the chart admits an exact inverse back to $\mathcal U$.

The key idea behind our method is to construct a search space $\mathcal U$ from a collection of $K$ seed latents, yielding a coordinate system whose generated objects inherit properties from these seeds. Let $\{\bm z_k\}_{k=1}^K$, with $\bm z_k \in \mathcal Z$, denote the seed latents. We define a surrogate chart
\[
\phi(\cdot;\{\bm z_k\}_{k=1}^K):\mathcal U \to \mathcal Z
\]
by
\begin{align}
\label{eq:surrogate_chart}
\phi(\bm u;\{\bm z_k\}_{k=1}^K) &= \bm z,
&
\bm z &:= \ell(\bm w;\{\bm z_k\}_{k=1}^K),
&
\bm w &:= \phi_w(\bm u),
\end{align}
where $\phi_w:\mathcal U\to\mathbb S_+^{K-1}$ is a bijection from the surrogate space to the positive orthant of the unit sphere, and
\[
\ell(\cdot;\{\bm z_k\}_{k=1}^K):\mathbb S_+^{K-1}\to\mathcal Z
\]
maps spherical weights to latent variables. Here,
\[
\mathbb{S}_{+}^{K - 1} := \big\{\, \bm{w} \in \mathbb{R}^K \;\big|\; \|\bm{w}\|_2 = 1,\; w_i \geq 0 \;\; \forall i \,\big\}.
\]
On its image, the chart admits the inverse
\begin{align}
\phi^{-1}(\bm z;\{\bm z_k\}_{k=1}^K)
&= \bm u,
&
\bm u
&:= \phi_w^{-1}(\bm w),
&
\bm w
&:= \ell^{-1}(\bm z;\{\bm z_k\}_{k=1}^K).
\end{align}

To ensure validity, each coordinate $\bm u\in\mathcal U$ must map to a latent realisation $\bm z\in\mathcal Z$ with the model's latent distribution $p$. We achieve this by defining $\ell$ using a transported aggregation. Let $\mathcal T_{\rightarrow}$ map model latents into an inner representation, and let $\mathcal T_{\leftarrow}$ map aggregated inner vectors back to valid samples from $p$. Given seed latents $\{\bm z_k\}_{k=1}^K$, define
\begin{align}
\label{eq:LOL}
\bm z = \ell(\bm w;\{\bm z_k\}_{k=1}^K)
&:= \mathcal T_{\leftarrow}(\bm\epsilon),
&
\bm\epsilon &:= \bm\xi\bm w,
&
\bm\xi &:= [\bm\epsilon_1,\dots,\bm\epsilon_K],
&
\bm\epsilon_k &:= \mathcal T_{\rightarrow}(\bm z_k).
\end{align}
The required validity condition is that, for independent seed latents $\bm z_k\sim p$ and any $\bm w\in\mathbb S^{K-1}$,
\[
\mathcal T_{\leftarrow}(\bm\xi\bm w)\sim p.
\]
In the Gaussian case, this holds because the inner aggregate itself has the target Gaussian law. In the hyperspherical case, the aggregate is not uniformly distributed on the sphere, but its normalised direction is.

For points generated by the forward map $\bm u\mapsto\bm w\mapsto\bm z$, the inverse can be computed exactly on the image of the chart under an additional directional-consistency condition on the transports. In particular, if $\mathcal T_{\rightarrow}(\mathcal T_{\leftarrow}(\bm\epsilon))$ is a positive scalar multiple of $\bm\epsilon$ and the seed matrix $\bm\xi$ has full column rank, then
\[
\bm w
=
\ell^{-1}(\bm z;\{\bm z_k\}_{k=1}^K)
=
\frac{\bm\xi^+ \mathcal T_{\rightarrow}(\bm z)}
{\|\bm\xi^+ \mathcal T_{\rightarrow}(\bm z)\|_2},
\]
where $\bm\xi^+$ denotes the Moore--Penrose inverse of $\bm\xi$; see Appendix~\ref{appendix:inverse}. Typically $\dim(\bm\epsilon)=\dim(\bm z)$, although this depends on the chosen transport map.

\paragraph{Inner aggregations preserving the target distribution.}
An inner representation is amenable to our methodology if unit-$\ell_2$ aggregation of independent transformed seed latents, followed by the backward transport, preserves the target latent distribution. Formally, if $\bm z_1,\dots,\bm z_K \stackrel{\mathrm{i.i.d.}}{\sim} p$, $\bm\epsilon_k=\mathcal T_{\rightarrow}(\bm z_k)$, and $\bm w\in\mathbb S^{K-1}$, then the requirement is
\[
\mathcal T_{\leftarrow}\!\left(\sum_{k=1}^K w_k\bm\epsilon_k\right)
\sim p.
\]
Equivalently, with $\bm\xi=[\bm\epsilon_1,\dots,\bm\epsilon_K]$ and $\bm\epsilon=\bm\xi\bm w$, we require
\[
\bm z=\mathcal T_{\leftarrow}(\bm\epsilon)\sim p.
\]
In some cases, such as zero-mean Gaussian inner latents, the aggregate $\bm\epsilon$ itself has a fixed inner distribution. In other cases, such as hyperspherical latents, $\bm\epsilon$ need only have a distribution whose image under $\mathcal T_{\leftarrow}$ is the desired target distribution. Specifying the maps $\mathcal T_{\rightarrow}$ and $\mathcal T_{\leftarrow}$ therefore gives a general framework for applying our method to different latent distributions:
\begin{itemize}
  \item \textbf{Gaussian latents.} If $p=\mathcal N(\bm 0,\bm\Sigma)$, then closure holds directly under $\|\bm w\|_2=1$, and we may set $\mathcal T_{\rightarrow}=\mathcal T_{\leftarrow}=\mathrm{id}$.
  If $p=\mathcal N(\bm\mu,\bm\Sigma)$ with $\bm\mu\neq\bm 0$, then
  \[
  \mathcal T_{\rightarrow}(\bm z)=\bm z-\bm\mu,
  \qquad
  \mathcal T_{\leftarrow}(\bm\epsilon)=\bm\epsilon+\bm\mu,
  \]
  which centre the distribution before aggregation and restore the mean afterwards.
  This case was treated in~\citet{bodin2024linear}.

  \item \textbf{Hyperspherical latents.} If $p=\mathrm{Unif}(\mathbb S^{D-1})$, where $\bm z\in\mathbb R^D$, then we may set $\mathcal T_{\rightarrow}=\mathrm{id}$ and let $\mathcal T_{\leftarrow}$ normalise the aggregated vector back onto the sphere,
  \[
  \mathcal T_{\leftarrow}(\bm\epsilon)=\frac{\bm\epsilon}{\|\bm\epsilon\|_2}.
  \]
  The aggregate $\bm\xi\bm w$ is not itself uniformly distributed on the sphere. However, for any fixed unit-norm $\bm w$, its distribution is rotation invariant, and therefore its normalised direction is uniformly distributed on $\mathbb S^{D-1}$.

  \item \textbf{Composite latents.} If the latent variable decomposes into $M$ statistically independent components, $\bm z=\{\bm z^{(1)},\dots,\bm z^{(M)}\}$, with
  \[
  p(\bm z)=\prod_{m=1}^M p_m(\bm z^{(m)}),
  \]
  then each component can be transported separately. This allows models with multiple latent variables, or latents with heterogeneous distributions, to be handled by applying component-wise transports and concatenating the resulting inner representations before computing linear combinations.

  \item \textbf{General scalar distributions.} If the dimensions of $\bm z$ are statistically independent scalar variables, each dimension can be transported to an amenable scalar distribution, such as $\mathcal N(0,1)$, using the probability integral transform. For example, if $F_i$ is the cumulative distribution function of $z_i$ and $\Phi$ is the standard Gaussian cumulative distribution function, one may use
  \[
  \epsilon_i=\Phi^{-1}(F_i(z_i)).
  \]
  This corresponds to the scalar transport construction proposed in~\citet{bodin2024linear}.
\end{itemize}

\section{Inverse of the $\ell$ map}
\label{appendix:inverse}

In this section, we derive the inverse of the map $\ell$ introduced in Section~\ref{appendix:non_gaussian}. The inverse is defined on the image of $\ell$, i.e.\ for latents generated by the forward map from a fixed set of seed latents.

\begin{lemma}
Let $\{\bm z_k\}_{k=1}^K \subset \mathcal Z$ be seed latents, and define their corresponding inner latents by
\[
\bm\epsilon_k = \mathcal T_{\rightarrow}(\bm z_k).
\]
Let
\[
\bm\xi = [\bm\epsilon_1,\dots,\bm\epsilon_K] \in \mathbb R^{D\times K}.
\]
Assume:
\begin{enumerate}[label=(A\arabic*)]
    \item $\bm\xi$ has full column rank, so that $\bm\xi^+\bm\xi = I_K$;
    \item for all relevant $\bm\epsilon \in \mathbb R^D$,
    \[
    \mathcal T_{\rightarrow}\!\big(\mathcal T_{\leftarrow}(\bm\epsilon)\big)
    =
    \alpha(\bm\epsilon)\,\bm\epsilon,
    \qquad
    \alpha(\bm\epsilon)>0.
    \]
\end{enumerate}
Define
\[
\ell(\bm w;\{\bm z_k\}_{k=1}^K)
=
\mathcal T_{\leftarrow}(\bm\xi\bm w),
\qquad
\bm w\in\mathbb S^{K-1}_+.
\]
Then $\ell$ is injective on $\mathbb S^{K-1}_+$ and admits the following inverse on its image:
\[
\ell^{-1}(\bm z;\{\bm z_k\}_{k=1}^K)
=
\frac{\bm\xi^+\,\mathcal T_{\rightarrow}(\bm z)}
{\|\bm\xi^+\,\mathcal T_{\rightarrow}(\bm z)\|_2}.
\]
\end{lemma}

\begin{proof}
Let
\[
\bm z
=
\ell(\bm w;\{\bm z_k\}_{k=1}^K)
=
\mathcal T_{\leftarrow}(\bm\xi\bm w)
\]
for some $\bm w\in\mathbb S^{K-1}_+$. Applying $\mathcal T_{\rightarrow}$ and using (A2) gives
\[
\mathcal T_{\rightarrow}(\bm z)
=
\mathcal T_{\rightarrow}\!\big(\mathcal T_{\leftarrow}(\bm\xi\bm w)\big)
=
\alpha(\bm\xi\bm w)\,\bm\xi\bm w.
\]
Multiplying by $\bm\xi^+$ and using (A1), we obtain
\[
\bm\xi^+\mathcal T_{\rightarrow}(\bm z)
=
\alpha(\bm\xi\bm w)(\bm\xi^+\bm\xi)\bm w
=
\alpha(\bm\xi\bm w)\bm w.
\]
Thus $\bm\xi^+\mathcal T_{\rightarrow}(\bm z)$ is a positive scalar multiple of $\bm w$. Since $\|\bm w\|_2=1$, normalising removes the unknown positive scale factor:
\[
\frac{\bm\xi^+\mathcal T_{\rightarrow}(\bm z)}
{\|\bm\xi^+\mathcal T_{\rightarrow}(\bm z)\|_2}
=
\bm w.
\]
Therefore the stated expression recovers the unique weight vector $\bm w$ that generated $\bm z$, establishing injectivity of $\ell$ on $\mathbb S^{K-1}_+$ and giving the inverse on the image of $\ell$.
\end{proof}

\begin{corollary}[When normalisation is redundant]
For generated points $\bm z=\ell(\bm w;\{\bm z_k\}_{k=1}^K)$, the normalisation in the inverse formula is redundant whenever
\[
\mathcal T_{\rightarrow}\!\circ\mathcal T_{\leftarrow}
=
\mathrm{id}
\]
on the relevant inner latent subspace. In this case, $\alpha(\bm\epsilon)\equiv 1$, and
\[
\ell^{-1}(\bm z;\{\bm z_k\}_{k=1}^K)
=
\bm\xi^+\mathcal T_{\rightarrow}(\bm z).
\]

\begin{itemize}
  \item \textbf{Gaussian latents.}
  The maps $\mathcal T_{\rightarrow}$ and $\mathcal T_{\leftarrow}$ are exact inverses, so $\alpha(\bm\epsilon)=1$ and normalisation is not required.

  \item \textbf{Hyperspherical latents.}
  With $\mathcal T_{\leftarrow}(\bm\epsilon)=\bm\epsilon/\|\bm\epsilon\|_2$ and $\mathcal T_{\rightarrow}=\mathrm{id}$,
  \[
  \mathcal T_{\rightarrow}\!\big(\mathcal T_{\leftarrow}(\bm\epsilon)\big)
  =
  \frac{\bm\epsilon}{\|\bm\epsilon\|_2},
  \]
  so $\alpha(\bm\epsilon)=1/\|\bm\epsilon\|_2$. Normalisation is therefore required to recover $\bm w$.

  \item \textbf{Independent scalar latents mapped via CDF transports.}
  When the scalar CDF transports and their inverses are exact, $\mathcal T_{\rightarrow}$ and $\mathcal T_{\leftarrow}$ are exact inverses, so $\alpha(\bm\epsilon)=1$ and normalisation is redundant.
\end{itemize}
\end{corollary}

\section{The unit hypersphere is a sufficient index set for Latent Optimal Linear combinations}
\label{sec:hypersphere_sufficient}

In this section, we show that Latent Optimal Linear combinations~\citep{bodin2024linear} (LOL) can be indexed by weights on the unit hypersphere. The essential point is that valid transported aggregation requires only a unit-norm direction among the seed latents. In the zero-mean Gaussian case, this can also be viewed as removing radial redundancy in an arbitrary nonzero weight vector: scaling all weights by a positive constant does not change the corrected latent. We first describe this in the Gaussian case and then extend the argument to the transported, non-Gaussian setting.

\paragraph{Gaussian latents.}
Let
\[
\bm y = \bm Z\bm a,
\qquad
\bm Z=[\bm z_1,\dots,\bm z_K],
\qquad
\bm a\in\mathbb R^K,
\]
where the seed latents are independent draws $\bm z_k\sim p$ with
\[
p=\mathcal N(\bm\mu,\bm\Sigma).
\]
Then
\[
\bm y \sim \mathcal N(\alpha\bm\mu,\beta\bm\Sigma),
\qquad
\alpha=\sum_{k=1}^K a_k,
\qquad
\beta=\sum_{k=1}^K a_k^2=\|\bm a\|_2^2.
\]
Unless $\alpha=1$ (or $\bm\mu=\bm 0$) and $\beta=1$, the linear combination $\bm y$ is not guaranteed to follow the target latent distribution $p$. In the Gaussian case, Latent Optimal Linear combinations correct this by applying the Monge optimal transport map from $\mathcal N(\alpha\bm\mu,\beta\bm\Sigma)$ to $\mathcal N(\bm\mu,\bm\Sigma)$,
\[
\mathcal T(\bm y)
=
\bm\mu+\frac{\bm y-\alpha\bm\mu}{\sqrt{\beta}}
=
\left(1-\frac{\alpha}{\|\bm a\|_2}\right)\bm\mu
+
\frac{\bm y}{\|\bm a\|_2}.
\]
For zero-mean Gaussian latents, this simplifies to
\[
\mathcal T(\bm y)
=
\frac{\bm Z\bm a}{\|\bm a\|_2}.
\]
Thus, for any nonzero $\bm a\in\mathbb R^K$, defining
\[
\bm w=\frac{\bm a}{\|\bm a\|_2}\in\mathbb S^{K-1}
\]
gives
\[
\mathcal T(\bm Z\bm a)
=
\bm Z\bm w.
\]
Therefore the norm of $\bm a$ is redundant in the zero-mean Gaussian case: all positive rescalings of $\bm a$ yield the same corrected latent. Consequently, weights on the unit hypersphere $\mathbb S^{K-1}$ are sufficient to index all zero-mean Gaussian Latent Optimal Linear combinations.

For nonzero-mean Gaussian latents, the same conclusion holds after centring. Writing
\[
\bm\epsilon_k=\bm z_k-\bm\mu,
\]
we form the centred combination
\[
\bm\epsilon=\sum_{k=1}^K w_k\bm\epsilon_k,
\qquad
\bm w\in\mathbb S^{K-1},
\]
and map back via
\[
\bm z=\bm\epsilon+\bm\mu.
\]
Since $\|\bm w\|_2=1$, the centred combination satisfies $\bm\epsilon\sim\mathcal N(\bm 0,\bm\Sigma)$, and therefore $\bm z\sim\mathcal N(\bm\mu,\bm\Sigma)$. Thus, centring removes the dependence on the mean, and unit-norm weights again provide a sufficient index set.

\paragraph{General transported latents.}
The same principle applies in the non-Gaussian setting whenever the latent distribution admits an inner representation whose unit-$\ell_2$ aggregations can be mapped back to valid samples from the target distribution. Let $\mathcal T_{\rightarrow}$ map model latents into an inner representation, and let $\mathcal T_{\leftarrow}$ map aggregated inner vectors back to samples from $p$. Given seed latents $\{\bm z_k\}_{k=1}^K$, define
\[
\bm\epsilon_k=\mathcal T_{\rightarrow}(\bm z_k),
\qquad
\bm\xi=[\bm\epsilon_1,\dots,\bm\epsilon_K].
\]
The required condition is that, for independent seed latents $\bm z_k\sim p$ and any
\[
\bm w\in\mathbb S^{K-1},
\qquad
\bm\epsilon=\bm\xi\bm w,
\]
the transported aggregate satisfies
\[
\bm z=\mathcal T_{\leftarrow}(\bm\epsilon)\sim p.
\]
Thus, the unit hypersphere provides a sufficient index set for valid Latent Optimal Linear combinations: the weights specify a unit-norm aggregation direction among the transformed seed latents, while the backward transport ensures that the resulting latent has the correct target distribution.

\paragraph{Summary.}
Both the Gaussian and transported non-Gaussian constructions rely on the same mechanism: aggregation is performed in an inner representation using unit-$\ell_2$ weights, and the resulting aggregate is mapped back to the target latent distribution. Therefore, valid Latent Optimal Linear combinations can be indexed by weights on the unit hypersphere $\mathbb S^{K-1}$. In the zero-mean Gaussian case, this additionally removes the radial redundancy of arbitrary nonzero weights in $\mathbb R^K$, since the corrected latent depends only on the normalised direction of the weight vector. 

\section{Weight charts $\phi_w$}
\label{appendix:charts}

Let $K$ denote the number of seed latents. A \emph{weight chart} is a map
\[
\phi_w:[0,1]^{K-1}\to\mathbb S^{K-1}_+,
\]
where
\[
\mathbb S^{K-1}_+
=
\left\{
\bm w\in\mathbb R^K
\;\middle|\;
\|\bm w\|_2=1,\; w_i\geq 0 \;\; \forall i
\right\}
\]
is the positive orthant of the unit hypersphere. The charts below define smooth bijections from the open cube $(0,1)^{K-1}$ to the interior of $\mathbb S^{K-1}_+$, and extend continuously to the boundary.

\subsection{Angular-coordinate chart}

The standard spherical-coordinate parametrisation gives a simple chart for $\mathbb S^{K-1}_+$. For $\bm u\in(0,1)^{K-1}$, set
\[
\theta_i=\frac{\pi}{2}u_i,
\qquad
i=1,\dots,K-1,
\]
and define
\begin{align}
w_1
&=
\cos\theta_1,
\\
w_k
&=
\left(\prod_{i=1}^{k-1}\sin\theta_i\right)\cos\theta_k,
\qquad
k=2,\dots,K-1,
\\
w_K
&=
\prod_{i=1}^{K-1}\sin\theta_i.
\end{align}
This gives $\bm w\in\mathbb S^{K-1}_+$.

The inverse is obtained recursively. First,
\[
\theta_1=\arccos(w_1),
\qquad
u_1=\frac{2}{\pi}\theta_1.
\]
For $k=2,\dots,K-1$,
\[
\theta_k
=
\arccos\left(
\frac{w_k}{\prod_{i=1}^{k-1}\sin\theta_i}
\right),
\qquad
u_k=\frac{2}{\pi}\theta_k.
\]
This chart is simple and smooth on the interior, but it is not measure-preserving: if $\bm u\sim\mathrm{Unif}([0,1]^{K-1})$, the induced distribution on $\mathbb S^{K-1}_+$ is not the uniform surface measure. The inverse can also be numerically unstable near the boundary, where some factors $\sin\theta_i$ are close to zero.

\subsection{Knothe--Rosenblatt stick-breaking chart}

We next define a measure-preserving chart based on the Knothe--Rosenblatt rearrangement for the Dirichlet distribution. Let $\bm u\in(0,1)^{K-1}$ and define independent stick-breaking variables
\[
v_k
=
I^{-1}_{u_k}\!\left(\frac12,\frac{K-k}{2}\right),
\qquad
k=1,\dots,K-1,
\]
where $I^{-1}_{u}(a,b)$ denotes the inverse regularised incomplete beta function. Define
\begin{align}
q_1
&=
v_1,
\\
q_k
&=
v_k\prod_{i=1}^{k-1}(1-v_i),
\qquad
k=2,\dots,K-1,
\\
q_K
&=
\prod_{i=1}^{K-1}(1-v_i),
\end{align}
and set
\[
w_i=\sqrt{q_i},
\qquad
i=1,\dots,K.
\]
Since $q_i\geq 0$ and $\sum_{i=1}^K q_i=1$, this gives $\bm w\in\mathbb S^{K-1}_+$.

The inverse is obtained by reversing the stick-breaking transform. Let
\[
q_i=w_i^2,
\qquad
s_k=\sum_{j=k}^K q_j.
\]
Then, for $k=1,\dots,K-1$,
\[
v_k=\frac{q_k}{s_k},
\qquad
u_k
=
I_{v_k}\!\left(\frac12,\frac{K-k}{2}\right),
\]
where $I_v(a,b)$ is the regularised incomplete beta function.

This chart pushes $\mathrm{Unif}([0,1]^{K-1})$ forward to the uniform surface measure on $\mathbb S^{K-1}_+$. In this measure-preserving sense, it is equal-area. In practice, stable implementations of \texttt{betainc} and \texttt{betaincinv} should be used, and inputs may be clipped away from exactly $0$ and $1$ to avoid numerical singularities at the boundary.

\paragraph{Uniformity of the Knothe--Rosenblatt chart.}
Let $\bm u\sim\mathrm{Unif}((0,1)^{K-1})$. By the inverse-CDF construction,
\[
v_k
=
I^{-1}_{u_k}\!\left(\frac12,\frac{K-k}{2}\right)
\quad\Longrightarrow\quad
v_k
\sim
\mathrm{Beta}\!\left(\frac12,\frac{K-k}{2}\right),
\]
independently for $k=1,\dots,K-1$. The stick-breaking transform then gives
\[
\bm q=(q_1,\dots,q_K)
\sim
\mathrm{Dirichlet}\!\left(\frac12,\dots,\frac12\right).
\]
For a point $\bm w$ uniformly distributed on $\mathbb S^{K-1}_+$, the squared coordinates satisfy
\[
\bm w\odot\bm w
\sim
\mathrm{Dirichlet}\!\left(\frac12,\dots,\frac12\right).
\]
Conversely, if
\[
\bm q\sim\mathrm{Dirichlet}\!\left(\frac12,\dots,\frac12\right)
\]
and $w_i=\sqrt{q_i}$, then $\bm w$ is uniformly distributed on $\mathbb S^{K-1}_+$. Therefore, the Knothe--Rosenblatt stick-breaking chart pushes the uniform distribution on $[0,1]^{K-1}$ to the uniform surface measure on the positive orthant of the hypersphere.

\section{The weight chart $\phi_w$ sets the similarity structure}
\label{appendix:weights_dominating}
In Figure~\ref{fig:dominating_weights} we demonstrate numerical evidence for the claim in Section~\ref{sec:surrogate_latent_spaces} that the dot product $\bm{w}_i^\top\bm{w}_j$ is the dominant factor in determining the cosine similarity between two latent variables $\bm{\epsilon}_i, \bm{\epsilon}_j \in \mathbb{R}^D$ indexed by a surrogate latent space $\mathcal{U}$. 
We see that already $D \approx 100, K \leq 10$ yields a dominating $\bm{w}_i^\top\bm{w}_j$, as shown by Pearson correlations of more than $0.95$, and correlations very close to $1$ for higher dimensionalities of $D$ (at a rate dependent on $K$). 
For reference, typical diffusion and flow matching  models~\citep{rombach2022high,flux2024,lipman2022flow} have a $D$ of \emph{tens of thousands} to \emph{hundreds of thousands}, and \cite{kong2024hunyuanvideo} has a dimensionality of several \emph{million}. 

\newcommand{\imgscaletwo}{0.35}

\begin{figure*}[htbp]
    \centering

    \begin{subfigure}{0.8\textwidth}
        \centering
        \includegraphics[width=\linewidth]{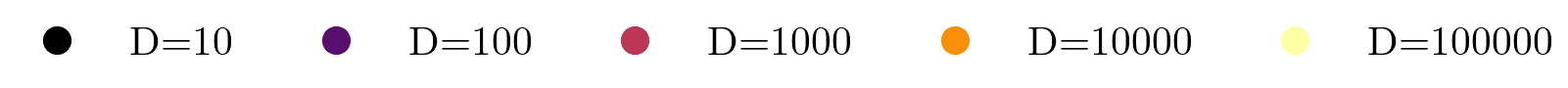}
    \end{subfigure}
    
    \begin{subfigure}{\imgscaletwo\textwidth}
        \centering
        \includegraphics[width=\linewidth]{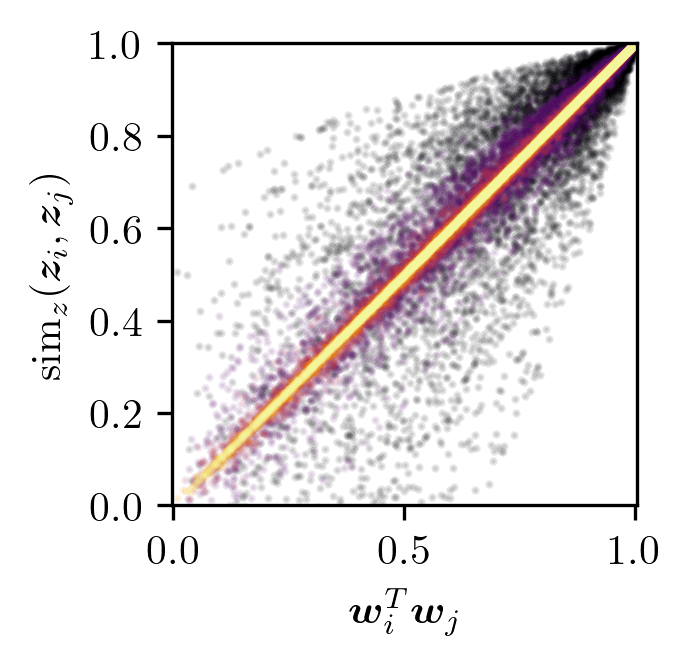}
        \caption{Number of seeds $K = 2$}
    \end{subfigure}
    \begin{subfigure}{\imgscaletwo\textwidth}
        \centering
        \includegraphics[width=\linewidth]{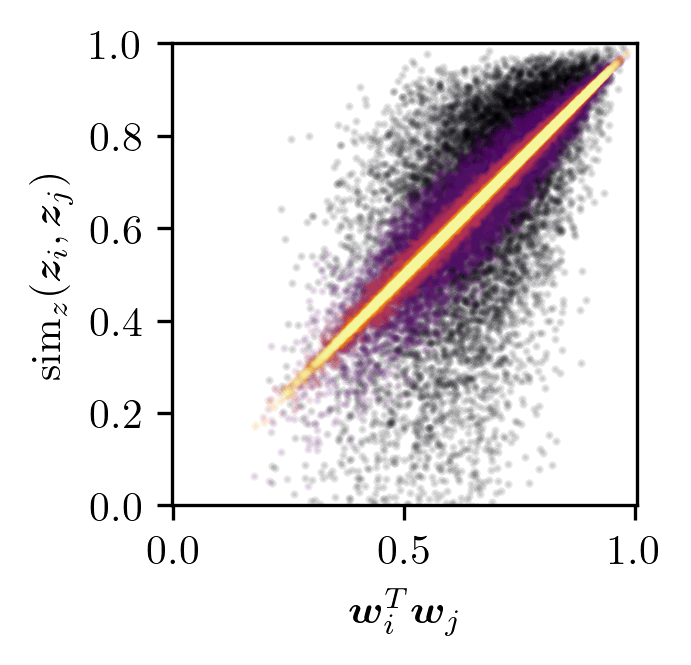}
        \caption{Number of seeds $K = 10$}
    \end{subfigure}
    
    \begin{subfigure}{\imgscaletwo\textwidth}
        \centering
        \includegraphics[width=\linewidth]{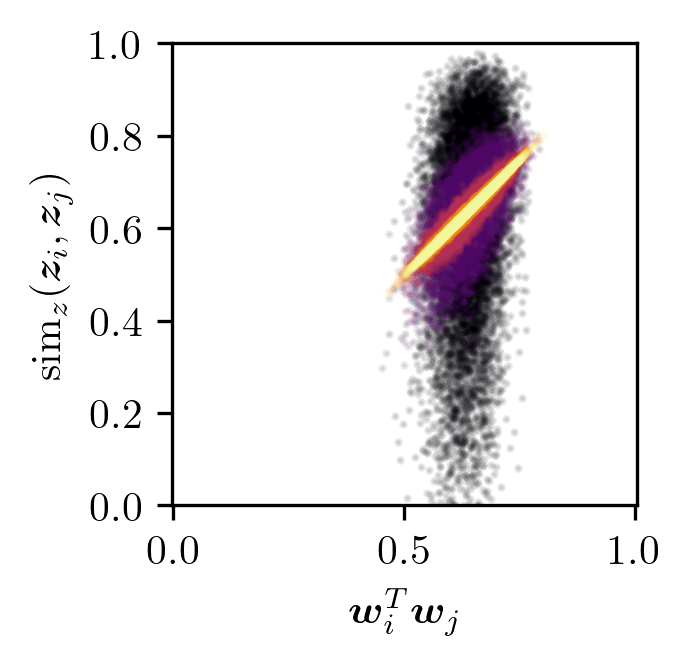}
        \caption{Number of seeds $K = 100$}
    \end{subfigure}
    \begin{subfigure}{\imgscaletwo\textwidth}
        \centering
        \includegraphics[width=\linewidth]{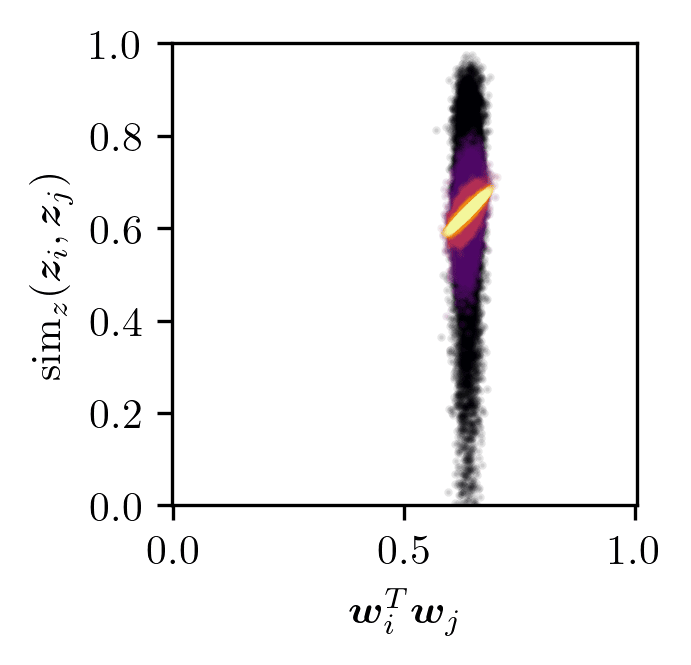}
        \caption{Number of seeds $K = 1000$}
    \end{subfigure}

    \vspace{1em}

    \begin{subfigure}{0.5\textwidth}
        \centering
        \includegraphics[width=\linewidth]{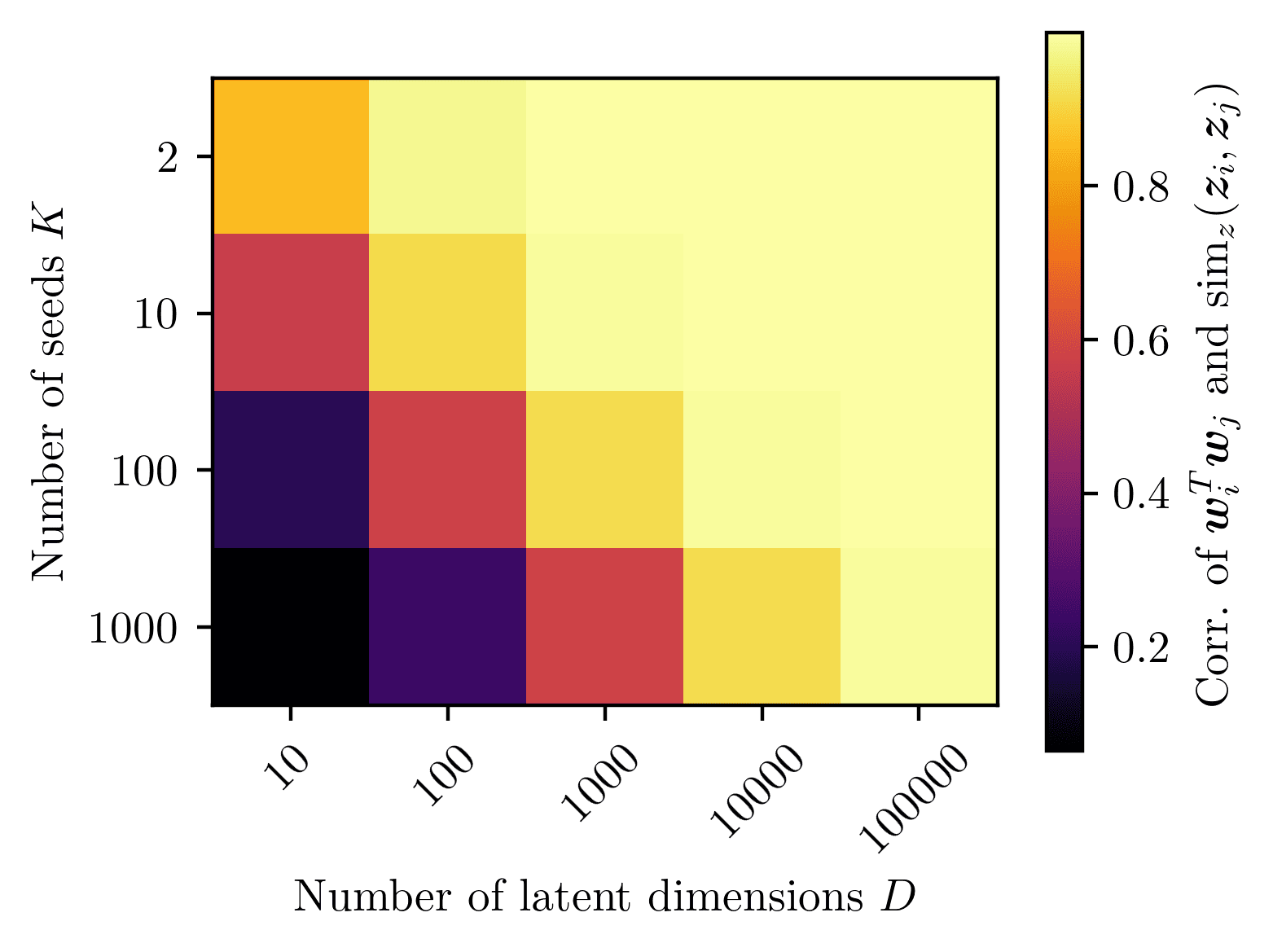}
        \caption{Top image}
    \end{subfigure}
    
    \caption{
    (\textit{top}) Dot products of the weights $\bm{w}_i^\top\bm{w}_j$ and the cosine similarity  $\text{sim}_{z}(\bm{z}_i, \bm{z}_j)$ for uniformly drawn samples in $\mathcal{U}$ for $K = 2, 10, 100, 1000$, respectively for various dimensions $D$, where $\bm{z} \in \mathbb{R}^D$. The number of samples per setting is $10,000$, with $100$ realisations of the seeds drawn from $\mathcal{N}(\bm{0}, \bm{I})$ and $100$ uniformly sampled $\bm{u}$ per sampled seeds realisation. (\textit{bottom}) Estimated correlations between $\bm{w}_i^\top\bm{w}_j$ and $\text{sim}_{z}(\bm{z}_i, \bm{z}_j)$ uses all the samples per setting.  
    }
    \label{fig:dominating_weights}
\end{figure*}

\section{Empirical stationarity assessments of $\phi_w$}
\label{appendix:empirical_stationarity}

In Section~\ref{sec:surrogate_latent_spaces}, we discuss the notion of stationarity and explain why, in practice, it can only be achieved approximately. In our setting, the relevant similarity structure is induced by the weight chart $\phi_w$, which maps surrogate coordinates $\bm u\in\mathcal U$ to weights $\bm w\in\mathbb S^{K-1}_+$. We would therefore like the chosen chart to preserve, as far as possible, a simple relationship between Euclidean distances in $\mathcal U$ and similarities between the corresponding weights. In this section, we make this notion precise and empirically evaluate the candidate charts introduced in Section~\ref{appendix:charts}. These candidates are not intended to be exhaustive; rather, our goal is to identify a sufficiently well-behaved chart for use in our methodology.

Define the chart-induced similarity kernel
\[
k_{\phi}(\bm u,\bm u')
=
\phi_w(\bm u)^\top \phi_w(\bm u').
\]
A stationary, or translation-invariant, kernel depends only on the displacement between its inputs,
\[
k_{\phi}(\bm u,\bm u')=\kappa(\bm u-\bm u').
\]
In this work, we consider the stronger isotropic form, where the dependence is only through Euclidean distance:
\[
k_{\phi}(\bm u,\bm u')
=
v\!\left(\|\bm u-\bm u'\|_2\right).
\]
Equivalently, for any two pairs of points $(\bm u_1,\bm u_2)$ and $(\bm u_3,\bm u_4)$ in $\mathcal U$ with the same Euclidean distance,
\[
\|\bm u_1-\bm u_2\|_2
=
\|\bm u_3-\bm u_4\|_2,
\]
isotropic stationarity would require
\[
\phi_w(\bm u_1)^\top \phi_w(\bm u_2)
=
\phi_w(\bm u_3)^\top \phi_w(\bm u_4).
\]
Since such an isotropic stationary relationship cannot hold exactly for a low-distortion chart from a Euclidean cube to the hypersphere, we instead assess \emph{approximate stationarity}: the induced dot products need not be identical for equal-distance pairs, but should vary predictably and monotonically with Euclidean distance across most of $\mathcal U$.

Figure~\ref{fig:approx_stationarity_paths} evaluates this property using one million pairs of points in $\mathcal U$ for both the Knothe--Rosenblatt (KR) chart and the angular-coordinate chart. Each pair is obtained by sampling two points uniformly along a randomly chosen line through $\mathcal U$. For comparison, Figure~\ref{fig:approx_stationarity} presents the same analysis when both points are sampled independently and uniformly from $\mathcal U$. These two sampling schemes behave differently in high dimensions. Under independent sampling, pairwise distances in the hypercube concentrate sharply, making it difficult to assess the dependence of induced dot products across a broad range of distances. In contrast, line-based sampling produces a wider and more informative range of separations, and is more relevant for optimisation because it probes behaviour along local search directions.

Across the dimensionalities tested, the KR chart exhibits a strong relationship between Euclidean distance in $\mathcal U$ and the corresponding weight-space dot product, with this relationship becoming clearer as the dimension increases. By contrast, the angular-coordinate chart shows substantially weaker distance--similarity alignment in higher dimensions, particularly for pairs that are farther apart. These empirical results motivate our use of the KR chart as the default choice for $\phi_w$.
\newcommand{\imgscale}{0.4}

\begin{figure*}[htbp]
    \centering

    \begin{minipage}{0.95\textwidth}
        \centering
        \makebox[\imgscale\linewidth]{\hspace{4em}Knothe--Rosenblatt chart}
        \makebox[\imgscale\linewidth]{Angular coordinates chart}
    \end{minipage}
    \vspace{0.5em}
    
    \begin{subfigure}{0.85\textwidth}
        \centering
        \includegraphics[width=\imgscale\linewidth]{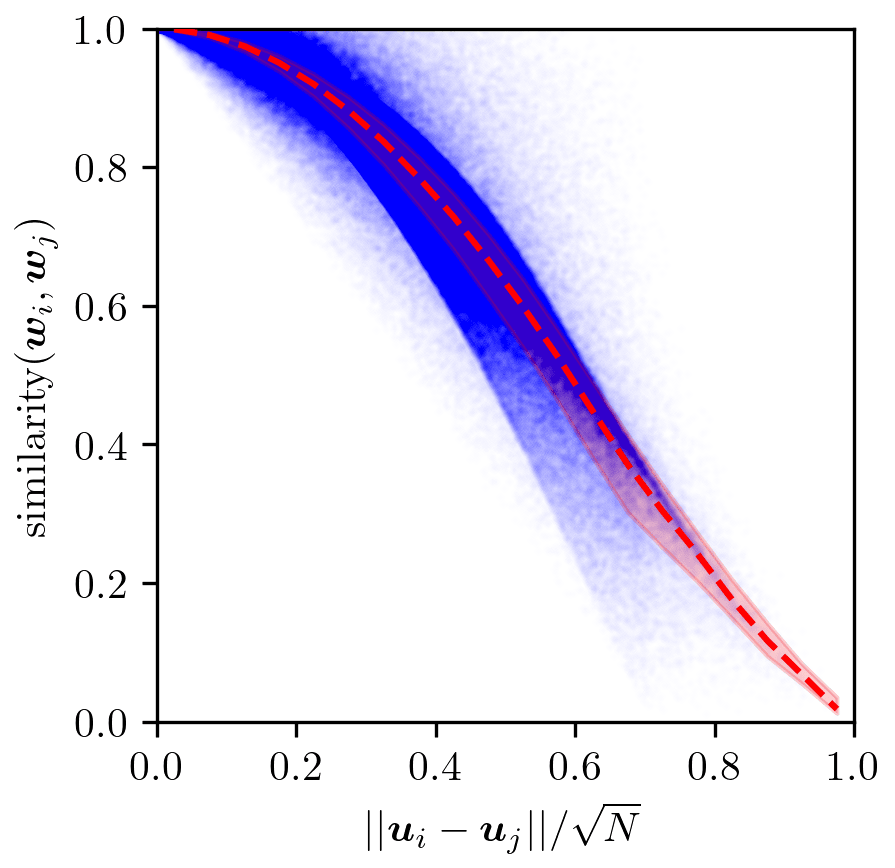}
        \includegraphics[width=\imgscale\linewidth]{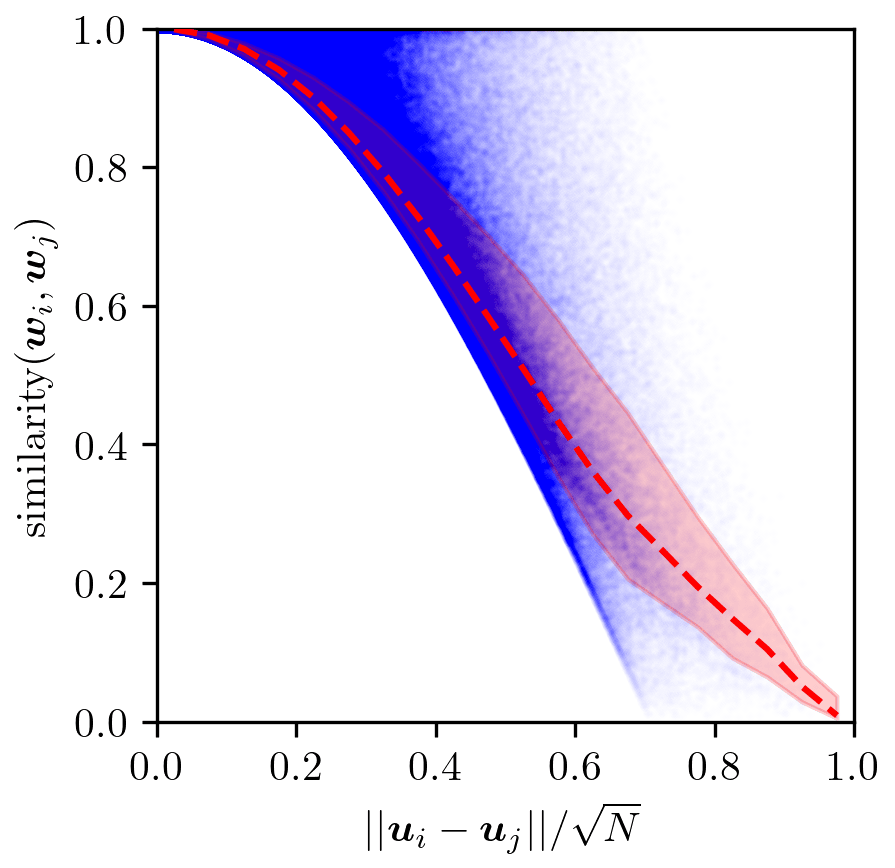}
        \caption{Surrogate space dimensionality: $2$}
    \end{subfigure}
    
    \begin{subfigure}{0.85\textwidth}
        \centering
        \includegraphics[width=\imgscale\linewidth]{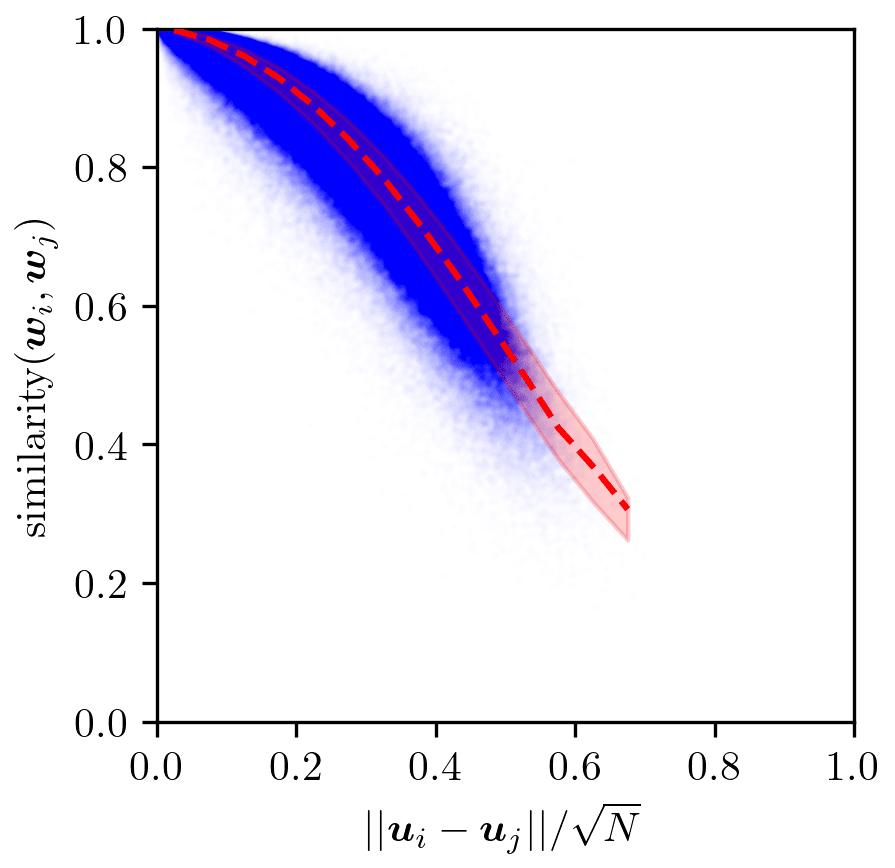}
        \includegraphics[width=\imgscale\linewidth]{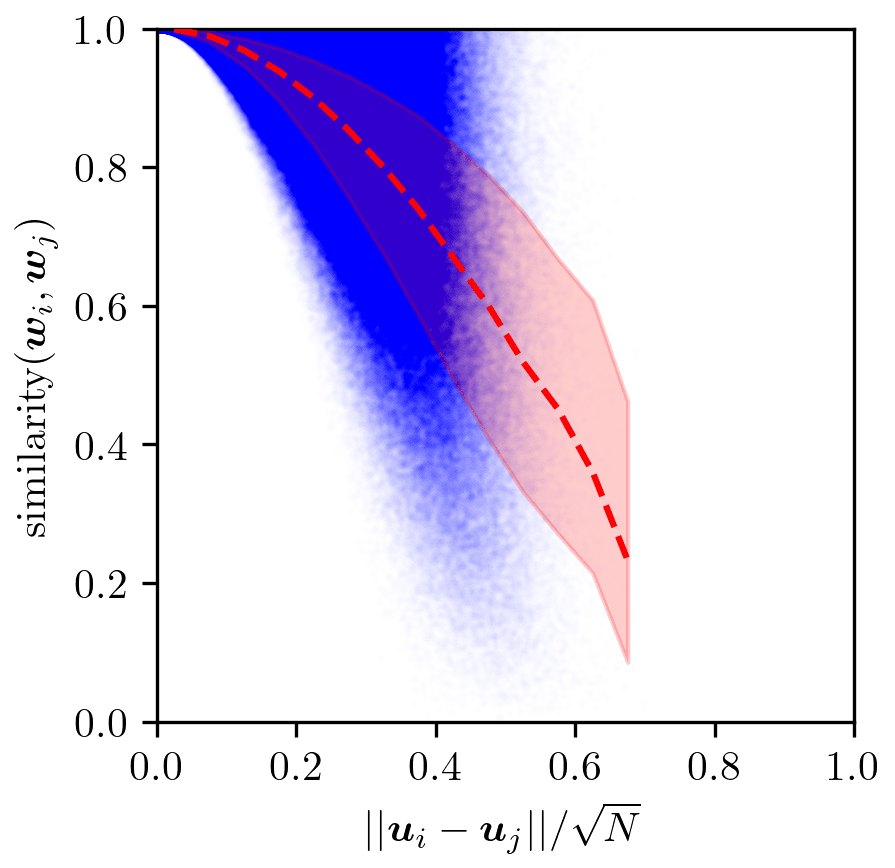}
        \caption{Surrogate space dimensionality: $10$}
    \end{subfigure}

    \begin{subfigure}{0.85\textwidth}
        \centering
        \includegraphics[width=\imgscale\linewidth]{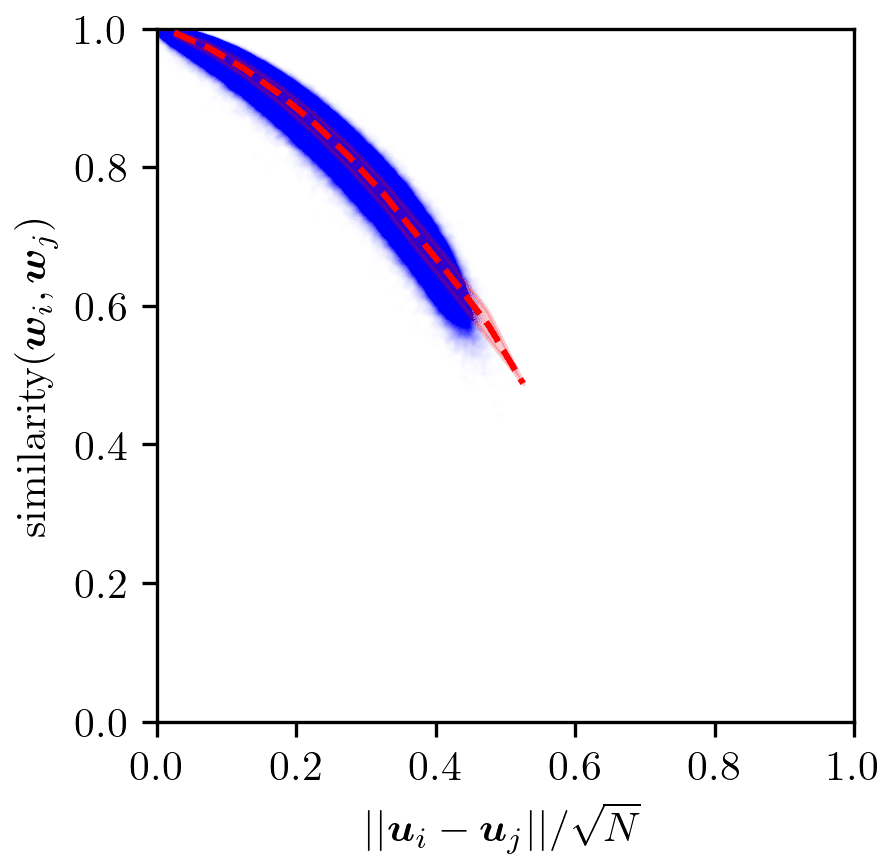}
        \includegraphics[width=\imgscale\linewidth]{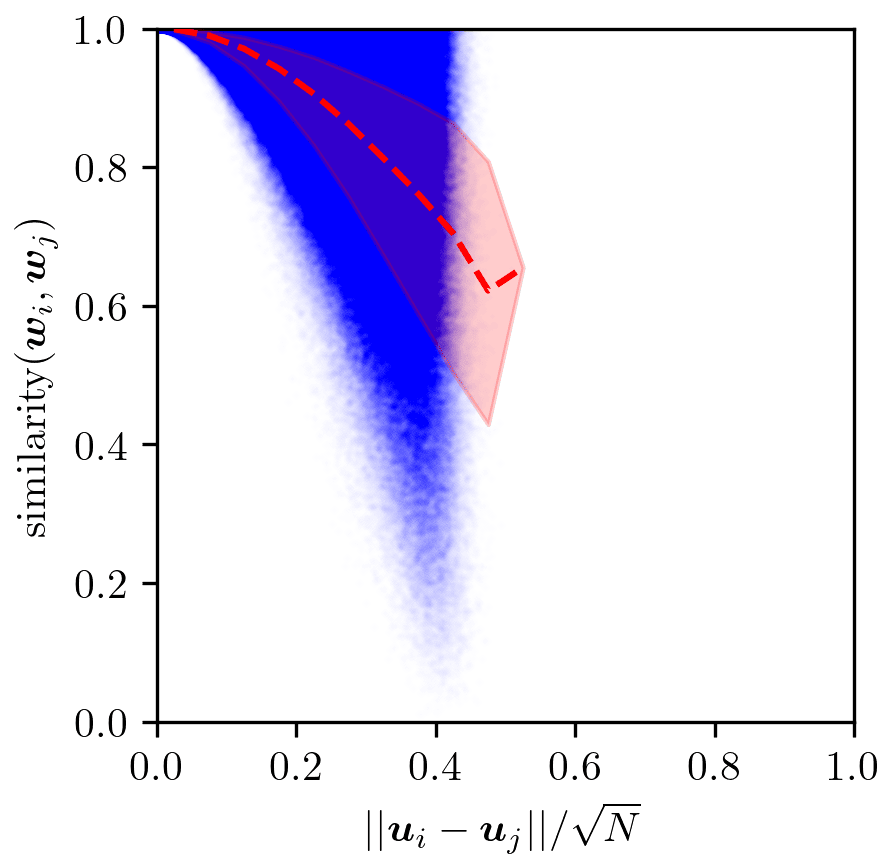}
        \caption{Surrogate space dimensionality: $100$}
    \end{subfigure}

    \caption{
    Empirical assessment of approximate stationarity for the Knothe--Rosenblatt chart and the angular-coordinate chart. Blue points show one million sampled pairs of surrogate coordinates in $\mathcal U$, plotting their normalised Euclidean distance against the dot product of their corresponding weight vectors. The left column uses the Knothe--Rosenblatt chart and the right column uses the angular-coordinate chart; see Appendix~\ref{appendix:charts}. The red dashed line shows the mean similarity, and the red shaded region shows the central 50\% interval. The top, middle, and bottom rows correspond to surrogate dimensions $2$, $10$, and $100$, respectively, or equivalently to $3$, $11$, and $101$ seed latents.
    Pairs are generated by repeating the following procedure one million times: (1) sample a point $\bm p_1$ uniformly in $\mathcal U$; (2) form a line segment from $\bm p_1$ to a point sampled uniformly on the boundary of $\mathcal U$; and (3) sample $\bm p_2$ uniformly along this line segment.
    }
    \label{fig:approx_stationarity_paths}
\end{figure*}

\begin{figure*}[htbp]
    \centering

    \begin{minipage}{0.95\textwidth}
        \centering
        \makebox[\imgscale\linewidth]{\hspace{4em}Knothe--Rosenblatt chart}
        \makebox[\imgscale\linewidth]{Angular coordinates chart}
    \end{minipage}
    \vspace{0.5em}
    
    \begin{subfigure}{0.85\textwidth}
        \centering
        \includegraphics[width=\imgscale\linewidth]{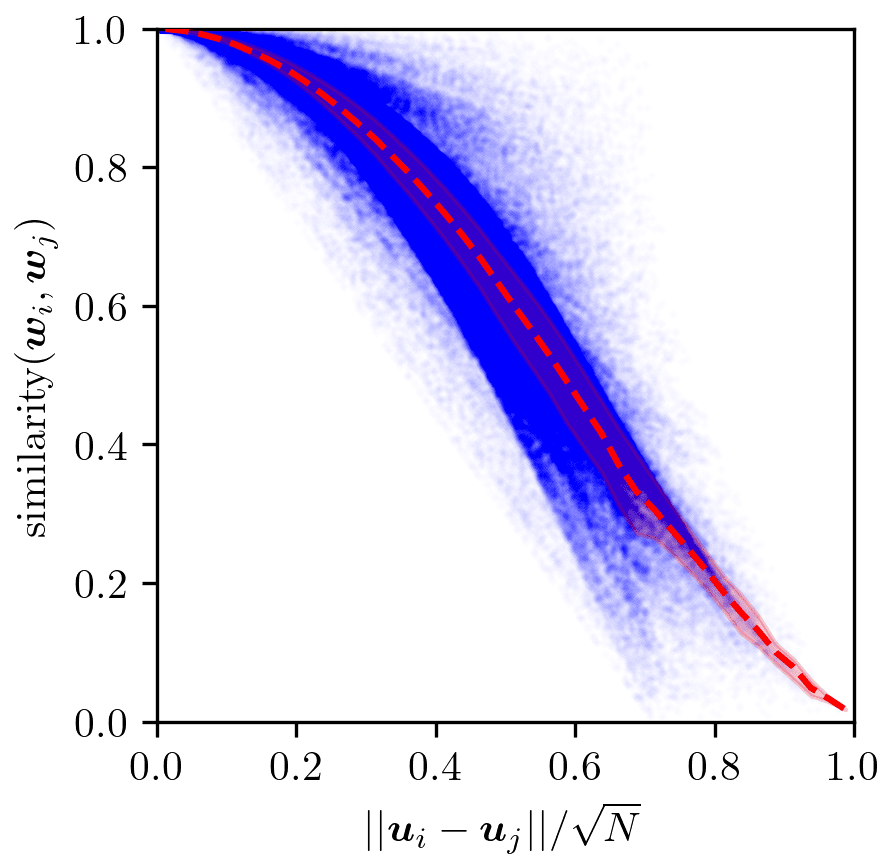}
        \includegraphics[width=\imgscale\linewidth]{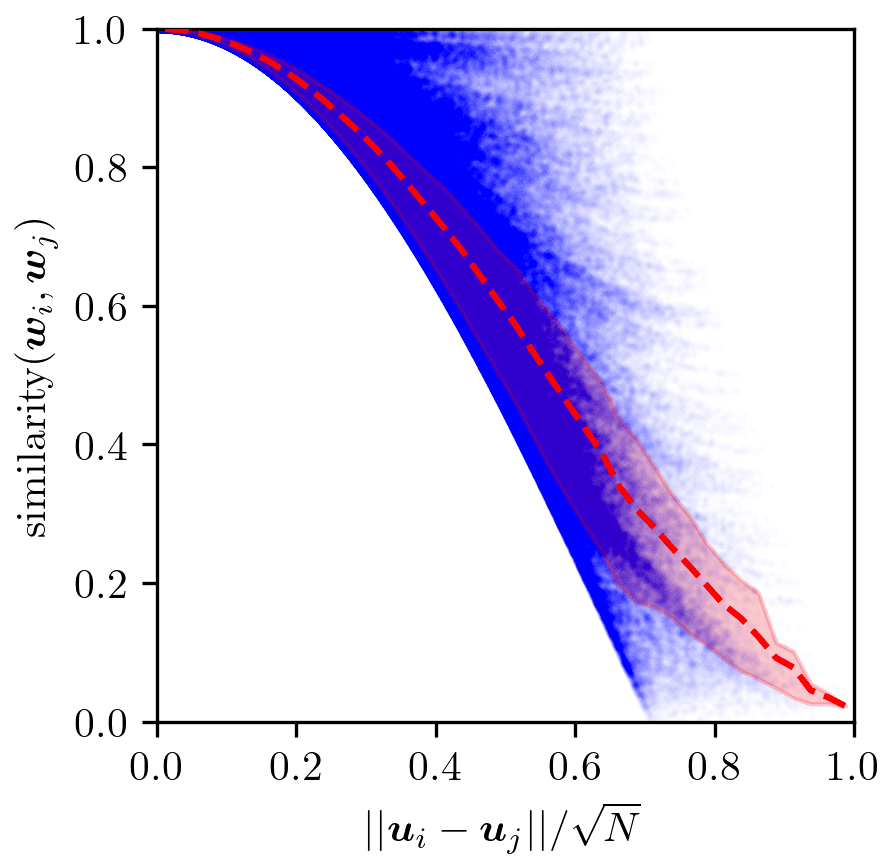}
        \caption{Surrogate space dimensionality: $2$}
    \end{subfigure}
    
    \begin{subfigure}{0.85\textwidth}
        \centering
        \includegraphics[width=\imgscale\linewidth]{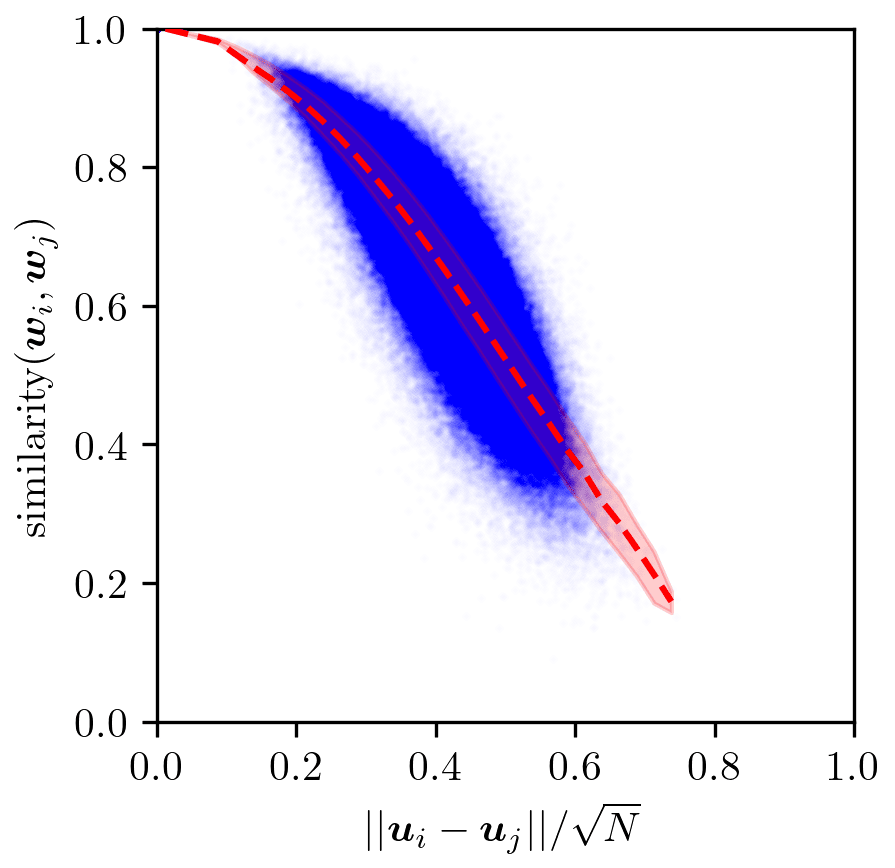}
        \includegraphics[width=\imgscale\linewidth]{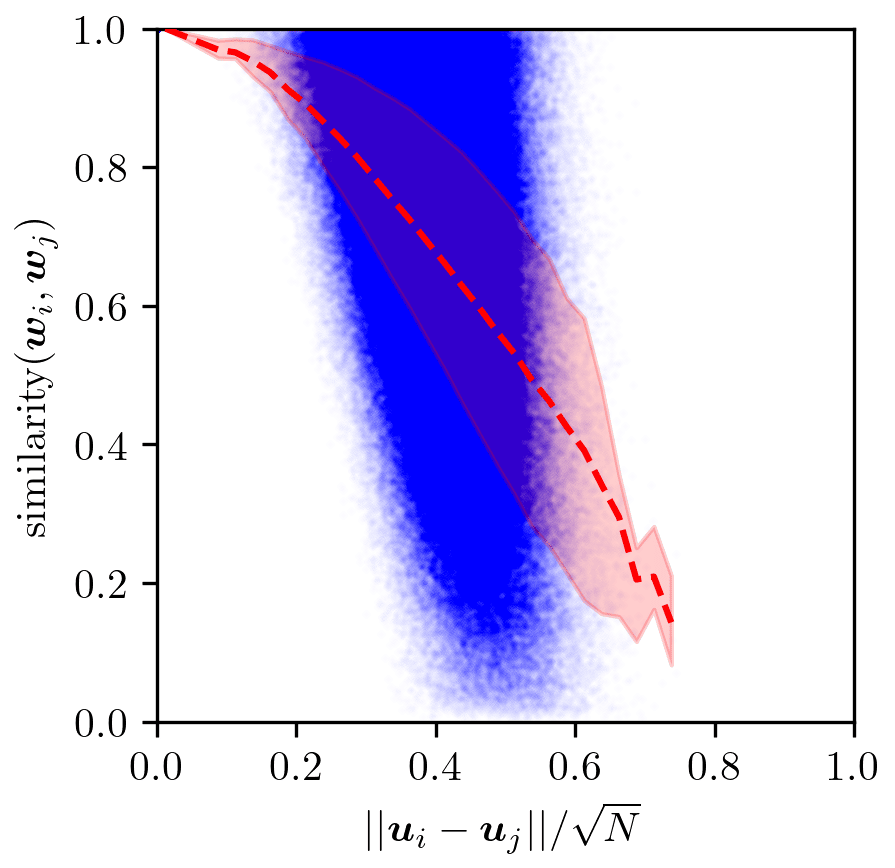}
        \caption{Surrogate space dimensionality: $10$}
    \end{subfigure}

    \begin{subfigure}{0.85\textwidth}
        \centering
        \includegraphics[width=\imgscale\linewidth]{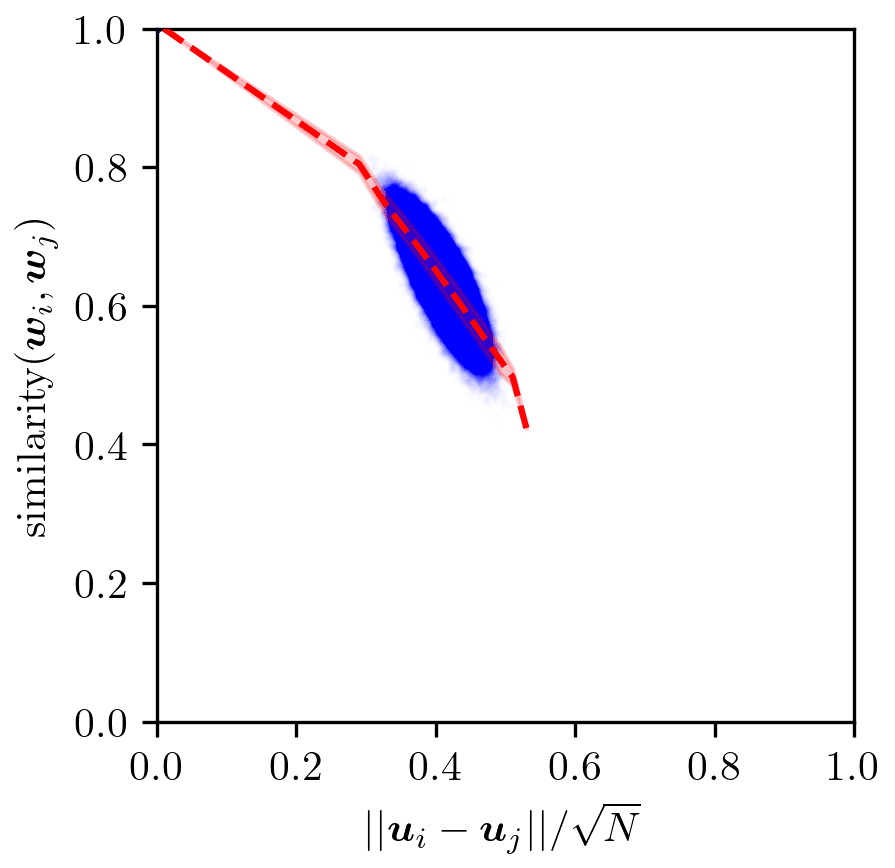}
        \includegraphics[width=\imgscale\linewidth]{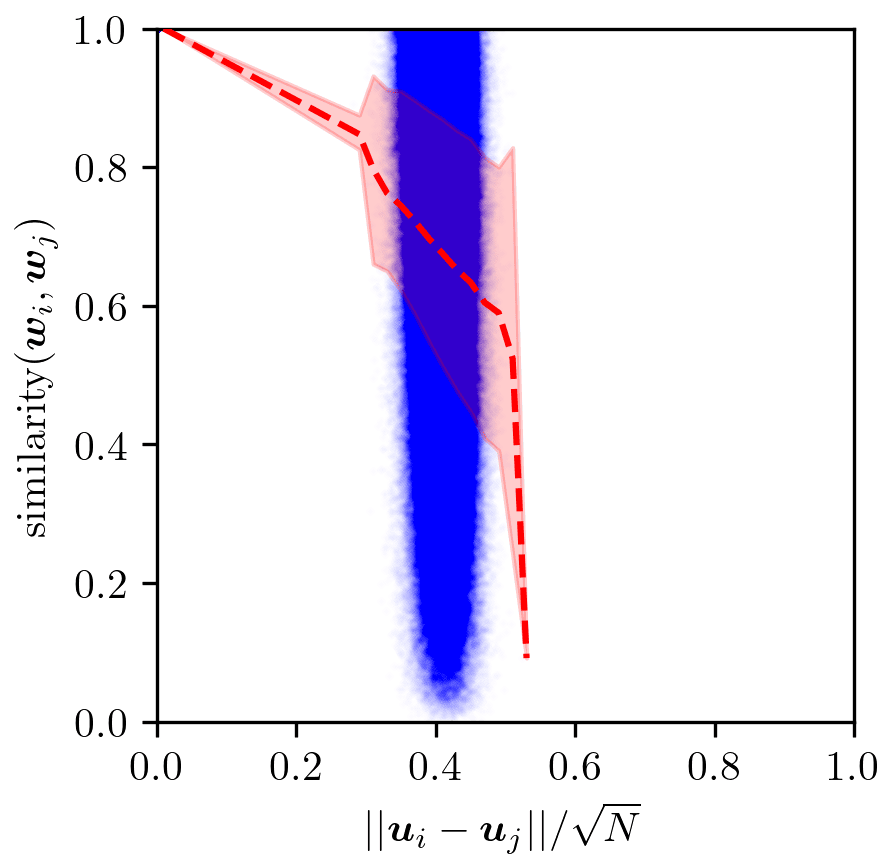}
        \caption{Surrogate space dimensionality: $100$}
    \end{subfigure}

   \caption{
    Empirical assessment of approximate stationarity under independent uniform sampling in the surrogate space $\mathcal U$. Blue points show one million independently sampled pairs of surrogate coordinates, plotting their normalised Euclidean distance against the dot product of their corresponding weight vectors. The left column uses the Knothe--Rosenblatt chart and the right column uses the angular-coordinate chart; see Appendix~\ref{appendix:charts}. The red dashed line shows the mean similarity, and the red shaded region shows the central 50\% interval. The top, middle, and bottom rows correspond to surrogate dimensions $2$, $10$, and $100$, respectively, or equivalently to $3$, $11$, and $101$ seed latents. The Knothe--Rosenblatt chart maintains a strong negative relationship between Euclidean distance and weight-space dot product, whereas this relationship is substantially weaker for the angular-coordinate chart.
    }
    \label{fig:approx_stationarity}
\end{figure*}

\section{Good examples define spaces with better solutions: details}
\label{appendix:good_seeds_better_solutions_details}

For our surrogate spaces to be useful they need to contain varied objects that share characteristics with the seed latents; especially those attributes that impact targeted objective functions. Figure~\ref{fig:image_grid} presented an illustrative example, where a 2D space formed from three seeds indeed contained large areas of solutions better than the seeds. We now asses this property qualitatively on the benchmark presented in~\cite{denker2025iterative}, where we seek generations from the diffusion model Stable Diffusion 1.5~\citep{rombach2022high} that score highly according to ImageReward~\citep{xu2023imagereward}, a measure of alignment with a target prompt. 
We also report diversity scores by computing one minus the mean cosine similarity of the CLIP~\citep{radford2021learning} embeddings of the images. 

\begin{table*}[ht]
\caption{
\textbf{
Scores and diversities of generated images
}
Reported is the median and the 90\% confidence interval of the mean score and diversity, respectively, for each method (row) and prompt (shown in the top). 
Each section of rows correspond to methods using different compute budgets; with standard model sampling, 100 samples are drawn randomly from the (base) model (SD 1.5), while the bottom section are methods requiring training beyond the base model. The `1/2 efficiency' are methods with 100 initial random samples at its disposal before then producing the final 100 samples, and `1/6 efficiency' is the same but using 500 initial samples. The `Best-of'-methods use the budget as per their name, and the `Grid in $\mathcal{U}$'-methods forms surrogate spaces from seeds being the highest scoring samples among the initial random samples, followed by producing a grid of 100 points.
}
\centering
\resizebox{\textwidth}{!}{%
\begin{tabular}{lccccccc}
\toprule
& \multicolumn{2}{c}{``A green colored rabbit.''} & \multicolumn{2}{c}{``Two roses in a vase.''} & \multicolumn{2}{c}{``Two dogs in the park.''} & GPU hrs (training) \\
\cmidrule(lr){2-3} \cmidrule(lr){4-5} \cmidrule(lr){6-7}
& Reward ($\uparrow$) & Diversity ($\uparrow$) & Reward ($\uparrow$) & Diversity ($\uparrow$) & Reward ($\uparrow$) & Diversity ($\uparrow$) &  \\
\midrule
Standard model sampling          &  -0.16 [-0.35, 0.03] & 0.17 [0.15,0.18] & 0.80 [0.66, 0.90] & 0.12 [0.10, 0.13] & 0.36 [0.29, 0.43] & 0.16 [0.16, 0.17] & N/A \\
\midrule
\multicolumn{8}{l}{\textit{1/2 efficiency}} \\
Best-1-of-2, 100 times   &  0.54 [0.39, 0.73] & 0.17 [0.16, 0.18] & 1.22 [1.15, 1.30] & 0.09 [0.09, 0.10] & 0.63 [0.58, 0.69] & 0.16 [0.15, 0.16] & N/A \\
Best-100-of-200      &  1.01 [0.78, 1.17] & 0.17 [0.15, 0.18] & 1.40 [1.34, 1.45] & 0.09 [0.09, 0.10] & 0.73 [0.68, 0.78] & 0.15 [0.14, 0.16] & N/A \\
Grid in $\mathcal{U}^1$ (2 seeds)   &  1.67 [1.08, 1.84] & 0.06 [0.03, 0.12] & 1.43 [0.70, 1.77] & 0.05 [0.04, 0.11] & 0.83 [0.32, 1.16] & 0.09 [0.07, 0.14] & N/A \\
Grid in $\mathcal{U}^3$ (4 seeds)   &  1.55 [1.27, 1.77] & 0.08 [0.06, 0.13] & 1.29 [0.96, 1.53] & 0.08 [0.07, 0.11] & 0.74 [0.49, 0.93] & 0.13 [0.11, 0.15] & N/A \\
Grid in $\mathcal{U}^5$ (6 seeds)   &  1.56 [1.28,  1.74] & 0.09 [0.07, 0.13] & 1.34 [1.01, 1.58] & 0.09 [0.07, 0.12] & 0.68 [0.53, 0.83] & 0.13 [0.11, 0.15] & N/A \\
\midrule
\multicolumn{8}{l}{\textit{1/6 efficiency}} \\
Best-1-of-6, 100 times     &  1.42 [1.35, 1.50] & 0.13 [0.12, 0.14] & 1.55 [1.52, 1.58] & 0.09 [0.09, 0.10] & 0.91 [0.87, 0.95] & 0.15 [0.14, 0.15] & N/A \\
Best-100-of-600      &  1.65 [1.61, 1.69] & 0.11 [0.10, 0.12] & 1.64 [1.61, 1.65] & 0.09 [0.08, 0.10] & 1.00 [0.97, 1.03] & 0.15 [0.14, 0.15] & N/A \\
Grid in $\mathcal{U}^1$ (2 seeds)   &  1.75 [1.50, 1.86] & 0.04 [0.03, 0.08] & 1.55 [1.19, 1.78] & 0.06 [0.04, 0.11] & 0.84 [0.35, 1.13] & 0.09 [0.07, 0.11] & N/A \\
Grid in $\mathcal{U}^3$ (4 seeds)   &  1.73 [1.31, 1.83] & 0.07 [0.05, 0.08] & 1.40 [1.21, 1.65] & 0.08 [0.06,0.10] & 0.69 [0.32, 0.94] & 0.13 [0.11, 0.16] & N/A \\
Grid in $\mathcal{U}^5$ (6 seeds)   &  1.71 [1.45, 1.79] & 0.07 [0.06, 0.09] & 1.32 [1.15, 1.55] & 0.09 [0.07, 0.12] & 0.68 [0.40, 0.86] & 0.14 [0.12, 0.17] & N/A \\
\midrule
DPOK & 1.62 & 0.07 & 1.59 & 0.11 & 1.01 & 0.14 & 28 (A100) \\
Adjoint Matching & 1.71 & 0.05 & 1.50 & 0.09 & 1.33 & 0.13 & 4 (A100) \\
Importance Fine-tuning       &  1.46 & 0.05 & 1.53 & 0.08 & 1.01 & 0.12 & 7 (RTX 4090) \\
\bottomrule
\end{tabular}
}
\label{table:sampling_seeds}
\end{table*}

\textbf{Results}: We test our surrogate spaces by seeing if we can form 100 generations (gridded over the surrogate space) that are both good (well-aligned with the target prompt) and diverse. To generate the seed latents that define our surrogate space, we use a budget of $S$ random generations and pick the top $K$ (a stand-in for having a-priori access to `good' seeds). We also report the canonical baselines of taking the top 100 directly from the $S$ random generations on each run. Table~\ref{table:sampling_seeds} reports the median and 90\% confidence interval of the mean score of the 100 generated images per method over 30 repetitions. To provide context for these scores, we also include the results from \cite{denker2025iterative} who use the same benchmarking setup to compare the performance of algorithms that require many GPU hours of training on the particular score function in order to produce generations with high scores. We include their reported scores (they only provide one repetition) for Importance Fine-tuning~\citep{denker2025iterative}, DPOK~\citep{fan2023reinforcement}, Adjoint Matching~\citep{domingo2024adjoint}, demonstrating that, on two out of three prompts, our surrogate spaces produce higher or same scoring generations than the expensive fine-tuning approaches, and similarly diverse. 
A larger relatively volume of high scoring solutions was observed for the lower dimensional surrogate spaces (but slightly less diverse) --- which can be told by the grid yielding high mean scores --- which is expected as the lower dimensionalities were produced from better seeds (with more random samples per seed to determine them). 

\section{Surrogate spaces searched by standard optimisation algorithms: details}
\label{appendix:image_opt_details}

\textbf{Goal.} We will now confirm that surrogate spaces enable effective LSO in high-dimensional latent variable models by deploying popular optimisation algorithms and compare how they perform in our surrogate latent spaces against the original latent space.
Our methodology enables good or informative solutions to guide the search by defining a targeted space, which is to be reflected in the test task.  
We use the popular methods of CMA-ES~\citep{hansen2016cma}, BO~\citep{shahriari2015taking}, as well as random search. 
The objective function is to optimise the Pick score~\citep{kirstain2023pick} for generations of the Stable Diffusion (SD) 2.1~\citep{rombach2022high} model. 
Specifically, the model is given a general prompt (`A vehicle') and the objective is to obtain high Pick-scores for a prompt sampled randomly from a grammar composed of three parts as `A <attribute> <vehicle type> <environment>', forming a million possible combinations.  
The grammar is given in Section~\ref{appendix:image_comp_benchmark}. 
The sampled target prompts are hidden from the methods, but implicitly conveyed via the objective function; the score as a function of the generated image. 
At our disposal we have $M$ examples for each part where the attribute, type, or environment match the target, but where the other parts are sampled randomly. 
This is to simulate the scenario where the practitioner has access to informative but incomplete solutions a-priori, having some of the target characteristics, but not all. 
This, per setup and sampled prompt, yields a number of seeds of $K = 3M$. The experiment is repeated 10 times; i.e. we sample 10 targets and their corresponding seed examples independently, and apply each optimisation algorithm and weight chart combination to these targets, producing 10 corresponding runs for which we report the median and the 90th confidence interval.   
For optimiser setups, see~\ref{appendix:optimiser_setups}.

\textbf{Results.} Figure~\ref{fig:image_optimisation_trajectories} demonstrates that optimisers perform better within our 
surrogate spaces than in the full, original latent space, typically outperforming the best solutions found over a whole run of random search in the full space (i.e. standard sampling from the generative model) in just handful of evaluations. CMA-ES deployed in the original latent space (by specifying points in $\bm{u} \in [0, 1]^D$ which are subsequently mapped to latent distribution samples via the inverse Gaussian CDF) failed to produce anything but black images (see Figure~\ref{fig:image_optimisation_best_images}), which is not surprising as it is unlikely to find a point on the manifold of realistic latent realisations (see~\cite{bodin2024linear}). 
In Figure~\ref{fig:weight_chart_comparison} we report results using surrogate spaces with alternative choices of $\phi_{w}$ (see Section~\ref{appendix:weight_chart_comparison}), as well as all combinations of optimisers, including random search in surrogate spaces. 
Within very low-dimensional surrogate spaces (not the full space), random search was nearly as effective as BO and CMA-ES, but as the dimensionality increased (by providing more seeds) CMA-ES and BO performed substantially better.

\section{Optimiser setups}
\label{appendix:optimiser_setups}
For the image and \textsc{RFdiffusion} experiments:
\begin{itemize}
    \item \textbf{CMA-ES.} We use the implementation from~\cite{nomura2024cmaes} with population size $4$ and $\sigma = 0.2$. 
    \item \textbf{BO.} For $K = 3$ and $K = 9$ (i.e. 2D and 8D search problems), we use a Gaussian Process prior with a (3/2)-Matérn kernel and DEFER~\citep{bodin2021black} --- with a budget of $300$ density function evaluations and $30$ hyperparameter posterior samples --- for Bayesian inference for the kernel scale, lengthscale, and Gaussian (homoscedastic) noise variance parameters. For $K = 90$ (i.e. 89D search problems), we use Turbo~\citep{eriksson2019scalable} and the author's official implementation. 
    \item \textbf{Random search in $\mathcal{U}$}. Uniform, independent sampling in $\mathcal{U}$. 
    \item \textbf{Random search in $\mathcal{Z}$}. Standard random (and independent) sampling from the latent distribution. 
    \item \textbf{CMA-ES in $\mathcal{Z}$}. CMA-ES deployed on $[0, 1]^D$, where evaluations are mapped to latent distribution samples via the inverse Gaussian CDF.
\end{itemize}

\begin{figure*}[ht]
    \vspace{-1em}
    \centering
    \includegraphics[width=\linewidth]{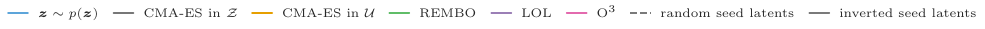}
    \includegraphics[width=\linewidth]{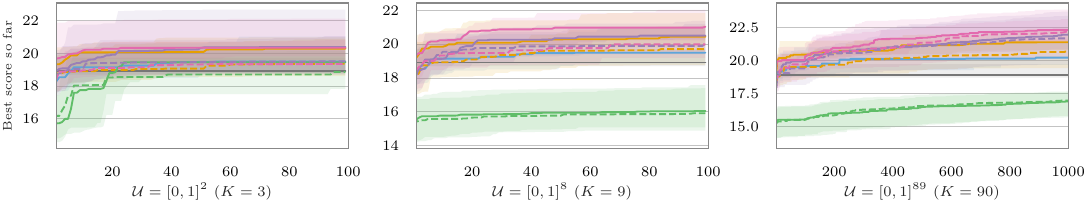}

    \caption{
    \textbf{Image optimisation: best score so far across function evaluations.}
    Per-iteration median (lines) and 90\% confidence interval (shaded, $5$th--$95$th      
    percentile across 10 runs) of the best score reached so far. The three panels                
    correspond to surrogate spaces of increasing dimension induced by $K \in                    
    \{3, 9, 90\}$ seed latents, giving $\mathcal{U} = [0,1]^2$, $[0,1]^8$, and                  
    $[0,1]^{89}$, respectively. Methods compared: random sampling                               
    $\bm{z} \sim p(\bm{z})$, CMA--ES applied directly in $\mathcal{Z}$,                     
    CMA--ES in $\mathcal{U}$, and BO in REMBO, LOL, and $\mathcal{U}$. For the latter four, dashed and solid lines indicate                   
    random and filtered seed latents, respectively. The optimisation budget is 100              
    function evaluations for the two low-dimensional panels and 1000 for the                    
    high-dimensional one.
    }
    \label{fig:image_optimisation_trajectories}
\end{figure*}

\section{Weight chart optimisation comparison}
\label{appendix:weight_chart_comparison}

In this section we report results for combinations of choices of $\phi_w$ (see Section~\ref{appendix:charts}). 

In Figure~\ref{fig:weight_chart_comparison} we see optimisation results for two different choices of $\phi_w$; the KR and the Angular chart, respectively, 
in the context of each of BO, CMA-ES, and random search within the formed surrogate space, and include random search (standard sampling) in the full latent space for reference. For optimiser setups, see~\ref{appendix:optimiser_setups}. 
In the context of each combination of choice for $\phi_w$, optimiser, and number of seeds, the surrogate spaces substantially outperform random sampling in the full latent space. 
Using few seeds, i.e. low-dimensional surrogate spaces, both choices of $\phi_w$ perform similarly, while for the relatively high-dimensional surrogate space ($89D$, formed from $90$ seeds) the KR chart substantially outperform the Angular chart when in the context of an optimisation algorithm (BO and CMA-ES) instead of uniform sampling within the surrogate space.
Although some differences in final score appear modest, they correspond to large differences in evaluation budget: in this low-budget regime, achieving even small additional improvements by random sampling would require many more model evaluations.

\begin{figure*}[ht]
\centering
\includegraphics[width=0.8\linewidth]{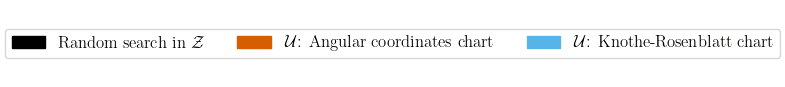}

\captionsetup[subfigure]{labelformat=empty,aboveskip=0pt,belowskip=2pt}
\setlength{\tabcolsep}{0pt}\renewcommand{\arraystretch}{0.95}

\resizebox{0.9\linewidth}{!}{
\begin{tabular}{@{} >{\centering\arraybackslash}m{0.08\textwidth}
                    *{3}{>{\centering\arraybackslash}m{0.29\textwidth}} @{}}
    & {\small 3 seed latents} & {\small 9 seed latents} & {\small 90 seed latents} \\[0.2em]
    \rotatebox{90}{\small {BO}} &
  \begin{subfigure}[t]{\linewidth}\includegraphics[width=\linewidth]{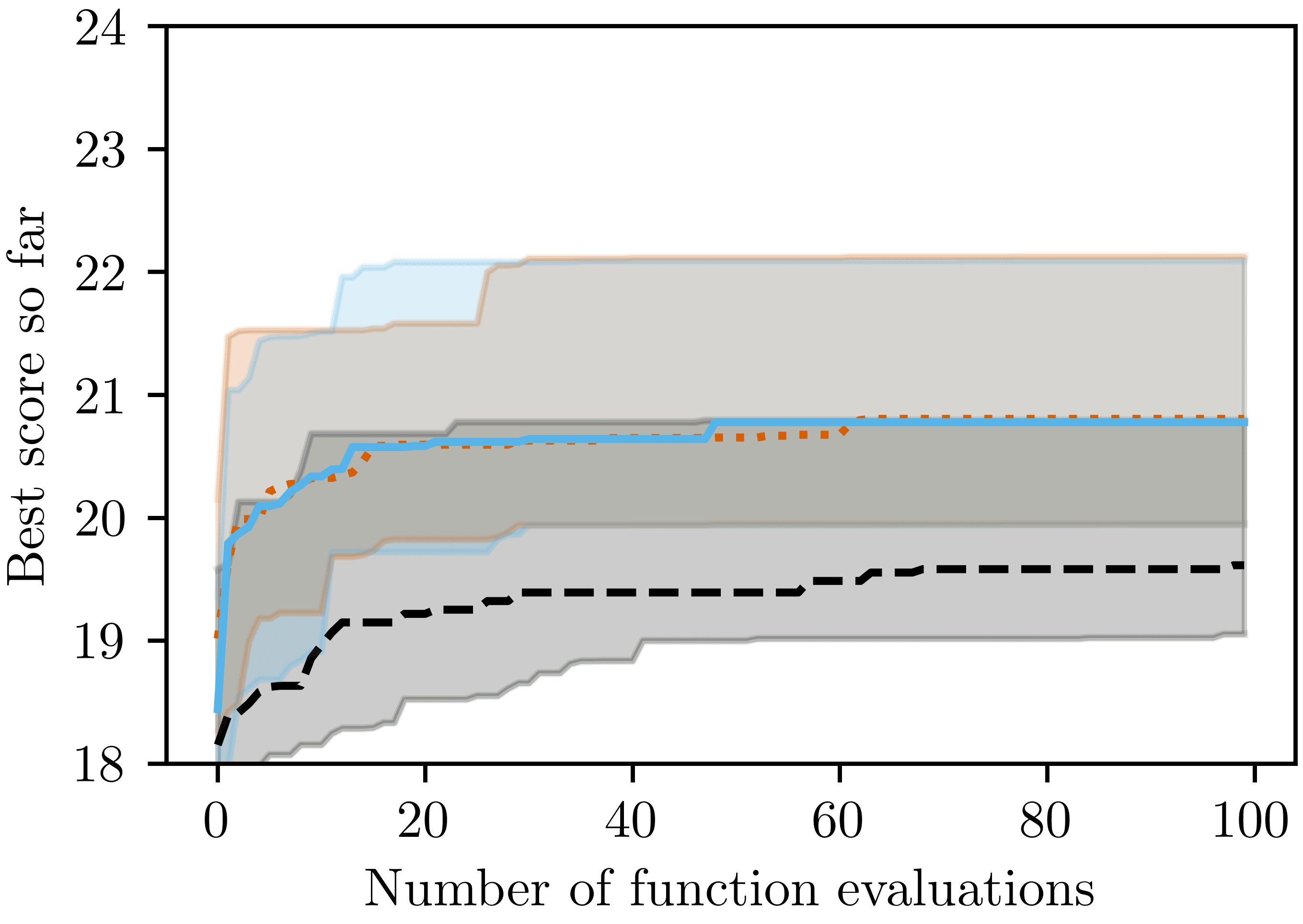}\caption{}\end{subfigure} &
  \begin{subfigure}[t]{\linewidth}\includegraphics[width=\linewidth]{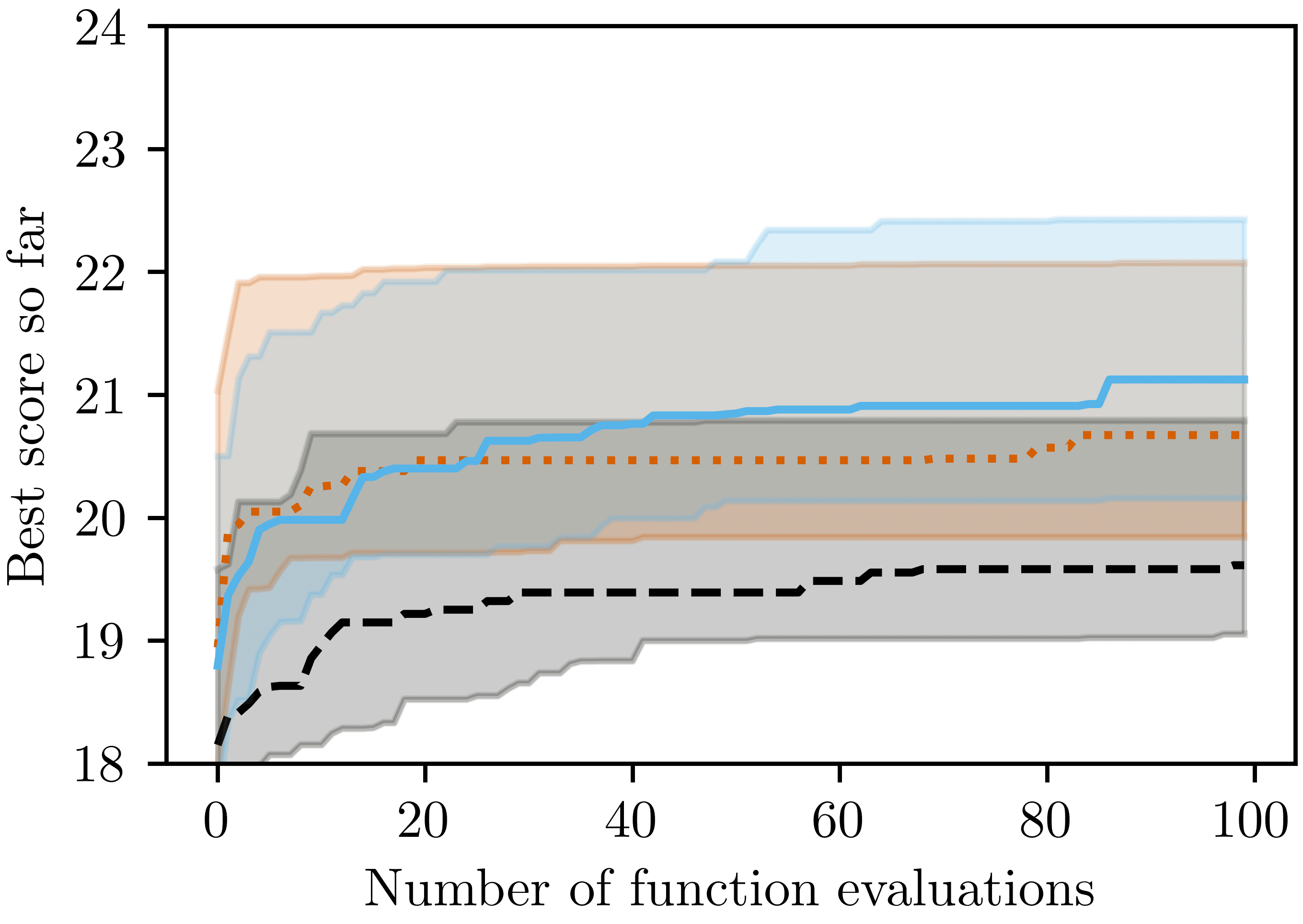}\caption{}\end{subfigure} &
  \begin{subfigure}[t]{\linewidth}\includegraphics[width=\linewidth]{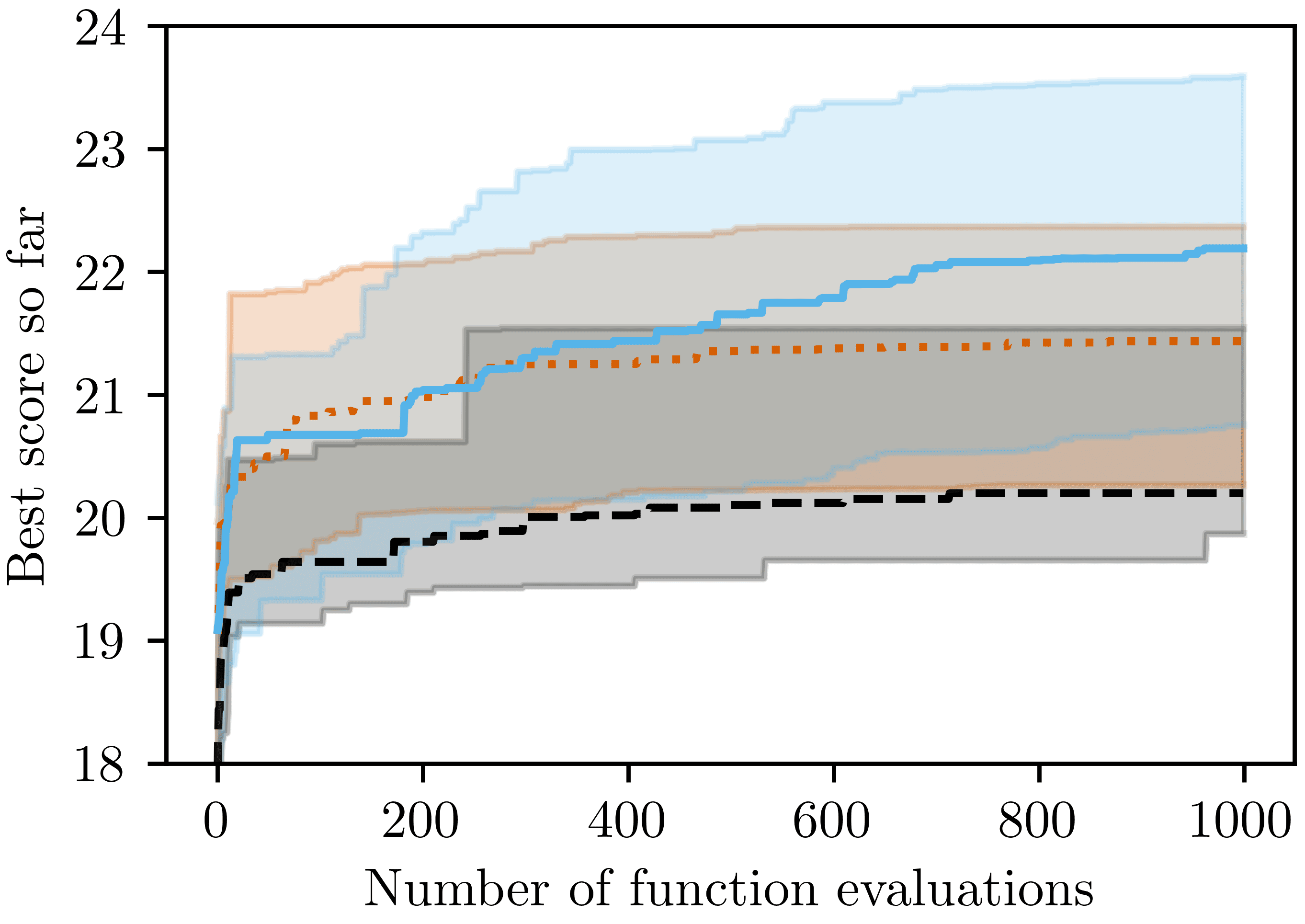}\end{subfigure} \\[0.3em]
    \rotatebox{90}{{\small CMA-ES}} &
      \begin{subfigure}[t]{\linewidth}\includegraphics[width=\linewidth]{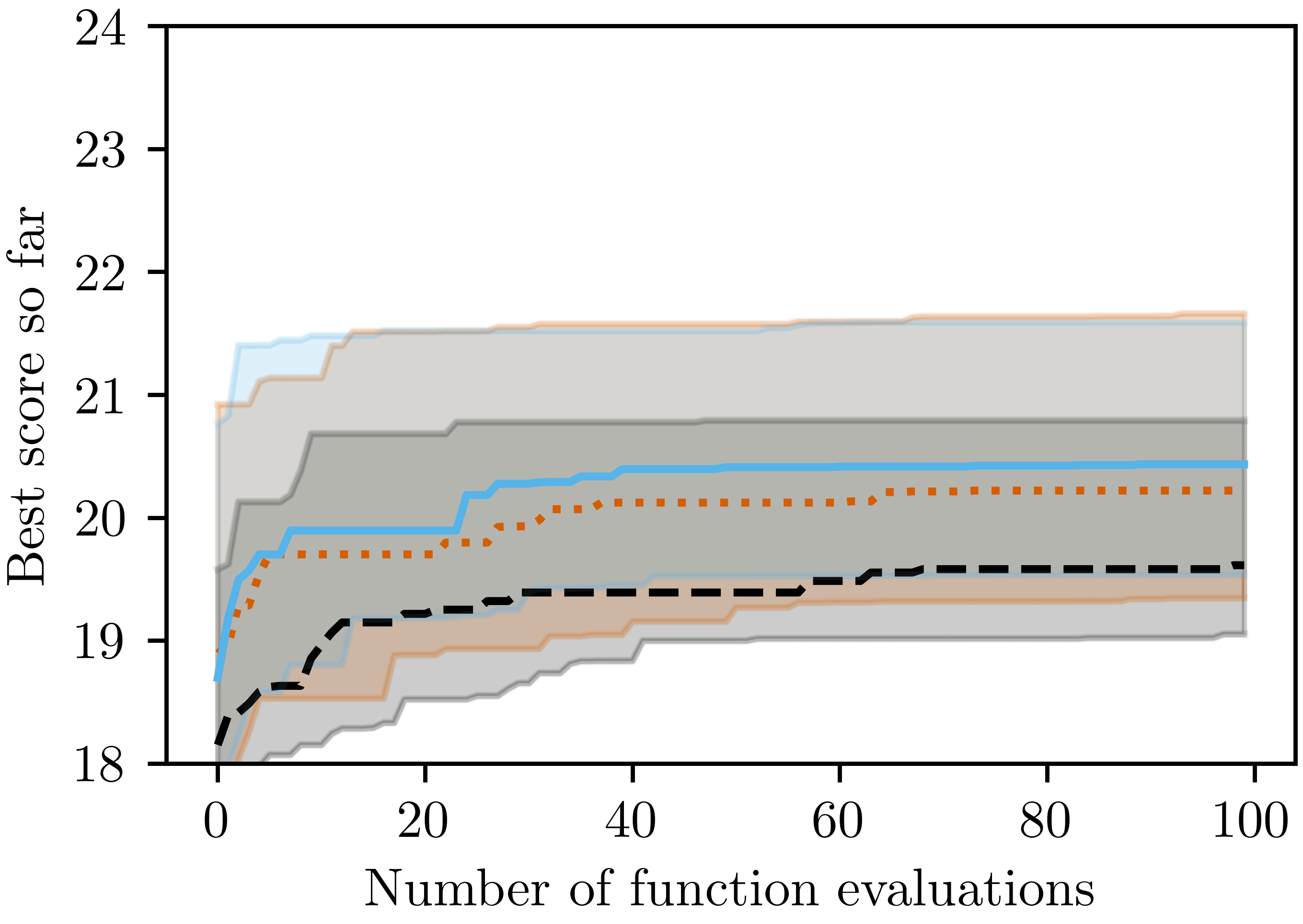}\caption{}\end{subfigure} &
      \begin{subfigure}[t]{\linewidth}\includegraphics[width=\linewidth]{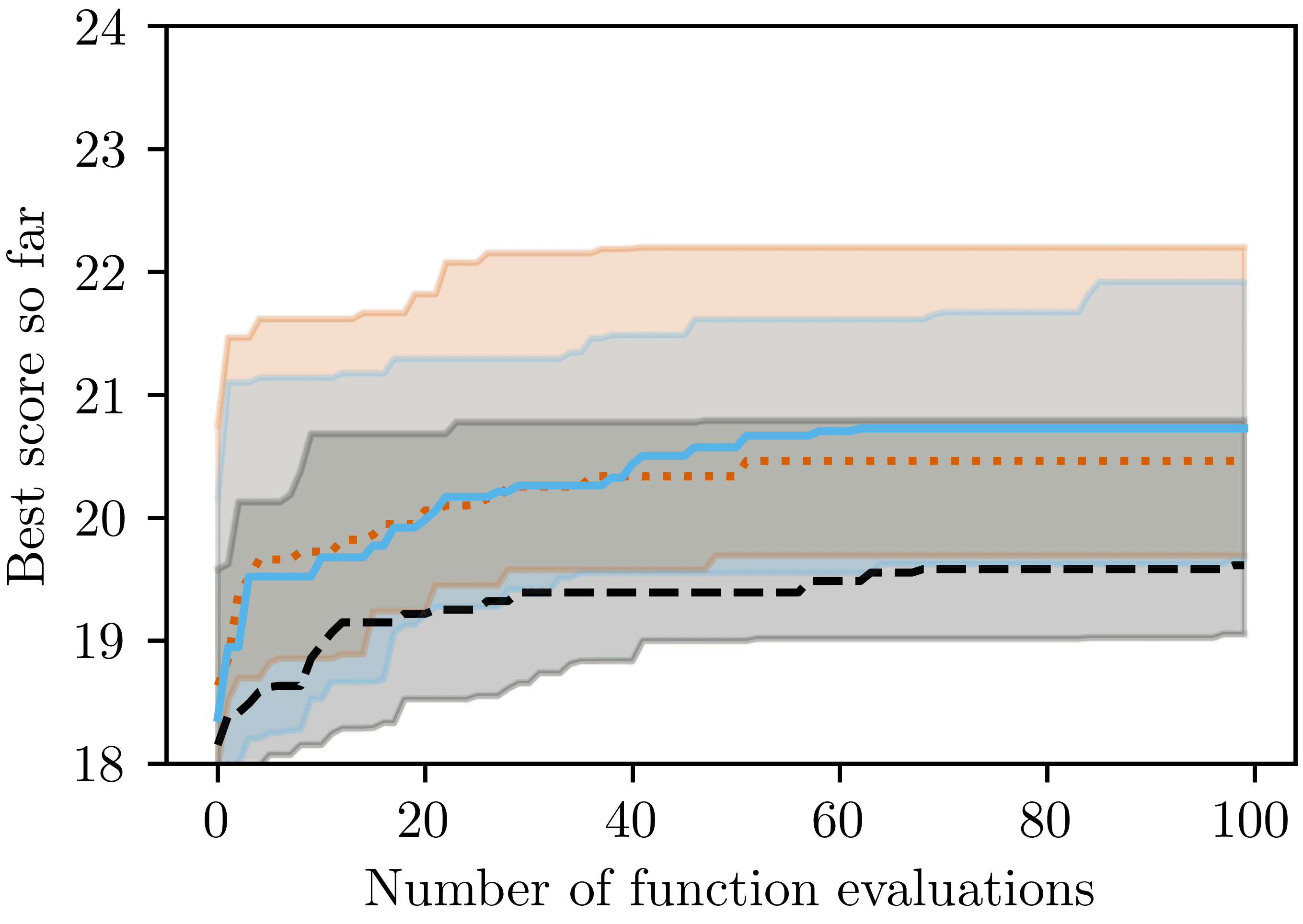}\caption{}\end{subfigure} &
      \begin{subfigure}[t]{\linewidth}\includegraphics[width=\linewidth]{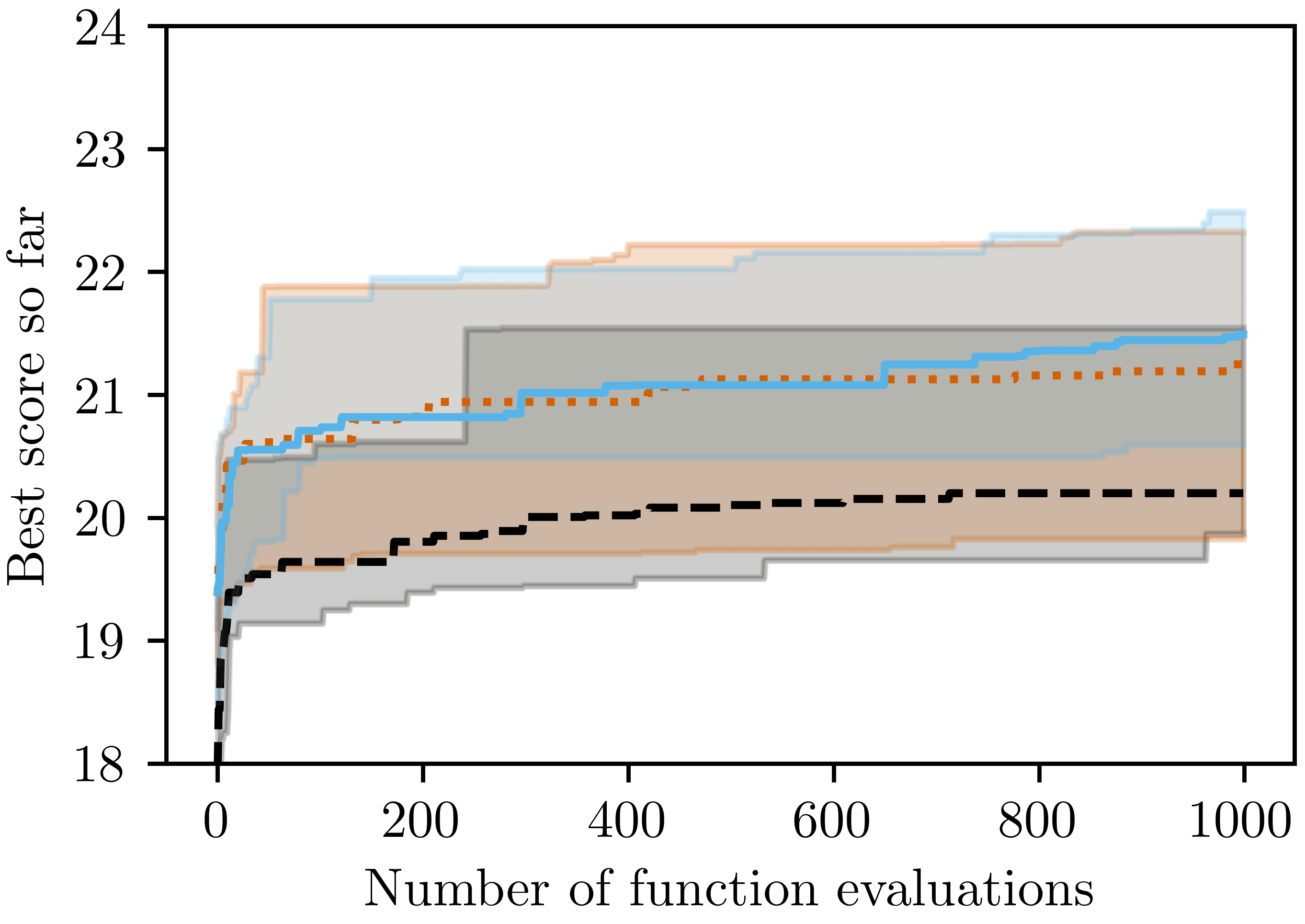}\caption{}\end{subfigure} \\[0.3em]
    \rotatebox{90}{{\small RS in~~$\mathcal{U}$}} &
      \begin{subfigure}[t]{\linewidth}\includegraphics[width=\linewidth]{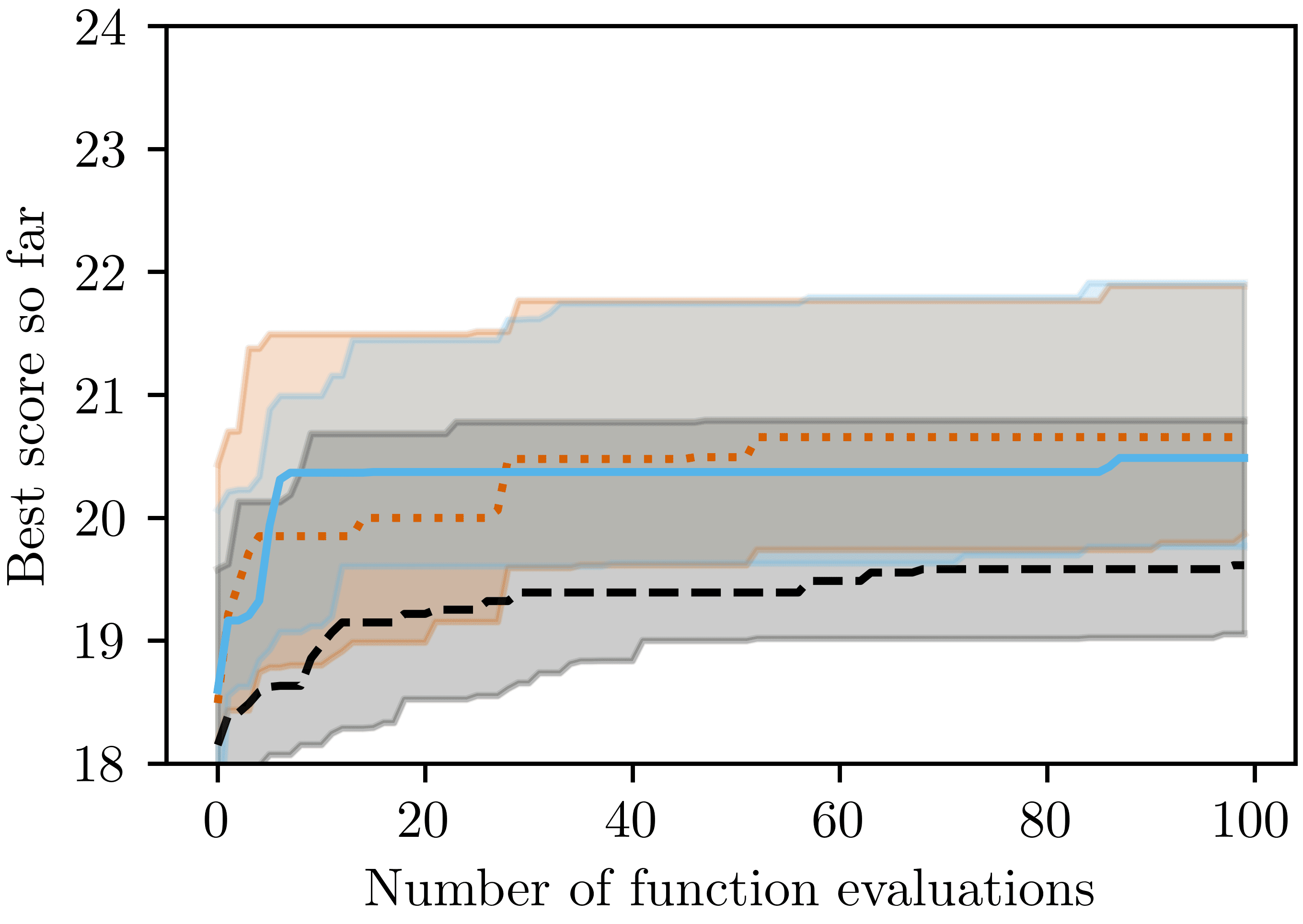}\caption{}\end{subfigure} &
      \begin{subfigure}[t]{\linewidth}\includegraphics[width=\linewidth]{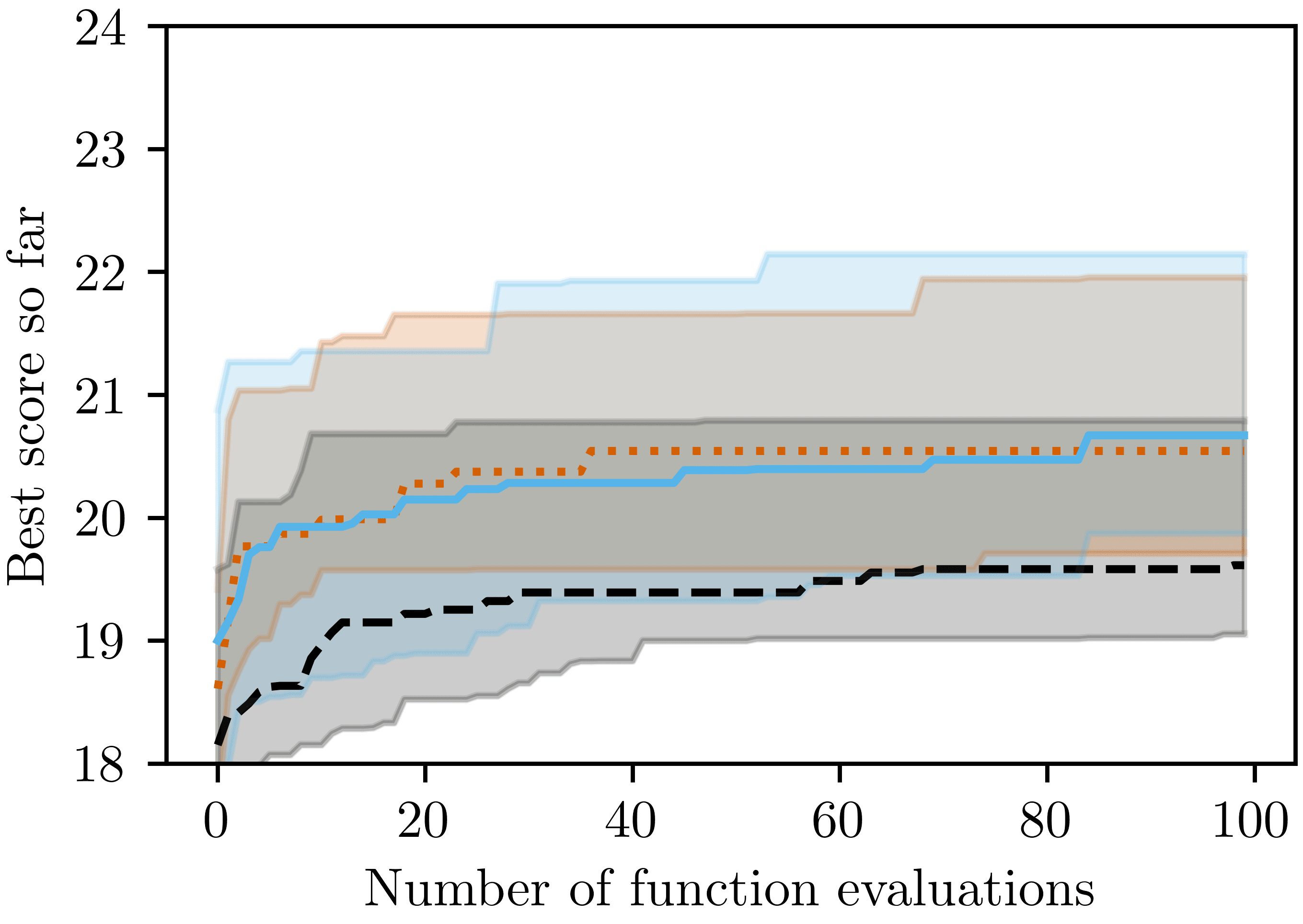}\caption{}\end{subfigure} &
      \begin{subfigure}[t]{\linewidth}\includegraphics[width=\linewidth]{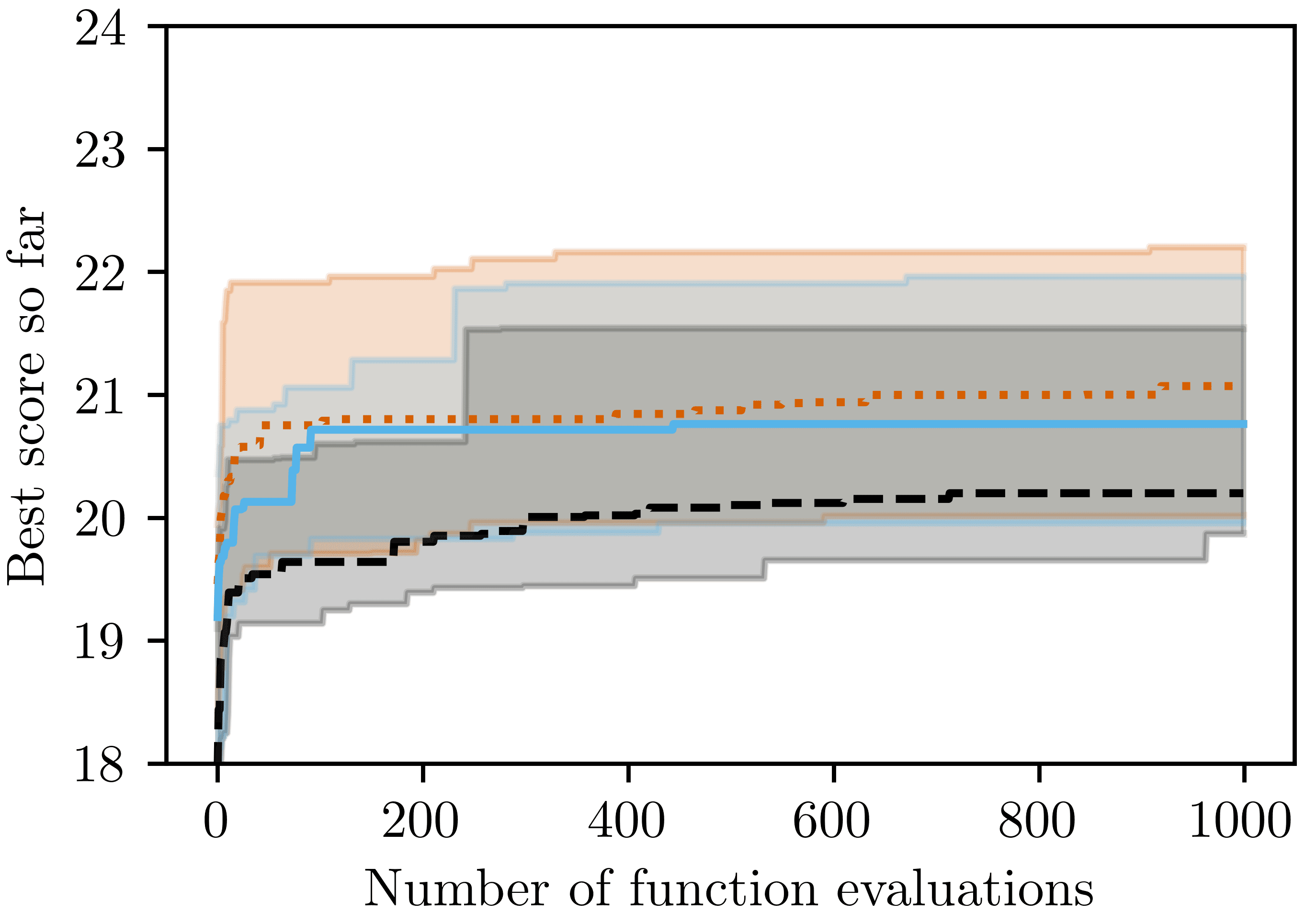}\caption{}\end{subfigure}
\end{tabular}
}

\caption{
\textbf{
Chart comparison across optimisers and surrogate space dimensionalities
}
Shown is the median and 90\% confidence interval of the best-so-far score found per step across runs, on the task described in Figure~\ref{fig:ours_vs_random} and Section~\ref{appendix:image_opt_details}, for surrogate spaces with dimensionality $(K - 1)$, where $K$ is the number of seeds (examples) provided.  
}
\label{fig:weight_chart_comparison}
\end{figure*}

\section{\textsc{Boltz-2} experiment details}
\label{app:boltz}

\paragraph{Base model.}
\textsc{Boltz-2}~\citep{passaro2025boltz} is a protein-structure prediction model that maps an amino-acid sequence to atomic 3D coordinates via a Pairformer trunk followed by a diffusion decoder, trained in the EDM framework~\citep{karras2022elucidating}. By default, the decoder samples from a stochastic reverse SDE with per-step noise injection controlled by the $\gamma_0$ hyperparameter and stochastic SE(3) coordinate augmentation at each step.

\paragraph{Probability-flow ODE conversion.}
\methodname{} requires a deterministic generative map so that latent samples can serve as a basis for the surrogate space $\mathcal{U}$. We adapt \textsc{Boltz-2}'s stochastic sampler to its equivalent probability-flow ODE within the EDM framework by setting $\gamma_0 = 0$ and disabling stochastic SE(3) augmentation, yielding a single deterministic Euler integrator from $\sigma_{\max}$ to $\sigma_{\min}$. We further reparametrise the latent prior to a unit Gaussian $\bm{z} \sim\mathcal{N}(0, I)$; the composite-latents construction $\bm{z}^* = \bm{Z}\phi_w(\bm{u})$ then applies directly with no additional transport map. The resulting generative map $g: \mathcal{Z}\to \mathcal{X}$ is deterministic and bijective on its image.

\paragraph{Oracle.}
We score generated structures with the TM-score~\citep{zhang2004tmscore} against the PDB 1CLL crystal structure of calmodulin~\citep{chattopadhyaya1992calmodulin}, computed on $\mathrm{C}_\alpha$ atoms using the \texttt{tmtools} package~\citep{zhang2004tmscore}. The score is normalised by the reference chain length ($n = 144$ for 1CLL). See Appendix~\ref{app:tm_score} for the TM-score definition.

\paragraph{Bayesian optimisation setup.}
At each Bayesian optimisation (BO) acquisition round, we fit a Gaussian process (GP) surrogate~\citep{williams2006gaussian} on the points acquired so far in $\mathcal{U}$. We use an RBF kernel, a constant mean function, and a single-task GP from BoTorch~\citep{balandat2020botorch}. Acquisitions are driven by Log Expected Improvement~\citep{ament2023unexpected}, sampled via BoTorch's Monte Carlo acquisition sampler. We project the $K$ seed latents into $\mathcal{U}$ via the surrogate chart $\phi_w$ and include them as additional training points alongside an initial set of $j$ uniformly drawn points in $\mathcal{U}$.

\paragraph{Visualisations.}
Figure~\ref{fig:boltz_lol_interp} shows a LOL interpolation between two seed structures, illustrating that the surrogate space supports smooth structural variation over \textsc{Boltz-2} generations. Figure~\ref{fig:boltz_2d_surrogate} shows a 2D surrogate space constructed from $K=3$ seed latents, with the GP posterior mean and the acquired points after two BO rounds.

 \begin{figure*}[h]
      \centering
      \includegraphics[width=0.75\textwidth]{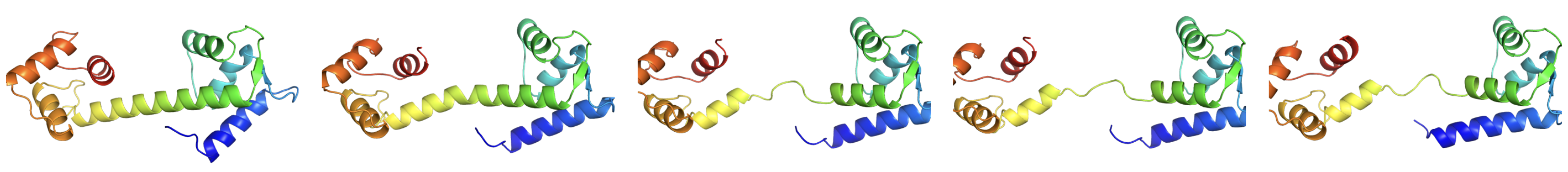}
      \caption{\textbf{LOL interpolation in a \textsc{Boltz-2} surrogate space.} Five intermediate structures along the LOL line between two calmodulin seeds (i.e.\ $K=2$ seed latents, surrogate space $\mathcal{U}^1$). Structures vary smoothly while remaining
  within the support of \textsc{Boltz-2}.}
      \label{fig:boltz_lol_interp}
  \end{figure*}

  \begin{figure*}[h]
      \centering
      \includegraphics[width=0.85\textwidth]{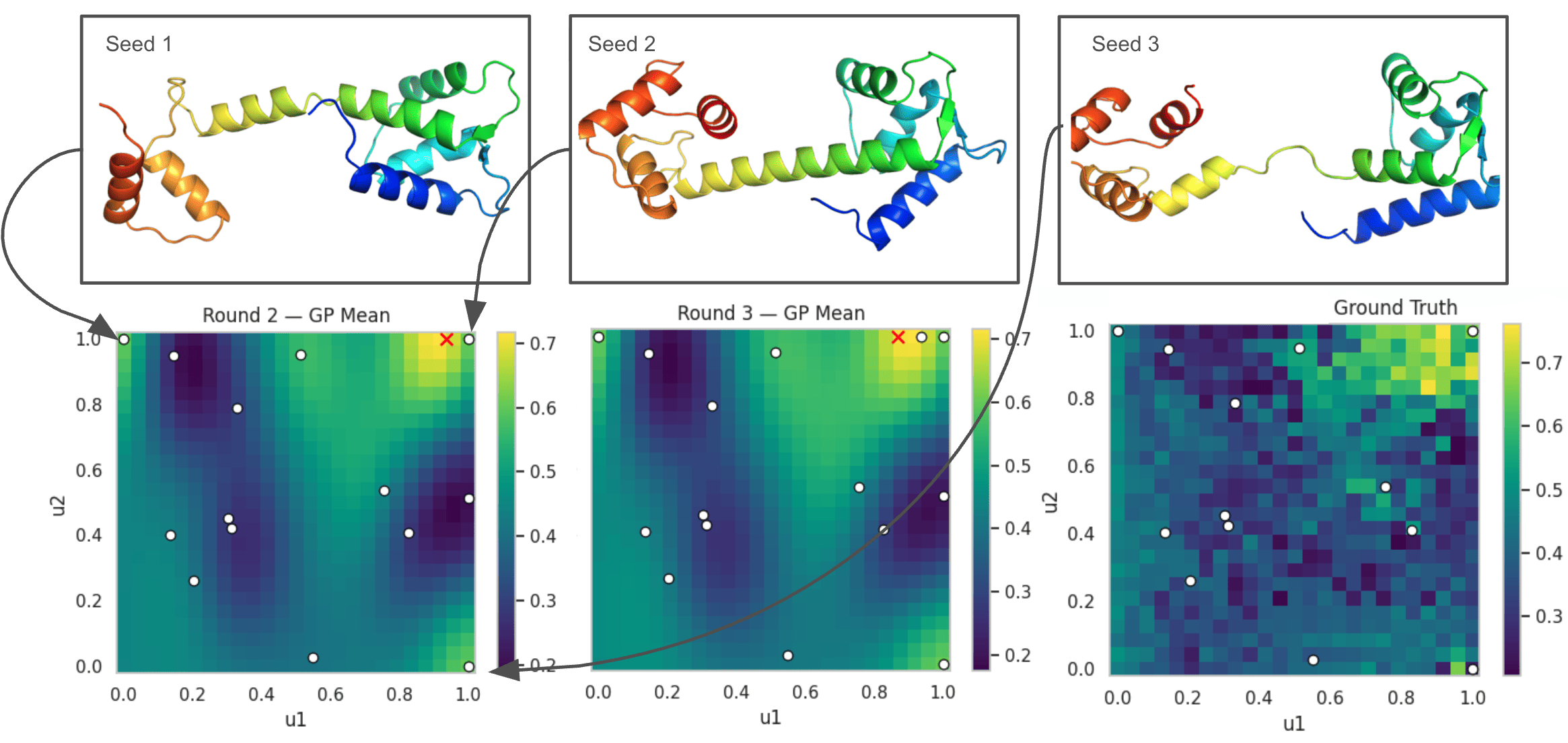}
      \caption{\textbf{Bayesian optimisation in a \textsc{Boltz-2} surrogate space defined by $K=3$ seed latents.} Three calmodulin seed structures (top), the resulting 2D surrogate space $\mathcal{U}^2$ (bottom), the GP training data (white dots) and predictive posterior mean
  over $\mathcal{U}^2$, the chosen acquisition point (red cross) for two BO rounds, and a $25{\times}25$ exhaustive ground-truth grid (bottom right).}
      \label{fig:boltz_2d_surrogate}
  \end{figure*}

\paragraph{Ablations.}
Below we report three ablations of the \methodname{} configuration on the 1CLL conformation-biasing task.
  
Figure~\ref{fig:boltz_ablate_K} sweeps the number of seed latents $K$ (which yields a surrogate space of dimension $K-1$) for each oracle-call budget $N$. The best-performing $K$ varies with the budget: at $N=20$ the strongest configuration uses $K=3$ seed latents, while at $N=500$ the strongest is $K=7$. Different oracle-call regimes favour different numbers of seed latents.

\begin{figure}[h]
  \centering
  \includegraphics[width=0.7\textwidth]{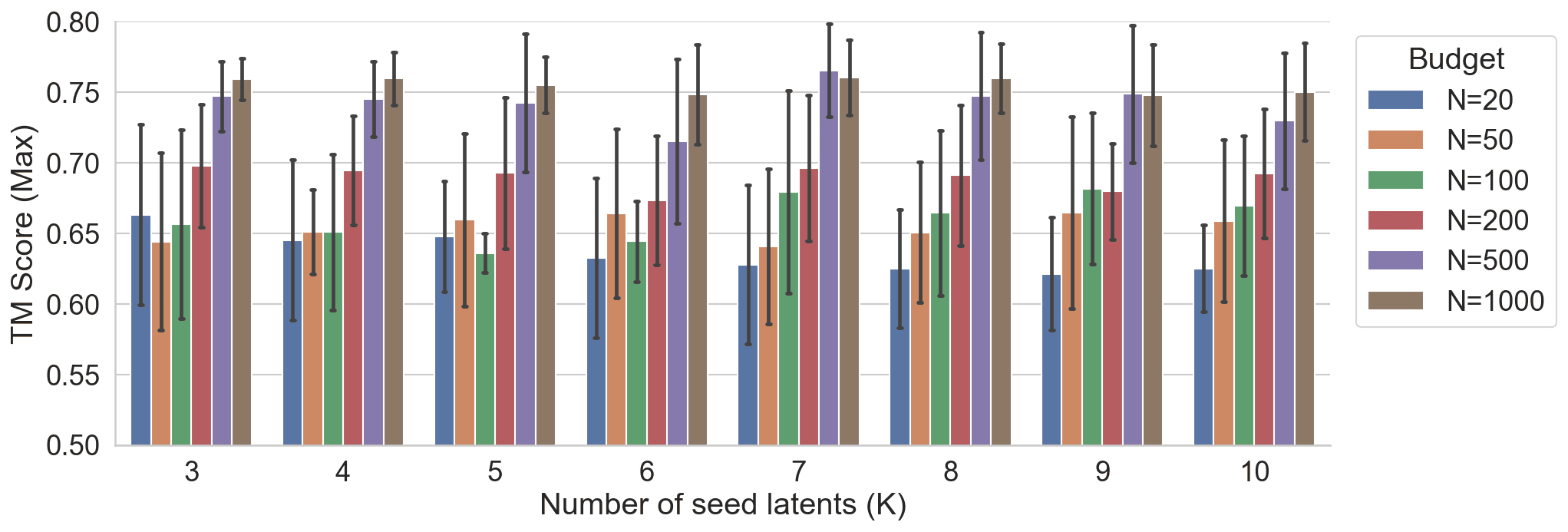}
  \caption{\textbf{Number of seed latents $K$ across budgets.} Max TM-score for \methodname{} with Bayesian optimisation, across number of seed latents $K$ for different budgets $N$, on 1CLL. The best-performing $K$ varies with $N$. Bars are means over 5 repetitions; error bars are $\pm 1$ std.}
  \label{fig:boltz_ablate_K}
\end{figure}

Figure~\ref{fig:boltz_ablate_bo} compares \methodname{} with Bayesian optimisation against \methodname{} with uniform random sampling in the same surrogate space at $N=100$. Random sampling alone fails to outperform Best-of-$N$, while BO does. The gain over Best-of-$N$ comes from the optimisation step in $\mathcal{U}$ rather than from the search space construction itself.

\begin{figure}[h]
  \centering
  \includegraphics[width=0.7\textwidth]{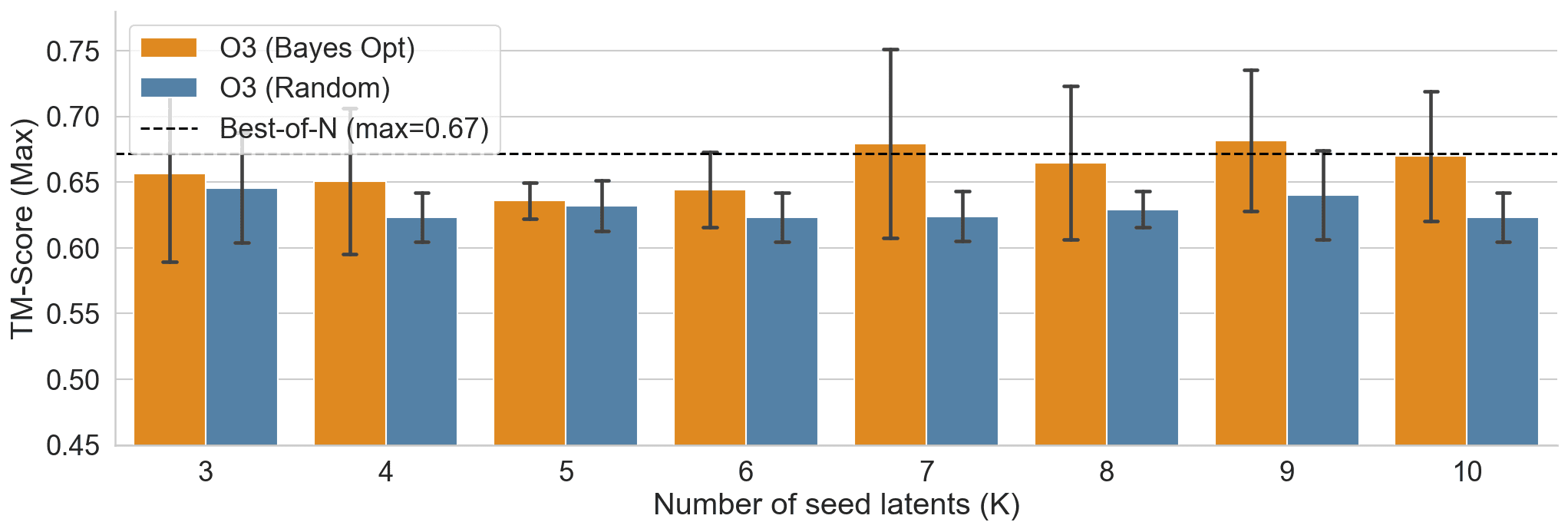}
  \caption{\textbf{Bayesian optimisation versus random sampling in the surrogate space.} Max TM-score at $N=100$ for \methodname{} with BO and \methodname{} with uniform random sampling in $\mathcal{U}$, across number of seed latents $K$. The Best-of-$N$ baseline is shown as a horizontal line. Bars are means over 5 repetitions; error bars are $\pm 1$ std.}
  \label{fig:boltz_ablate_bo}
\end{figure}

Figure~\ref{fig:boltz_ablate_seeds} compares two strategies for choosing the $K$ seed latents that define $\mathcal{U}$: \emph{filtered}, which samples $m$ initial structures, scores them, and keeps the top $K$; and \emph{random}, which draws $K$ latents directly from the prior. Filtered seed latents give a small but consistent gain at fixed $K$; random seed latents remain a useful default when the additional $m$ initial scorings are too expensive.

\begin{figure}[h]
  \centering
  \includegraphics[width=0.7\textwidth]{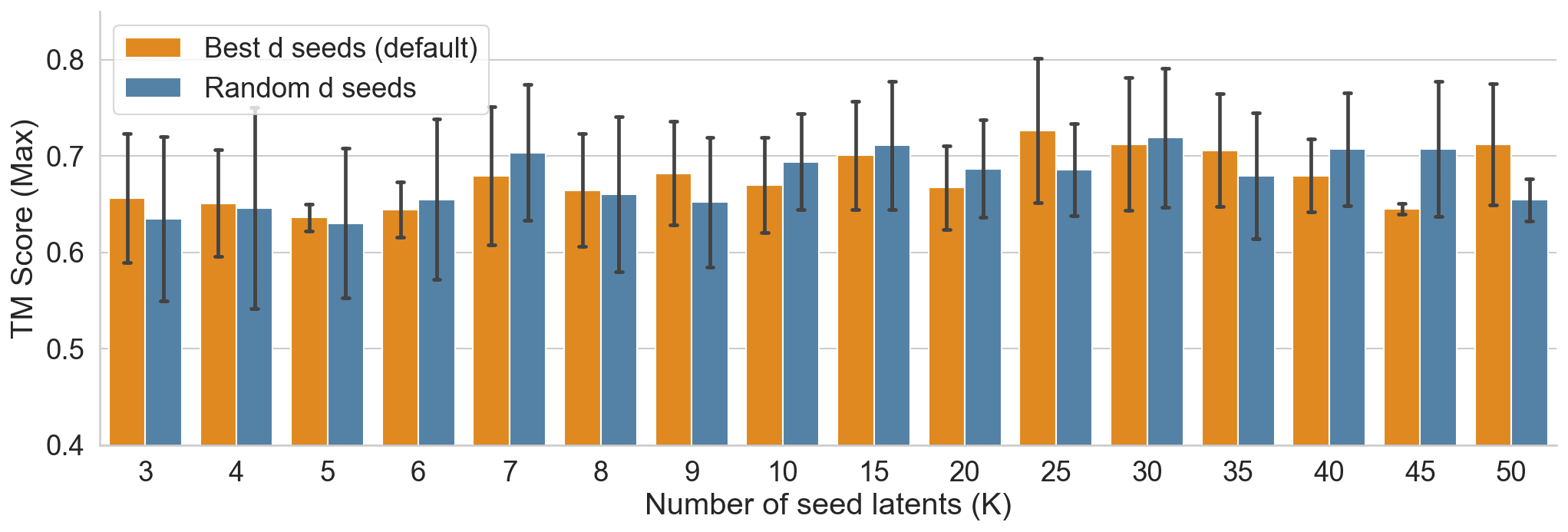}
  \caption{\textbf{Seed-selection strategy.} Max TM-score for two seed-selection strategies (filtered and random) across number of seed latents $K$ at $N=100$. Bars are means over 3 repetitions; error bars are $\pm 1$ std.}
  \label{fig:boltz_ablate_seeds}
\end{figure}

\paragraph{FK-steering configuration.}
We use $K$ particles and run Sequential Monte Carlo with reward-based weights computed at intermediate denoising steps where \textsc{Boltz-2} injects positive noise; the final, deterministic step is excluded since reward differences there carry no resampling signal. The reward-scale hyperparameter is $\lambda=50$. We set the proposal-kernel to the original \textsc{Boltz-2} transition kernel to leverage gradient-free nature of FK-steering, since most biological oracles are non-differentiable. The total oracle-call budget $N$ is distributed throughout denoising: one call at the start of denoising, one at the end, and $N{-}2$ calls distributed across resampling steps in the first two-thirds of the noise-injecting trajectory.

\paragraph{Online DPO configuration.}
At each of $E$ epochs we generate $N/E$ structures from the current model, score them with the oracle, form preference pairs by opposing structures with above-median scores against those with below-median scores, and update the model for one epoch on the resulting pairs. The reference model $p_{\mathrm{ref}}$ is reset to the current policy at the start of each epoch, which softens the KL constraint and lets the policy drift further from the pre-trained \textsc{Boltz-2} checkpoint. We use the Diffusion-DPO loss~\citep{wallace2023diffusion} with denoising-score-matching loss differences as a tractable surrogate for the log-likelihood ratio.


\section{\textsc{RFdiffusion}}
\label{appendix:rfdiff}

We adopt the pipeline of \citet{watson2023rfdiffusion}, consisting of:  
(1) backbone generation with \textsc{RFdiffusion},  
(2) sequence design with \textsc{ProteinMPNN}~\citep{dauparas2022proteinmpnn},  
(3) structure reconstruction with \textsc{AlphaFold2}~\citep{pak2023alphafold}, and  
(4) evaluation by C$_\alpha$–frame RMSE. Lower RMSE indicates closer agreement between the generated and reconstructed backbones. 

In step~1, candidate backbones are sampled from the original \textsc{RFdiffusion} model using DDIM. A backbone is defined as the set of C$_\alpha$ coordinates and residue-wise rotations, but does not include categorical amino acid identities.  
In step~2, each backbone is completed with $M=8$ amino acid sequences predicted by \textsc{ProteinMPNN}. This introduces the missing categorical information; however, the predictions are noisy, motivating multiple samples.  
In step~3, the sequences are passed to \textsc{AlphaFold2}, which reconstructs 3D structures from sequence alone, testing whether the backbone proposed by \textsc{RFdiffusion} is compatible with realistic sequences.  
In step~4, reconstructed proteins are aligned to the original backbones, and C$_\alpha$ RMSE is computed. For each backbone we report the best sequence (minimum RMSE over $M=8$), following the evaluation protocol of \citet{watson2023rfdiffusion}. 
For optimiser setups, see~\ref{appendix:optimiser_setups}

As in \citet{watson2023rfdiffusion} we adopt a threshold of $T=2.0$\,Å RMSE to define successful recovery, however we drop their secondary filtering metric of designs having PAE $<5.0$ to focus on proof of principle, although in future this could naturally be supported by considering multi-objective optimisation. For fairness, all baselines were recomputed under our evaluation. Each optimisation run used 200 function evaluations (50 CMA-ES steps of population 4), twice the 100 generations of the original paper. 

\textsc{RFdiffusion} parametrises a backbone of length $N$ by residue-wise frames $(\bm{x}_{pos}^{(t)}, \bm{x}_{rot}^{(t)}) \!\in\! \mathbb{R}^{3N} \times \mathrm{SO}(3)^N$, where $\bm{x}_{pos}^{(t)}$ are C$_\alpha$ coordinates and $\bm{x}_{rot}^{(t)}$ are orientations derived from N–C$_\alpha$–C triplets, measured from a reference frame. The forward diffusion process applies Gaussian noise to $\bm{x}_{pos}$ and Brownian motion on $\bm{x}_{rot}$; generation is by reverse integration of the probability–flow ODE. The resulting latent is
\[
\bm{z} = (\bm{z}_{pos}, \bm{z}_{rot}) = (\bm{x}_{pos}^{(T)}, \bm{x}_{rot}^{(T)}),
\quad \bm{z}_{pos}\!\sim\!\mathcal{N}(\mathbf{0},\mathbf{I^{3N}}), \;
\bm{z}_{rot}\!\sim\!\mathrm{Unif}(\mathrm{SO}(3)^N).
\]
Because $\bm{z}_{pos}$ follows a Gaussian distribution and we parametrise $\bm{z}_{rot}$ as quaternions which are uniformly distributed on $\bm{z}_{rot}\!\sim\!\mathrm{Unif}({\mathbb{S}^3}^{N})$, we can directly apply composite latents from Section~\ref{appendix:non_gaussian} to construct surrogate latent spaces $\mathcal{U}$.

\begin{figure*}[ht]
\centering




\vspace{0.2em}
\begin{minipage}[b]{0.245\textwidth}
  \centering
  {\small $\bm{x}_1$}\\[0.2em]
  \begin{tikzpicture}[remember picture]
    \node[inner sep=0] (seed1)
      {\includegraphics[width=0.9\textwidth]{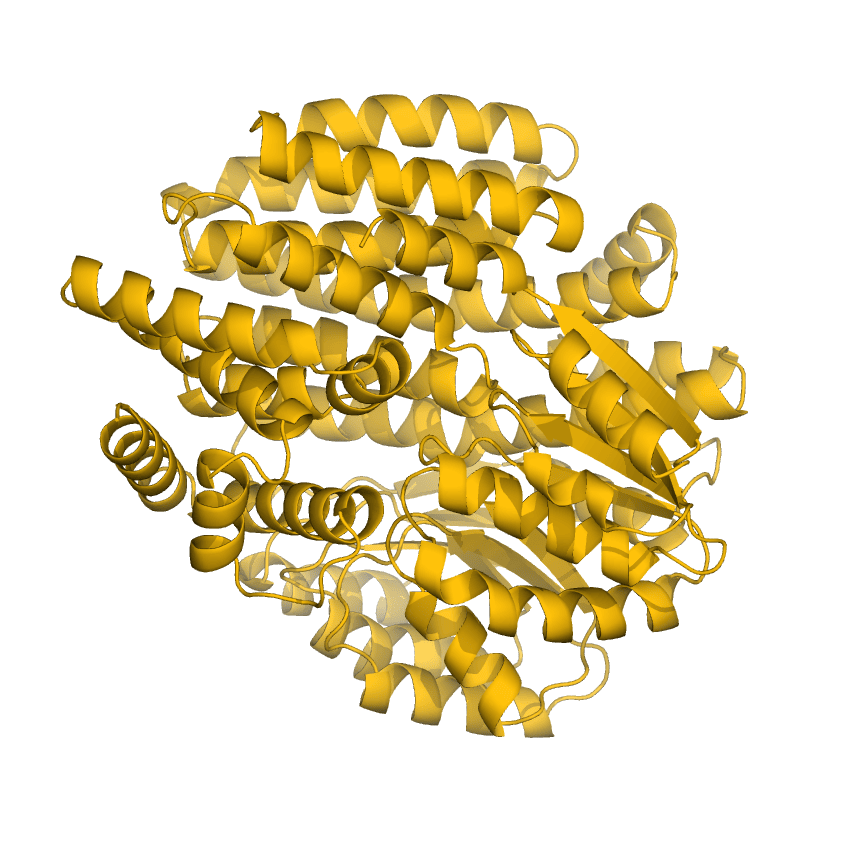}};
  \end{tikzpicture}\\[-0.2em]
\end{minipage}
\begin{minipage}[b]{0.245\textwidth}
  \centering
  {\small $\bm{x}_2$}\\[0.2em]
  \begin{tikzpicture}[remember picture]
    \node[inner sep=0] (seed2)
      {\includegraphics[width=0.9\textwidth]{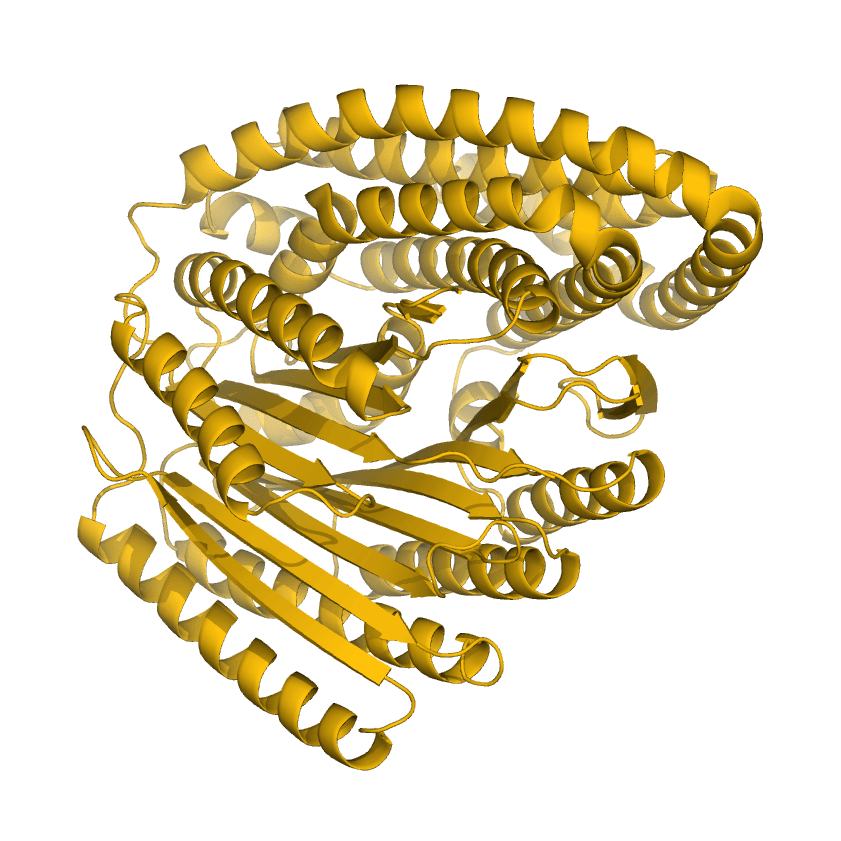}};
  \end{tikzpicture}\\[-0.2em]
\end{minipage}
\begin{minipage}[b]{0.245\textwidth}
  \centering
  {\small $\bm{x}_3$}\\[0.2em]
  \begin{tikzpicture}[remember picture]
    \node[inner sep=0] (seed3)
      {\includegraphics[width=0.9\textwidth]{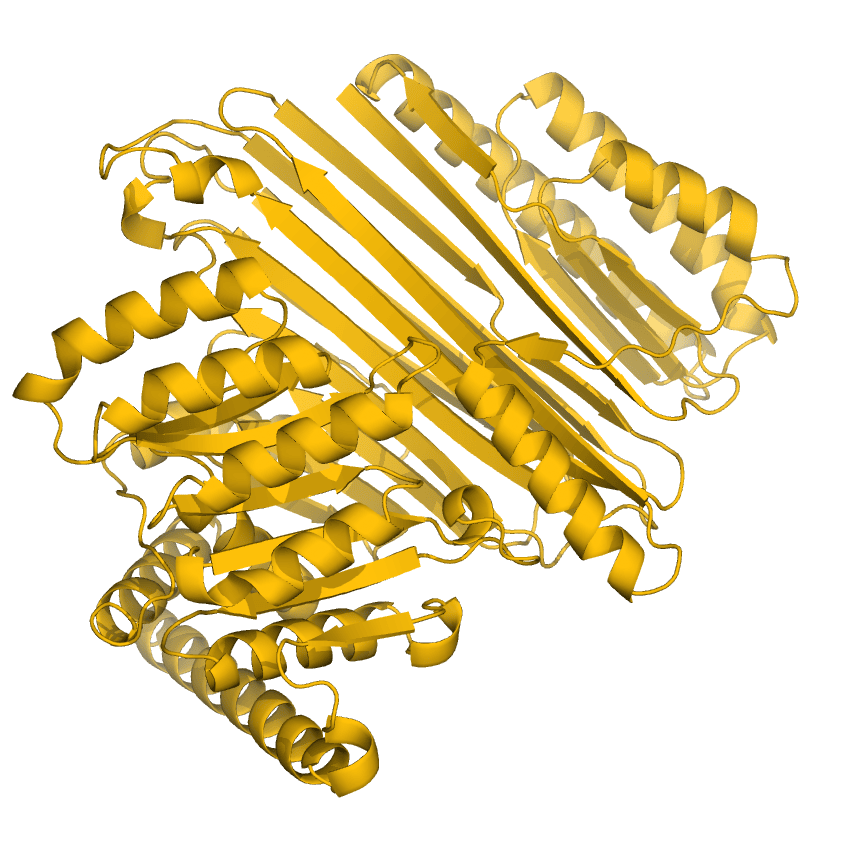}};
  \end{tikzpicture}\\[-0.2em]
\end{minipage}
\vskip\baselineskip

\def\axgap{4pt}        
\def\axlblgapx{2pt}    
\def\axlblgapy{6pt}    
\def\arrowpad{2pt}
\def\axlbl{\scriptsize}

\begin{minipage}[b]{0.395\textwidth}
  \hspace*{-2em}
  \centering
  \begin{tikzpicture}[baseline=(img.south)]
    \node[anchor=south west, inner sep=0, outer sep=0] (img) at (0,0)
      {\includegraphics[width=\textwidth]{compressed_figures/proteins/protein_grid.jpg}};
    \path let \p1=(img.south west), \p2=(img.north east) in
      coordinate (SW) at (\x1,\y1)
      coordinate (SE) at (\x2,\y1)
      coordinate (NW) at (\x1,\y2);
    \draw[-{Stealth}, line width=0.3pt]
      ($(SW)+(\arrowpad,-\axgap)$) -- ($(SE)+(-\arrowpad,-\axgap)$)
      node[midway, anchor=center, below=\axlblgapx] {\axlbl $u_1$};
    \draw[-{Stealth}, line width=0.3pt]
      ($(SW)+(-\axgap,\arrowpad)$) -- ($(NW)+(-\axgap,-\arrowpad)$)
      node[pos=0.55, anchor=center, left=\axlblgapy, rotate=90] {\axlbl $u_2$};
  \end{tikzpicture}
\end{minipage}
\hspace{1em}
\begin{minipage}[b]{0.46\textwidth}
  \centering
  \begin{tikzpicture}[baseline=(img.south), remember picture]
    \node[anchor=south west, inner sep=0, outer sep=0] (img) at (0,0)
      {\includegraphics[width=\textwidth]{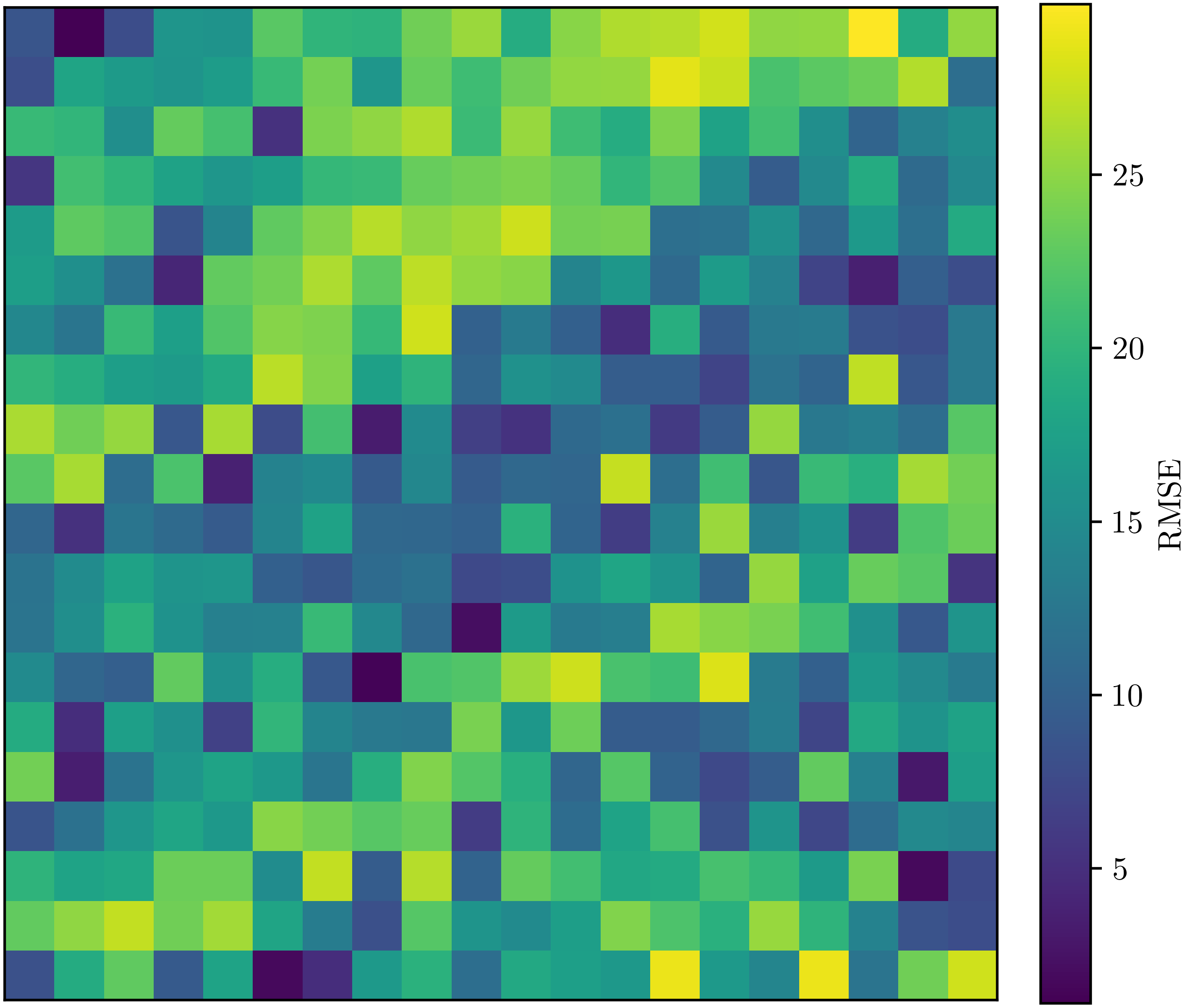}};
    \path let \p1=(img.south west), \p2=(img.north east) in
      coordinate (SW) at (\x1,\y1)
      coordinate (SE) at (\x2,\y1)
      coordinate (NW) at (\x1,\y2);
    \def\cbarfrac{0.17}
    \pgfmathsetmacro{\plotfrac}{1-\cbarfrac}
    \coordinate (SEsq) at ($(SW)!\plotfrac!(SE)$);
    \draw[-{Stealth}, line width=0.3pt]
      ($(SW)+(\arrowpad,-\axgap)$) -- ($(SEsq)+(-\arrowpad,-\axgap)$)
      node[midway, anchor=center, below=\axlblgapx] {\axlbl $u_1$};
    \draw[-{Stealth}, line width=0.3pt]
      ($(SW)+(-\axgap,\arrowpad)$) -- ($(NW)+(-\axgap,-\arrowpad)$)
      node[pos=0.55, anchor=center, left=\axlblgapy, rotate=90, xshift=0.6pt] {\axlbl $u_2$};
  \end{tikzpicture}
\end{minipage}

\caption{
\textbf{2D surrogate space for proteins.} A surrogate space $\mathcal{U}^2$ defined by $K=3$ seed latents (top) yields a structured objective landscape. \textit{Left:} grid of generated backbones across $\mathcal{U}^2$. \textit{Right:} corresponding evaluation scores (C$_\alpha$–frame RMSE).
}
\label{fig:protein_grid}
\end{figure*}

Figure~\ref{fig:protein_grid} shows a 2-dimensional latent space formed from K=3 seed latents, over which a grid of protein structures have been generated and evaluated according to the target objective. Clear structure in shown in the objective space which makes this objective amenable to optimisation. 

For the protein optimisation experiments we set the number of seed latents to $K=24$. Optimisation is performed in $\mathcal{U}$ via CMA–ES, with candidates mapped back into $\mathcal{Z}$ for decoding and evaluation. Two seed selection strategies were used. 
\emph{Random seed latents:} sampled directly from the latent distributions, incurring no additional cost. 
\emph{Filtered seed latents:} obtained by first generating 100 backbones from the base model, ranking them by RMSE, and selecting the top $K=24$ latents as seeds. None of these passed the $T=2.0$ threshold, but they provided a stronger starting point than random seed latents. The extra cost relative to random seed latents is generating and evaluating the pipeline 100 times.

For reference, the top row of Figure~\ref{fig:protein_optimisation} provides qualitative context for the C$_\alpha$–frame RMSE metric. Errors around 5\,\AA\ reflect substantial global structural discrepancies, while RMSE values below 2\,\AA\ correspond to largely overlapping backbone geometries, justifying the recovery threshold adopted by \citet{watson2023rfdiffusion} and used throughout our experiments.

\paragraph{PAE results.}
While we optimise and filter on RMSE only, we report the predicted aligned error (PAE) of the best designs across all conditions, allowing comparison against the full recoverability criterion of \citet{watson2023rfdiffusion} (which also requires PAE $<5$). Table~\ref{tab:rfdiff_pae} reports best-per-run RMSE and PAE, aggregated across $10$ runs of $200$ function evaluations.

\methodname{} substantially improves both metrics over random sampling in $\mathcal{Z}$. Median best-per-run PAE drops from $9.07$ for random sampling to $6.34$ with random seed latents and $6.84$ with filtered seed latents. RMSE drops from $2.33$ to $1.19$ and $1.08$ respectively. Filtered seed latents tighten the run-to-run variance considerably: across $10$ runs, the 5th--95th percentile range of best RMSE is $[0.89, 2.82]$ with random seed latents versus $[0.90, 1.81]$ with filtered seed latents, indicating that informative initialisation produces consistently good designs rather than merely a better average.

However, no run achieves the $\text{PAE} < 5$ threshold across any condition, so none of the designs satisfy the full recoverability criterion. Notably, filtered seed latents achieve worse median PAE than random seed latents despite better median RMSE; this is consistent with the single-objective formulation pushing harder on RMSE at slight cost to PAE when given a stronger starting point. Multi-objective extensions to \methodname{}, which our framework readily supports, are a natural direction for addressing both criteria jointly.

\begin{table}[h]
\caption{Best per-run RMSE and PAE for each method on the $600$-residue \textsc{RFdiffusion} task, reported as median with 5th--95th percentile range across $10$ runs of $200$ evaluations. Lower is better for both metrics.}
\centering
\small
\begin{tabular}{lcc}
\toprule
Method & Best RMSE (\,\AA) & Best PAE \\
\midrule
Random in $\mathcal{Z}$ & $2.33$ {\footnotesize $[1.44, 3.82]$} & $9.07$ {\footnotesize $[6.92, 11.87]$} \\
CMA-ES, random seed latents & $1.19$ {\footnotesize $[0.89, 2.82]$} & $6.34$ {\footnotesize $[5.38, 10.83]$} \\
CMA-ES, filtered seed latents & $1.08$ {\footnotesize $[0.90, 1.81]$} & $6.84$ {\footnotesize $[5.43, 8.20]$} \\
\bottomrule
\end{tabular}
\label{tab:rfdiff_pae}
\end{table}

\section{Template Modelling Score (TM-score)}
\label{app:tm_score}

The Template Modelling score (TM-score) is a widely used measure of structural similarity between two protein backbones. Unlike RMSE, which is sensitive to local deviations and scales poorly with chain length, the TM-score is normalised to the length of the target protein and therefore more suitable for comparing proteins of different sizes \citep{zhang2004tmscore}.

Given a target structure of length $L$ and a comparison structure, the TM-score is defined as
\begin{equation}
    \mathrm{TM\!-\!score}
    \;=\;
    \max_{\text{alignments}}
    \frac{1}{L} \sum_{i=1}^{L}
    \frac{1}{1 + \bigl(\tfrac{d_i}{d_0(L)}\bigr)^2},
    \label{eq:tm_score}
\end{equation}
where $d_i$ is the distance between the $i$th pair of aligned C$_\alpha$ atoms under a given alignment, and $d_0(L) = 1.24\sqrt[3]{L-15} - 1.8$ is a
normalisation factor that accounts for protein length. The score lies in $[0,1]$, with higher values indicating greater structural similarity.

As a rule of thumb, $\text{TM-score} > 0.5$ indicates that two structures share the same fold, while $\text{TM-score} < 0.17$ corresponds to similarity expected by chance. 

In our case, all generations are of equal length, so RMSE remains valid; however, using TM-score not only allows us to apply established interpretative thresholds, but also lets us follow \citet{watson2023rfdiffusion} in treating two designs as \emph{non-diverse} if their pairwise TM-score exceeds $0.6$.

\paragraph{Diversity counting.}
To compute the number of diverse generations reported in Section~\ref{sec:protein_opt}, we apply the following greedy procedure:  
1. Sort generated proteins by reconstruction accuracy (lowest RMSE first).  
2. Initialise the diverse set with the best structure.  
3. For each subsequent protein, compute its TM-score against all members of the current diverse set.  
4. Add it to the diverse set if its TM-score is $\leq 0.6$ with respect to all previously accepted members; otherwise, discard it.  

This ensures that each counted generation is both accurate (passes the RMSE threshold) and structurally distinct under TM-score. \emph{Note:} because a newly generated protein may achieve lower RMSE than existing members of the diverse set while simultaneously being non-diverse with respect to several of them, the overall count of diverse structures can decrease across iterations.

\section{Audio/Video}
\label{appendix:modalities}

Figure~\ref{fig:audio/video} illustrates that the surrogate latent space construction applies beyond images, and can be used with generative models operating on high-dimensional continuous representations across diverse modalities. In each case, we form a low-dimensional surrogate space using a small set of seed latents corresponding to a deterministic inversion of targets from the data distribution, and evaluate a dense 2D grid within this space to visualise the resulting generations. For all examples we observe smooth variation of the output across the surrogate coordinates, indicating that the surrogate chart preserves locality even in extremely high-dimensional latent spaces. These qualitative results support the claim that surrogate latent spaces provide well-behaved, low-dimensional manifolds suitable for optimisation across a wide range of generative models and data modalities.

\newpage
\section{Image Feature Composition Benchmark grammar}
\label{appendix:image_comp_benchmark}

\begin{table}[h!]
\centering
\caption{
\textbf{Image Feature Composition Benchmark} vehicle grammar. 
Each combination of attributes, type, and environment strings are uniformly and independently sampled to form target prompts of the form: `A <attributes> <type> <environment>' for the generation prompt `A vehicle'. 
}
\scalebox{0.46}{
\begin{tabular}{p{0.28\textwidth} p{0.34\textwidth} p{0.32\textwidth}}
\toprule
\textbf{Attributes} & \textbf{Types} & \textbf{Environments} \\
\midrule
"red, shiny" & "sedan car" & "on a mountain road" \\
"blue, glossy" & "hatchback car" & "by the ocean beach" \\
"green, matte" & "coupe car" & "in a desert with sand dunes" \\
"black, reflective" & "convertible car" & "through a forest trail" \\
"white, clean" & "station wagon car" & "on a snowy mountain peak" \\
"silver, metallic" & "SUV" & "beside a flowing river" \\
"gold, polished" & "pickup truck" & "in a dense jungle" \\
"yellow, bright" & "minivan" & "on a frozen lake" \\
"orange, vibrant" & "cargo van" & "next to a waterfall" \\
"purple, glossy" & "limousine" & "in a grassy meadow" \\
"pink, pastel" & "sports car" & "through a rocky canyon" \\
"brown, rustic" & "microcar" & "in heavy rainstorm" \\
"gray, matte" & "standard motorcycle" & "on a wide highway" \\
"beige, plain" & "motor scooter" & "near an active volcano" \\
"teal, glossy" & "moped" & "under the northern lights" \\
"navy blue, shiny" & "dirt bike motorcycle" & "in a futuristic city" \\
"maroon, matte" & "touring motorcycle" & "inside a highway tunnel" \\
"ivory, smooth" & "cruiser motorcycle" & "on a suspension bridge" \\
"bronze, metallic" & "all-terrain vehicle (ATV)" & "beside a tall lighthouse" \\
"copper, shiny" & "utility task vehicle (UTV)" & "on a sandy dune" \\
"chrome, reflective" & "monster truck" & "in a busy marketplace" \\
"pearl white, shimmering" & "golf cart" & "under cherry blossom trees" \\
"matte black, dull" & "go-kart" & "in front of a medieval castle" \\
"glossy white, polished" & "city bus" & "at an airport runway" \\
"emerald green, shiny" & "double-decker bus" & "on a racetrack" \\
"ruby red, glossy" & "school bus" & "inside a scrapyard" \\
"sapphire blue, shiny" & "electric trolleybus" & "on a battlefield" \\
"amber yellow, glowing" & "street tram" & "in an abandoned ghost town" \\
"charcoal gray, matte" & "light rail train" & "beside a farm barn" \\
"steel silver, brushed" & "monorail train" & "through vineyards" \\
"deep purple, glossy" & "subway train" & "on cobblestone streets" \\
"forest green, matte" & "passenger train" & "in a suburban neighborhood street" \\
"sky blue, bright" & "freight train" & "beside a skyscraper" \\
"sunset orange, glowing" & "high-speed train" & "inside a factory yard" \\
"lemon yellow, bright" & "armored personnel carrier (APC)" & "on a cliffside road" \\
"rose pink, soft" & "military tank" & "through misty hills" \\
"sand beige, dusty" & "bulldozer" & "in a crater" \\
"stone gray, rough" & "excavator" & "inside a dark cave" \\
"lava red, fiery" & "forklift truck" & "in an abandoned warehouse" \\
"ice blue, frosty" & "cement mixer truck" & "under a starry night sky" \\
"neon green, glowing" & "fire engine truck" & "beside a spaceport" \\
"neon pink, glowing" & "ambulance vehicle" & "at sunset on the horizon" \\
"pastel blue, soft" & "police patrol car" & "on a frozen tundra" \\
"pastel yellow, soft" & "tow truck" & "in thick fog" \\
"midnight black, glossy" & "garbage truck" & "through rice fields" \\
"frost white, icy" & "snowplow truck" & "beside a wind farm" \\
"mirror chrome, shiny" & "logging truck" & "under a rainbow" \\
"brushed aluminum, dull" & "farm tractor" & "near a medieval stone gate" \\
"glossy teal, shiny" & "combine harvester" & "on an icy highway" \\
"metallic purple, shiny" & "horse-drawn carriage" & "in a neon-lit street" \\
"bronze, weathered" & "canoe boat" & "beside a carnival fairground" \\
"flat black, matte" & "kayak boat" & "in a junkyard" \\
"desert tan, dusty" & "rowboat" & "through a wheat field" \\
"jungle green, camo" & "pedal boat" & "in a tropical rainforest" \\
"navy gray, military" & "sailboat" & "on a wooden boardwalk" \\
"rust red, corroded" & "luxury yacht" & "at a construction site" \\
"storm gray, rough" & "catamaran boat" & "on a winding mountain pass" \\
"bright yellow, shiny" & "inflatable dinghy" & "beside a glacier" \\
"glossy red, polished" & "fishing boat" & "on a cratered moon surface" \\
"flat white, plain" & "harbor tugboat" & "inside a space station" \\
"sparkling silver, glittery" & "passenger ferry" & "through an asteroid field" \\
"dull gray, industrial" & "speedboat" & "on the surface of Mars" \\
"deep green, glossy" & "jet ski watercraft" & "inside a lunar base" \\
"ocean blue, wavy" & "hovercraft vehicle" & "inside an aircraft hangar" \\
"fire orange, glowing" & "houseboat" & "on a rocket launch pad" \\
"sun gold, shiny" & "pontoon boat" & "at a desert oasis" \\
"candy apple red, glossy" & "container cargo ship" & "on a tropical island beach" \\
"storm gray, matte" & "general cargo ship" & "in a canyon riverbed" \\
"ice silver, frosty" & "oil tanker ship" & "on an offshore oil rig" \\
"jet black, shiny" & "cruise ship" & "on a dry salt flat" \\
"steel blue, metallic" & "battleship" & "inside a military base" \\
"military green, matte" & "aircraft carrier ship" & "in an amusement park" \\
"glossy maroon, shiny" & "military submarine" & "on a snowy city street" \\
"matte navy blue" & "destroyer warship" & "beside a frozen waterfall" \\
"carbon fiber pattern" & "frigate warship" & "at a cultural festival plaza" \\
"transparent, glassy" & "hot air balloon" & "inside an industrial plant" \\
"camouflage green, patterned" & "sailplane glider" & "on a dirt trail" \\
"chrome gold, shiny" & "hang glider" & "in a foggy swamp" \\
"metallic blue, glossy" & "paraglider" & "beside a mountain lake" \\
"dark gray, dull" & "airship blimp" & "on a coastal cliff road" \\
"vibrant purple, glowing" & "helicopter" & "in front of ancient ruins" \\
"fluorescent yellow, glowing" & "gyrocopter" & "beside a pyramid" \\
"sparkling white, glittery" & "small propeller aircraft" & "inside an old temple" \\
"pearl blue, shimmering" & "seaplane" & "through rolling hills" \\
"shiny copper, metallic" & "amphibious aircraft" & "in a field of flowers" \\
"bronze, antique" & "commercial airliner jet" & "beside a farmstead" \\
"olive green, matte" & "private jet plane" & "at a roadside gas station" \\
"bright teal, glowing" & "supersonic passenger jet" & "inside a futuristic arena" \\
"plain beige, flat" & "fighter jet aircraft" & "inside a spaceship hangar" \\
"polished black, shiny" & "bomber aircraft" & "on a collapsing bridge" \\
"bright gold, reflective" & "stealth aircraft" & "on a volcanic lava plain" \\
"storm blue, dark" & "quadcopter drone" & "beside a crystal cave" \\
"camo brown, patterned" & "cargo plane" & "on a wooden pier" \\
"stealth gray, matte" & "snowmobile vehicle" & "inside a mining colony" \\
"metallic orange, glossy" & "mountain cable car" & "on a tall city rooftop" \\
"diamond white, shiny" & "space shuttle orbiter" & "through a canyon pass" \\
"rust brown, corroded" & "spaceplane vehicle" & "inside a virtual reality world" \\
"emerald green, glossy" & "rocketship" & "on an alien desert planet" \\
"jet silver, reflective" & "lunar exploration rover" & "inside a submarine base" \\
"space black, glossy" & "mars exploration rover" & "inside an underground bunker" \\
\bottomrule
\end{tabular}
}
\label{table:ifcb_grammar}
\end{table}

Table~\ref{table:ifcb_grammar} lists the possible attributes, vehicle types, and environment strings used for the prompt grammar used in the experiments reported in Section~\ref{sec:experiments} (with details in Section~\ref{appendix:image_opt_details}) and Section~\ref{appendix:weight_chart_comparison}. 

\section{Experiments compute resources}
\label{appendix:compute}

This appendix summarises the compute used by each experiment in the paper. The dominant cost in every experiment is generation and evaluation of the underlying generative model; \methodname{} itself adds negligible overhead, since the surrogate chart $\phi$ is evaluated in closed form and the optimisers (BO, CMA-ES, random search) operate on $\mathcal U \subseteq \mathbb R^{K-1}$ at low dimensionality. We report the type of GPU used, the wall-clock time of a single optimisation run, the number of function evaluations per run, the number of runs, and the resulting total wall-clock. The reported figures do not include preliminary or failed experiments; an estimate of that overhead is given at the end of this appendix.

\begin{table*}[h!]
\centering
\caption{
\textbf{Compute resources used by the experiments reported in this paper.} `Evals/run' counts oracle (objective) evaluations, each of which corresponds to one full generation through the relevant model pipeline.
}
\resizebox{\textwidth}{!}{%
\begin{tabular}{@{}lllrrrl@{}}
\toprule
Experiment & Model / pipeline & Hardware & Evals/run & \# runs & Per-run wall-clock & Total wall-clock \\
\midrule
\multicolumn{7}{l}{\textit{Figure~\ref{fig:cosim_vs_objectives} -- Latent cosine similarity vs.\ objective similarity (motivating analysis)}} \\
Pairs analysis, images & Flux~\citep{flux2024} + ImageReward + PickScore   & RTX~5090      & 1000 generations per (seed, $K$); 5 seeds $\times$ 4 values of $K$ $\in$ \{3, 10, 30, 90\} & 1 & $\sim 75$ h & $\sim 75$ h \\
Pairs analysis, proteins & \textsc{RFdiffusion} + RMSE + TM-score          & RTX~4090      & 450 & 1 & $\sim 40$ h & $\sim 40$ h \\
\midrule
\multicolumn{7}{l}{\textit{Section~\ref{sec:exp_dense_with_solutions} / Appendix~\ref{appendix:good_seeds_better_solutions_details} -- ImageReward benchmark on SD~1.5 (30 reps per prompt $\times$ 3 prompts $=$ 90 total runs per row)}} \\
Standard model sampling          & SD~1.5                & RTX~5090 & 100  & 90 & $\sim 17$ m & $\sim 25$ h \\
Best-1-of-2, 100 times (1/2 eff.)& SD~1.5                & RTX~5090 & 200  & 90 & $\sim 33$ m & $\sim 50$ h \\
Best-100-of-200 (1/2 eff.)       & SD~1.5                & RTX~5090 & 200  & 90 & $\sim 33$ m & $\sim 50$ h \\
Grid in $\mathcal U^{1,3,5}$ (1/2 eff.) & SD~1.5         & RTX~5090 & 200  & 90 & $\sim 33$ m & $\sim 50$ h \\
Best-1-of-6, 100 times (1/6 eff.)& SD~1.5                & RTX~5090 & 600  & 90 & $\sim 1.7$ h & $\sim 150$ h \\
Best-100-of-600 (1/6 eff.)       & SD~1.5                & RTX~5090 & 600  & 90 & $\sim 1.7$ h & $\sim 150$ h \\
Grid in $\mathcal U^{1,3,5}$ (1/6 eff.) & SD~1.5         & RTX~5090 & 600  & 90 & $\sim 1.7$ h & $\sim 150$ h \\
\multicolumn{7}{l}{\hfill \textit{Aggregate wall-clock for the SD~1.5 image experiments above: $\sim 625$ RTX~5090~h.}} \\
\midrule
\multicolumn{7}{l}{\textit{Section~\ref{sec:exp_optimisation} / Appendix~\ref{appendix:image_opt_details} -- Image Feature Composition Benchmark on SD~2.1}} \\
BO / CMA-ES / RS in $\mathcal U^{2}$  ($K{=}3$ seeds) & SD~2.1 & GH200 & 100  & 30 & $\sim 0.1$ h & $\sim 3$ h \\
BO / CMA-ES / RS in $\mathcal U^{8}$  ($K{=}9$ seeds) & SD~2.1 & GH200 & 100  & 30 & $\sim 0.1$ h & $\sim 3$ h \\
TuRBO / CMA-ES / RS in $\mathcal U^{89}$ ($K{=}90$ seeds) & SD~2.1 & GH200 & 1000 & 30 & $\sim 1$ h & $\sim 30$ h \\
Random / CMA-ES in $\mathcal Z$ baselines & SD~2.1   & GH200 & 100--1000 & 20 & $\sim 1.1$ h & $\sim 22$ h \\
\multicolumn{7}{l}{\hfill \textit{Aggregate wall-clock for the SD 2.1 image experiments above: $\sim 60$ GH200~h.}} \\
\midrule
\multicolumn{7}{l}{\textit{Section~\ref{sec:experiments} -- \textsc{Boltz-2} guidance under matched oracle-call budgets}} \\
\methodname{} + BO          & \textsc{Boltz-2} (PF-ODE) & H100 & 6 budgets up to $N{=}1000$ & 5 per budget & $\sim 0.1$ h & $\sim 3$ h \\
      &     &    & $N{=}20$ & 5 & $2.7$ m & $13$ m \\
      &     &    & $N{=}50$ & 5 & $2.8$ m & $14$ m \\
      &     &    & $N{=}100$ & 5 & $3.1$ m & $16$ m \\
      &     &    & $N{=}200$ & 5 & $2.9$ m & $15$ m \\
      &     &    & $N{=}500$ & 5 & $3.4$ m & $17$ m \\
      &     &    & $N{=}1000$ & 5 & $18$ m & $1$ h $31$ m \\
Best-of-$N$                  & \textsc{Boltz-2}          & H100 & 6 budgets up to $N{=}1000$ & 5 per budget & $\sim 0.1$ h & $\sim 3$ h \\
      &     &    & $N{=}20$ & 5 & $2.7$ m & $13$ m \\
      &     &    & $N{=}50$ & 5 & $2.8$ m & $14$ m \\
      &     &    & $N{=}100$ & 5 & $3.1$ m & $16$ m \\
      &     &    & $N{=}200$ & 5 & $2.9$ m & $15$ m \\
      &     &    & $N{=}500$ & 5 & $3.4$ m & $17$ m \\
      &     &    & $N{=}1000$ & 5 & $18$ m & $1$ h $31$ m \\
FK-steering~\citep{singhal2025general} & \textsc{Boltz-2} & H100 & 6 budgets up to $N{=}1000$ & 10 per budget & $\sim 0.1$ h & $\sim 6$ h \\
      &     &    & $N{=}20$ & 10 & $3.3$ m & $33$ m \\
      &     &    & $N{=}50$ & 10 & $3.8$ m & $38$ m \\
      &     &    & $N{=}100$ & 10 & $4.1$ m & $41$ m \\
      &     &    & $N{=}200$ & 10 & $5$ m & $50$ m \\
      &     &    & $N{=}500$ & 10 & $6.9$ m & $1$ h $09$ m \\
      &     &    & $N{=}1000$ & 10 & $11$ m & $1$ h $50$ m \\

DPO~\citep{rafailov2023direct,wallace2023diffusion} online & \textsc{Boltz-2} (fine-tuned) & \texttt{H100} & 6 budgets up to $N{=}1000$ &  &  &  \\
      &     &    & $N{=}20$ & 4 & $4$ m & $16$ m \\
      &     &    & $N{=}50$ & 4 & $6$ m & $24$ m \\
      &     &    & $N{=}100$ & 5 & $11$ m & $55$ m \\
      &     &    & $N{=}200$ & 3 & $17$ m & $51$ m \\
      &     &    & $N{=}500$ & 5 & $44$ m & $3$ h $40$ m \\
      &     &    & $N{=}1000$ & 4 & $1$ h $31$ m & $6$ h $04$ m \\
DPO~\citep{rafailov2023direct,wallace2023diffusion} offline& \textsc{Boltz-2} (fine-tuned) & \texttt{H100} & 6 budgets up to $N{=}1000$ &  &  &  \\
      &     &    & $N{=}20$ & 4 & $4$ m & $16$ m \\
      &     &    & $N{=}50$ & 5 & $7$ m & $35$ m \\
      &     &    & $N{=}100$ & 4 & $18$ m & $1$ h $12$ m \\
      &     &    & $N{=}200$ & 5 & $56$ m & $4$ h $40$ m \\
      &     &    & $N{=}500$ & 4 & $5$ h $40$ m & $22$ h $40$ m \\
      &     &    & $N{=}1000$ & 3 & $21$ h $57$ m & $65$ h $51$ m \\
\midrule
\multicolumn{7}{l}{\textit{Section~\ref{sec:protein_opt} / Appendix~\ref{appendix:rfdiff} -- \textsc{RFdiffusion} pipeline at $N{=}600$ residues}} \\
Random sampling in $\mathcal Z$ (baseline) & \textsc{RFdiffusion} + \textsc{ProteinMPNN} + \textsc{AlphaFold2} & RTX~4090 & 200       & 10 & $\sim 33$~h & $\sim 330$~h \\
CMA-ES in $\mathcal U^{K-1}$, random seed latents ($K{=}24$) & \textsc{RFdiffusion} + \textsc{ProteinMPNN} + \textsc{AlphaFold2} & RTX~4090 & 200       & 10 & $\sim 33$~h & $\sim 330$~h \\
CMA-ES in $\mathcal U^{K-1}$, filtered seed latents ($K{=}24$, +100 init.) & \textsc{RFdiffusion} + \textsc{ProteinMPNN} + \textsc{AlphaFold2} & RTX~4090 & 100 + 200 & 10 & $\sim 50$~h & $\sim 500$~h \\
\multicolumn{7}{l}{\hfill \textit{Aggregate wall-clock for the \textsc{RFdiffusion} experiments above: $\sim 1000$ RTX~4090~h.}} \\
\midrule
\multicolumn{7}{l}{\textit{Appendix~\ref{appendix:modalities} -- Cross-modality grids (qualitative)}} \\
$20\times20$ slice of 7D $\mathcal U$ grid      & StableAudio2.0~\citep{evans2025stable}      & V100 & 400           & 1 & $\sim 1.1$ h & $\sim 1.1$ h \\
$7\times7$ video grid in $\mathcal U^{2}$       & HunyuanVideo~\citep{kong2024hunyuanvideo}   & V100 & 49            & 1 & $\sim 10$ h & $\sim 10$ h \\
2D protein grid in $\mathcal U^{2}$ ($K{=}3$)   & \textsc{RFdiffusion} pipeline               & RTX 4090 & 400 & 1 & $\sim 66$ h & $\sim 132$ h \\
\midrule
\multicolumn{7}{l}{\textit{Appendix~\ref{appendix:empirical_stationarity} -- Stationarity assessment of $\phi_w$}} \\
Chart-similarity sampling & \textit{(no generative model; closed-form $\phi_w$)} & CPU & $10^{6}$ pairs & 12 (2 charts $\times$ 3 dims $\times$ 2 schemes) & $\sim 20$ s & $\sim 4$ m \\
\midrule
\multicolumn{7}{l}{\textit{Reference baselines from~\citet{denker2025iterative} (reproduced from their paper)}} \\
DPOK~\citep{fan2023reinforcement}                     & SD~1.5 (fine-tuned)        & A100      & --- & --- & 28~h  & 28~A100~h \\
Adjoint Matching~\citep{domingo2024adjoint}           & SD~1.5 (fine-tuned)        & A100      & --- & --- & 4~h   & 4~A100~h \\
Importance Fine-tuning~\citep{denker2025iterative}    & SD~1.5 (fine-tuned)        & RTX~4090  & --- & --- & 7~h   & 7~RTX~4090~h \\
\bottomrule
\end{tabular}
}
\label{tab:compute}
\end{table*}


\paragraph{Total project compute, including preliminary experiments.}
The numbers in Table~\ref{tab:compute} cover only the experiments reported in the paper. The full research project additionally required compute for preliminary explorations -- including alternative weight charts $\phi_w$, ablations of seed-selection strategies, alternative optimisers, hyperparameter sweeps, and exploratory generations on image, audio, video, and protein modalities -- that did not make it into the final paper. We estimate this additional overhead at approximately 50\% of the wall-clock figures reported in Table~\ref{tab:compute}, distributed across the same hardware mix.

\newcommand{\rowlabel}[1]{%
    \adjustbox{valign=c}{%
        \makebox[0.28\linewidth][r]{#1}%
    }%
}
\begin{figure}[t]
    \centering
    \setlength{\tabcolsep}{4pt}
    \renewcommand{\arraystretch}{1.0}
    \setlength{\extrarowheight}{2pt}

    \newcommand{\imgcell}[2]{%
        \adjustbox{valign=m}{%
            \begin{tabular}{@{}c@{}}
                $#1$ \\[-1pt]
                \includegraphics[width=0.140\linewidth]{#2}
            \end{tabular}%
        }%
    }

    \begin{tabular}{c c c c}
        &
        $\mathcal{U} = [0, 1]^{2}$
        &
        $\mathcal{U} = [0, 1]^{9}$
        &
        $\mathcal{U} = [0, 1]^{89}$
        \\[2pt]

        \rowlabel{$\bm{z} \sim p(\bm{z})$}
        &
        \imgcell{21.1}{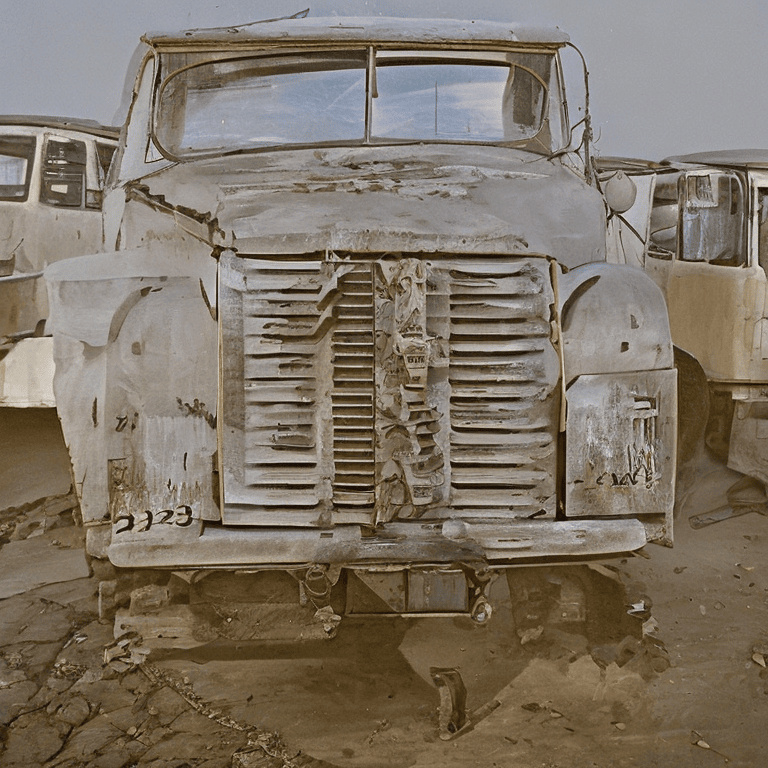}
        &
        \imgcell{21.1}{compressed_figures/image_optimisation/best_images/num_seeds_per_feature=1/target=1/score=21.1146__method=Random_in_Z_baseline.png}
        &
        \imgcell{21.4}{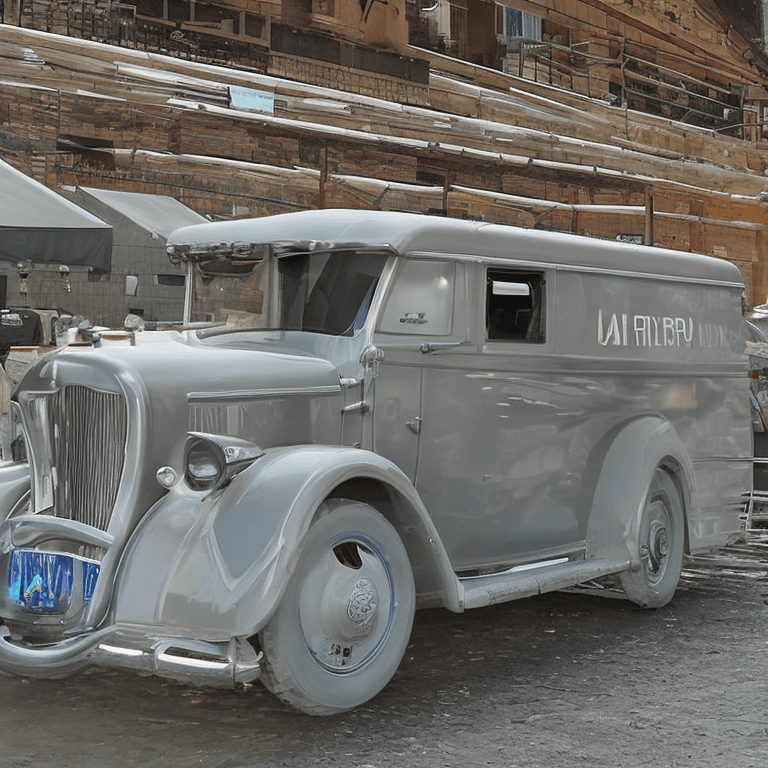}
        \\[2pt]

        \rowlabel{CMA-ES in $\mathcal{Z}$}
        &
        \imgcell{19.3}{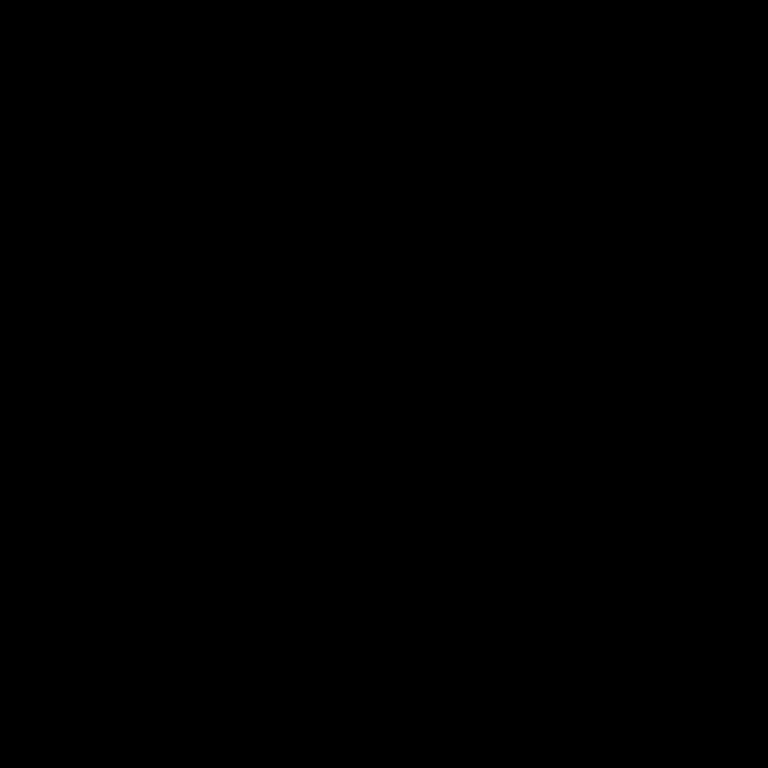}
        &
        \imgcell{19.3}{compressed_figures/image_optimisation/best_images/num_seeds_per_feature=30/target=1/score=19.2939__method=CMA-ES_in_Z.png}
        &
        \imgcell{19.3}{compressed_figures/image_optimisation/best_images/num_seeds_per_feature=30/target=1/score=19.2939__method=CMA-ES_in_Z.png}
        \\[2pt]

        \rowlabel{REMBO random seed latents (BO)}
        &
        \imgcell{19.9}{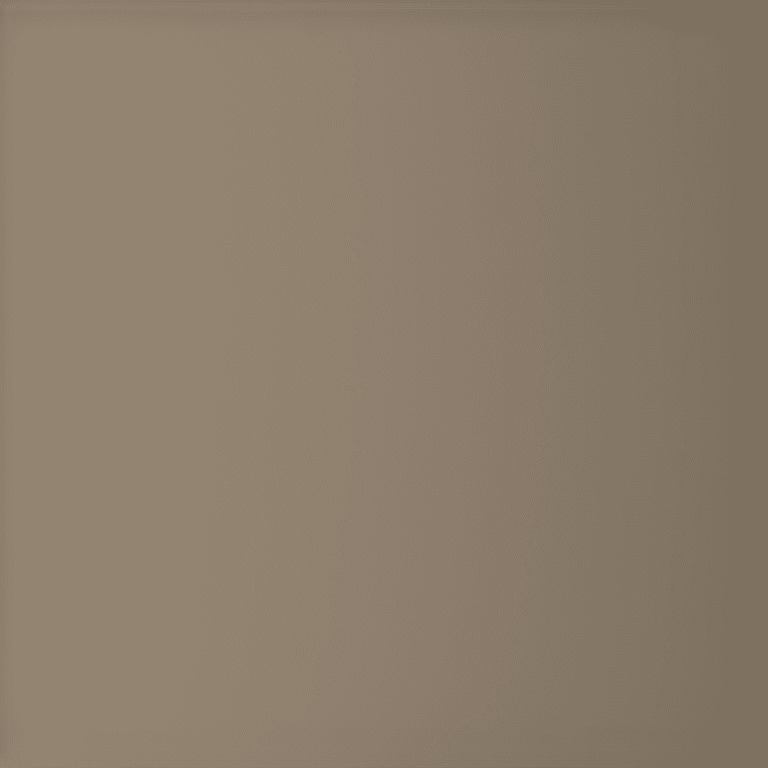}
        &
        \imgcell{15.8}{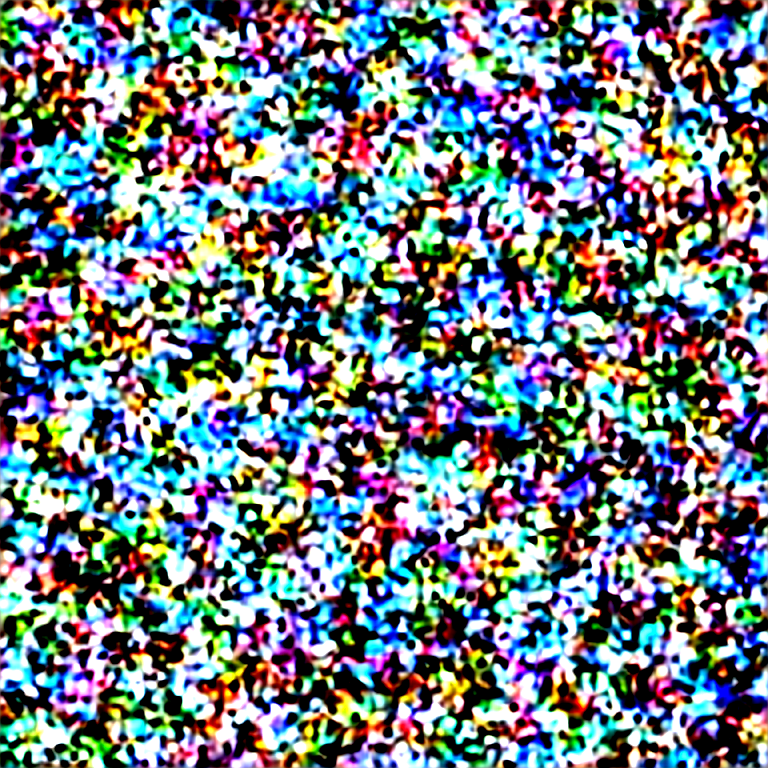}
        &
        \imgcell{16.8}{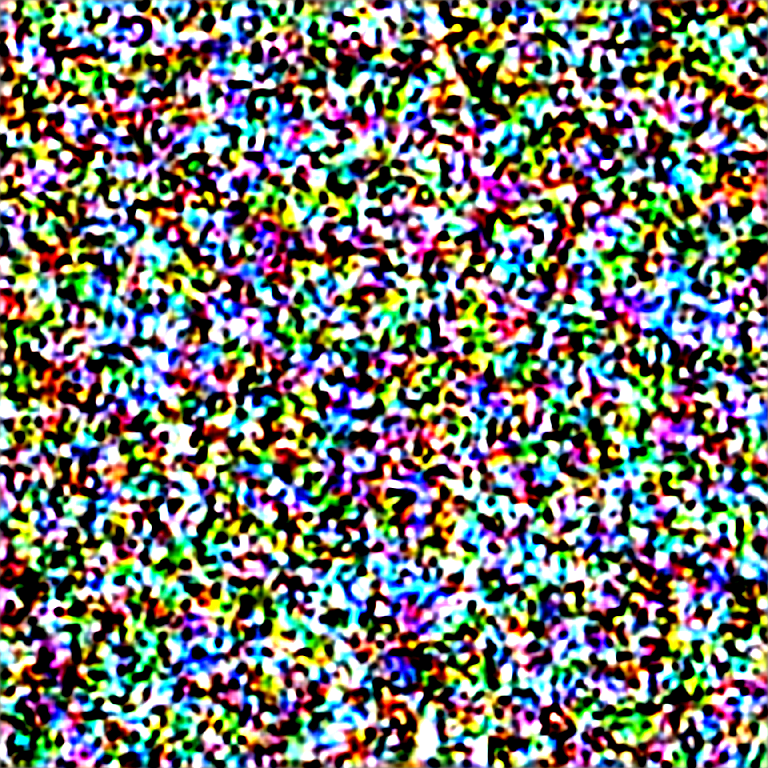}
        \\[2pt]

        \rowlabel{REMBO filtered seed latents (BO)}
        &
        \imgcell{19.7}{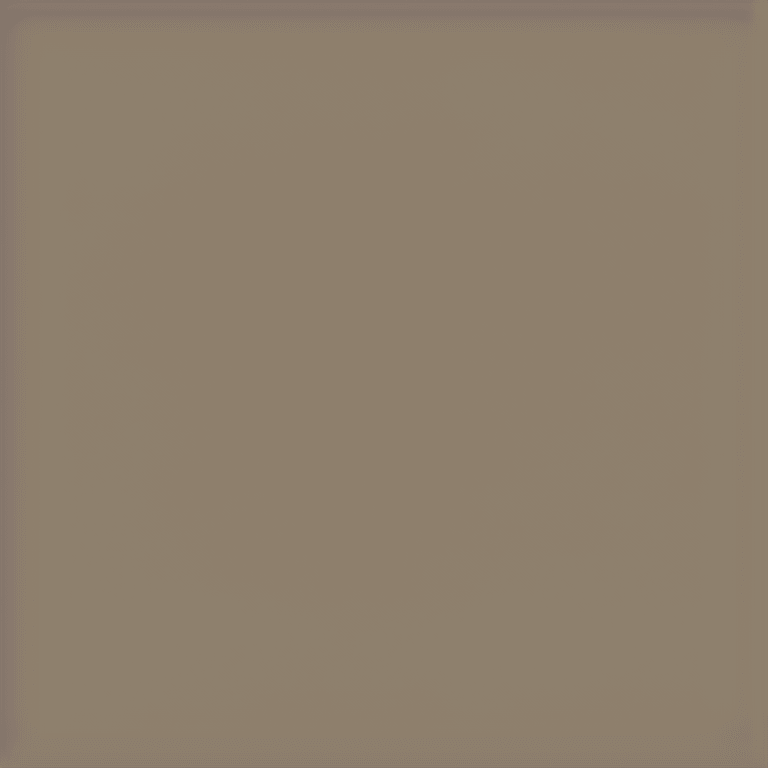}
        &
        \imgcell{16.0}{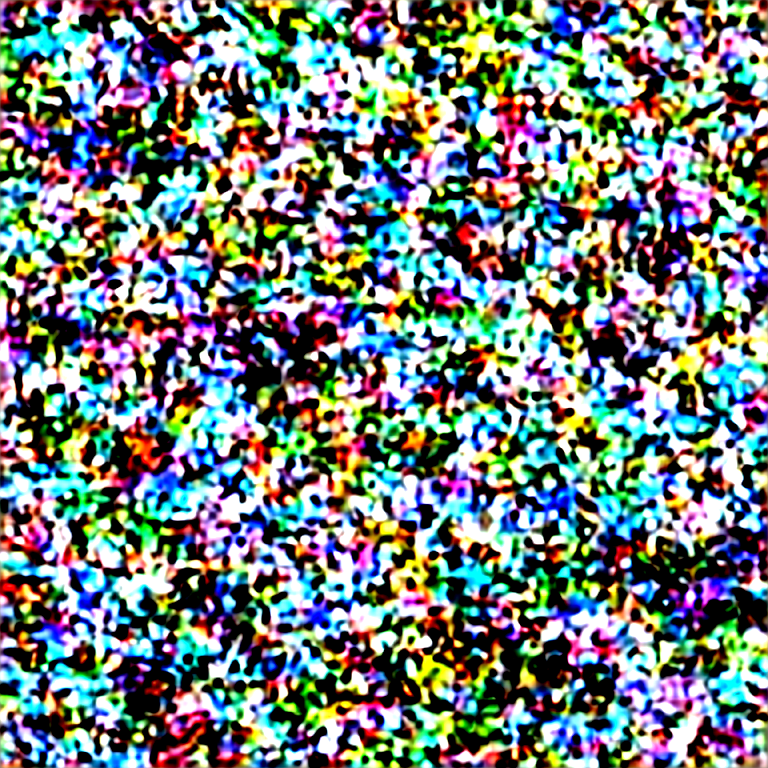}
        &
        \imgcell{16.9}{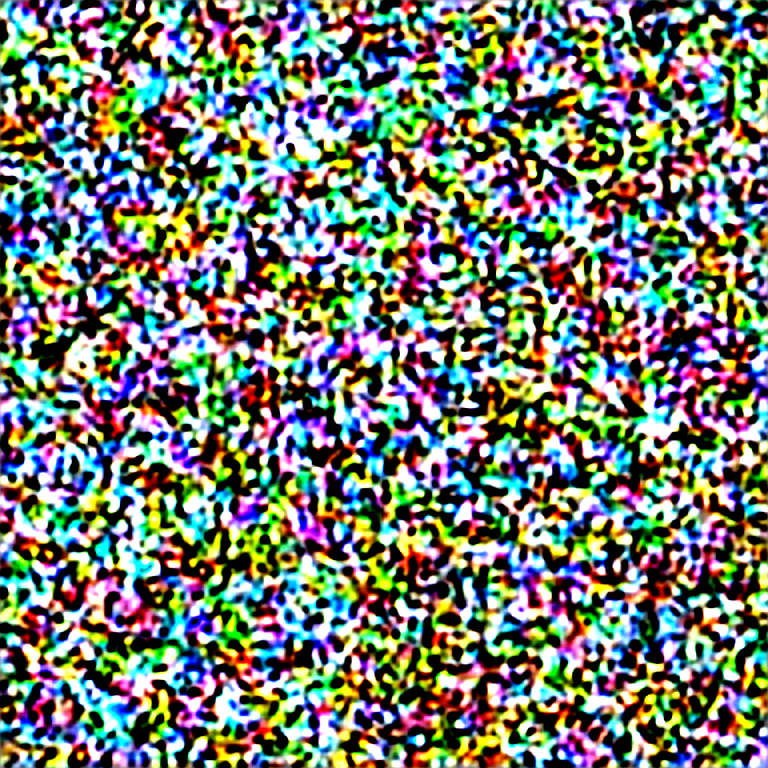}
        \\[2pt]

        \rowlabel{LOL random seed latents (BO)}
        &
        \imgcell{21.4}{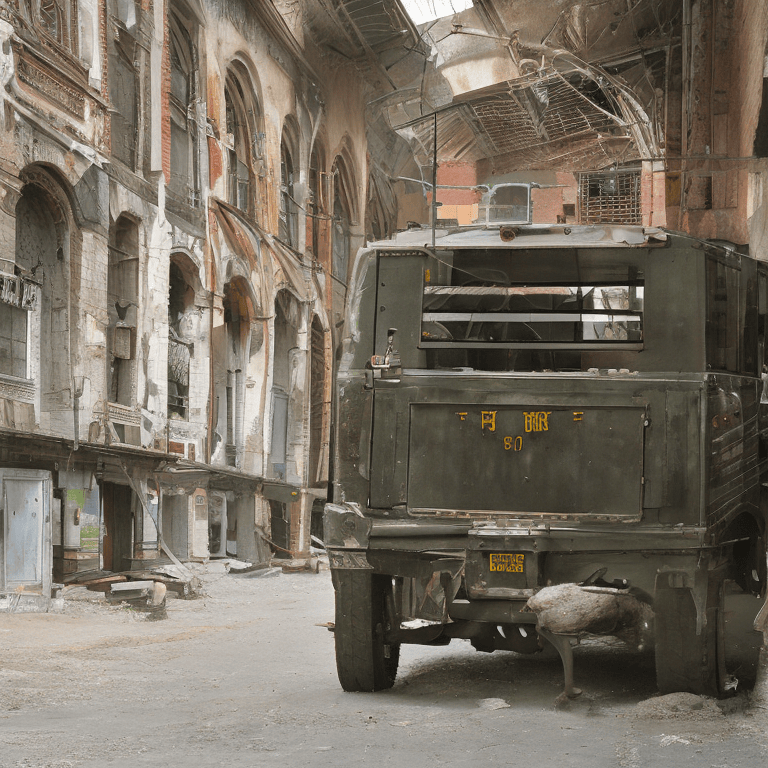}
        &
        \imgcell{21.1}{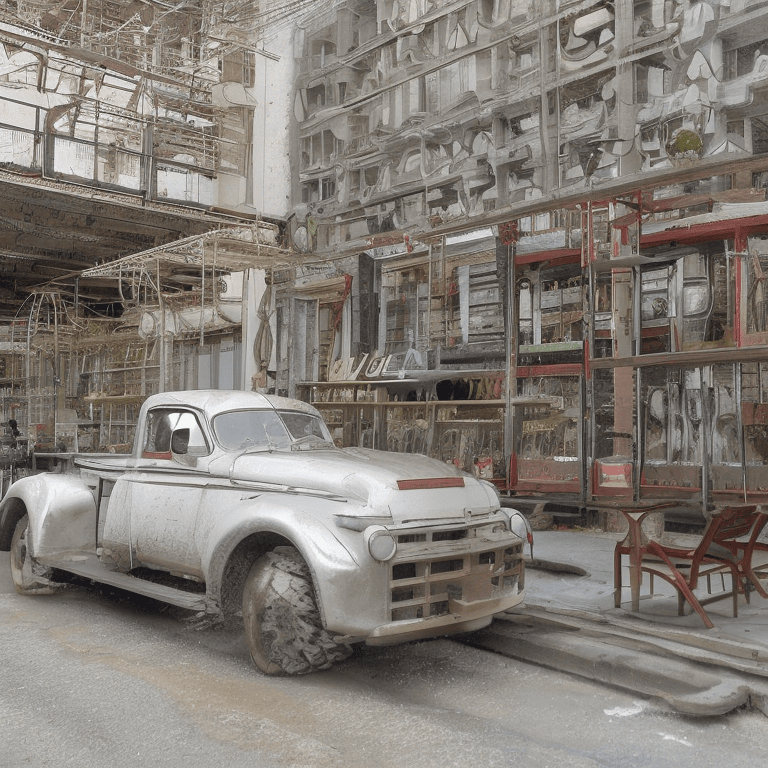}
        &
        \imgcell{22.8}{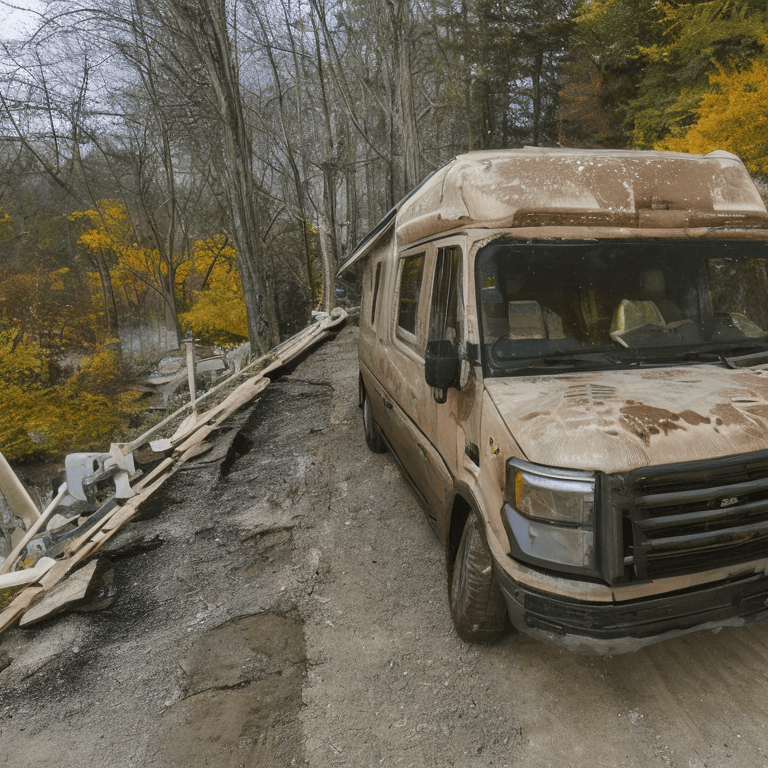}
        \\[2pt]

        \rowlabel{LOL filtered seed latents (BO)}
        &
        \imgcell{20.6}{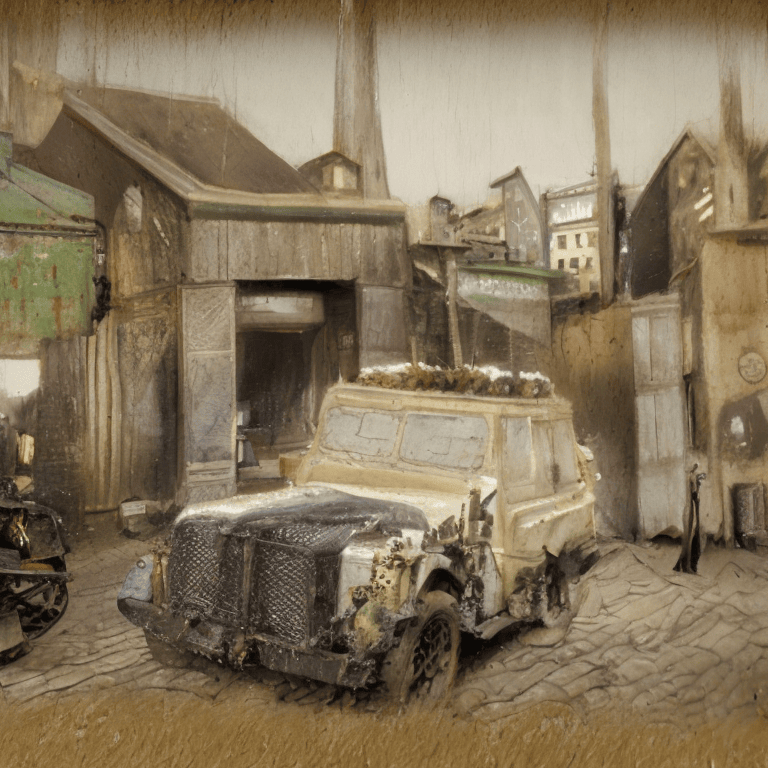}
        &
        \imgcell{21.5}{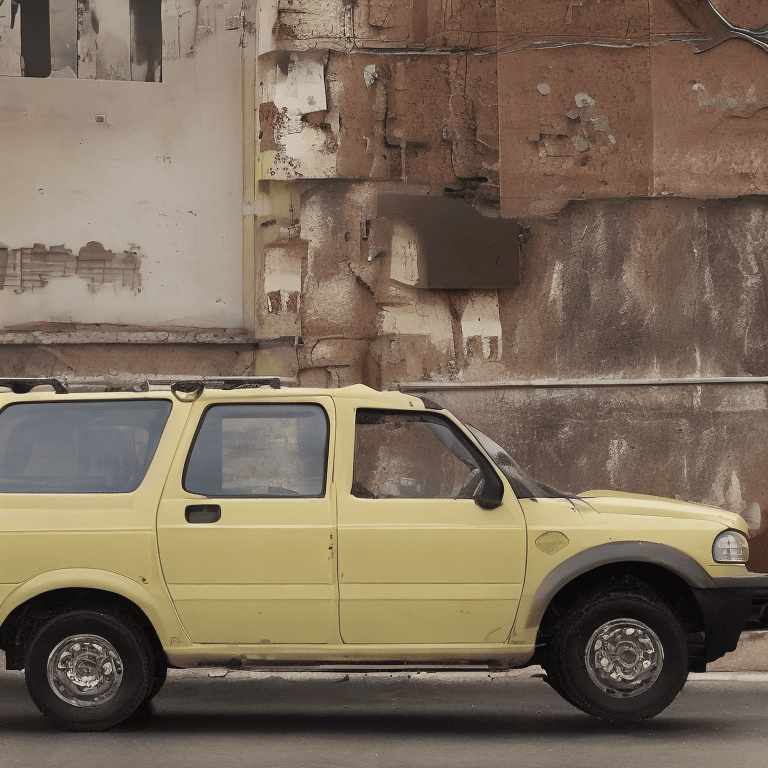}
        &
        \imgcell{23.0}{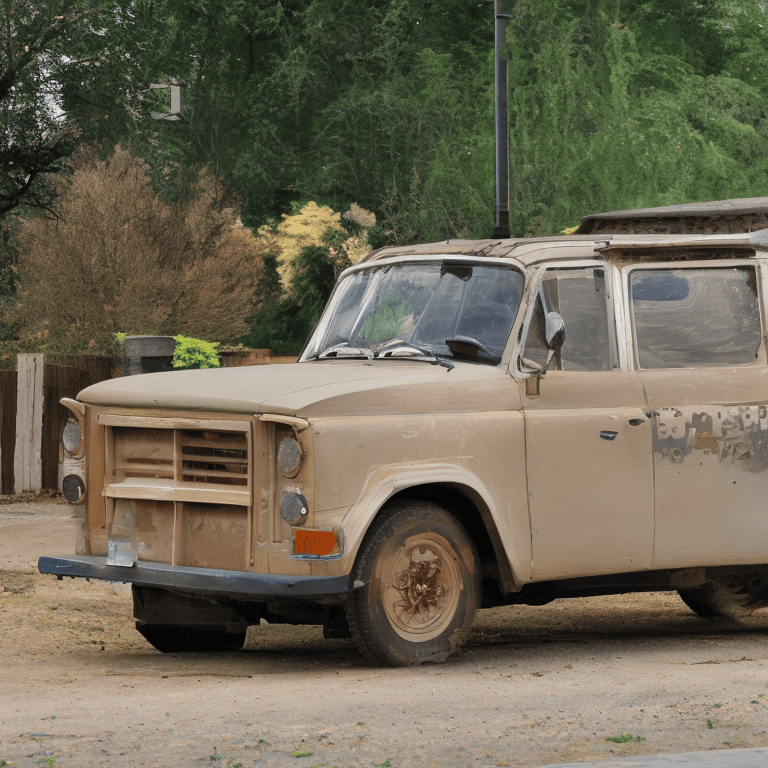}
        \\[2pt]

        \rowlabel{$O^3$ random seed latents (BO)}
        &
        \imgcell{21.1}{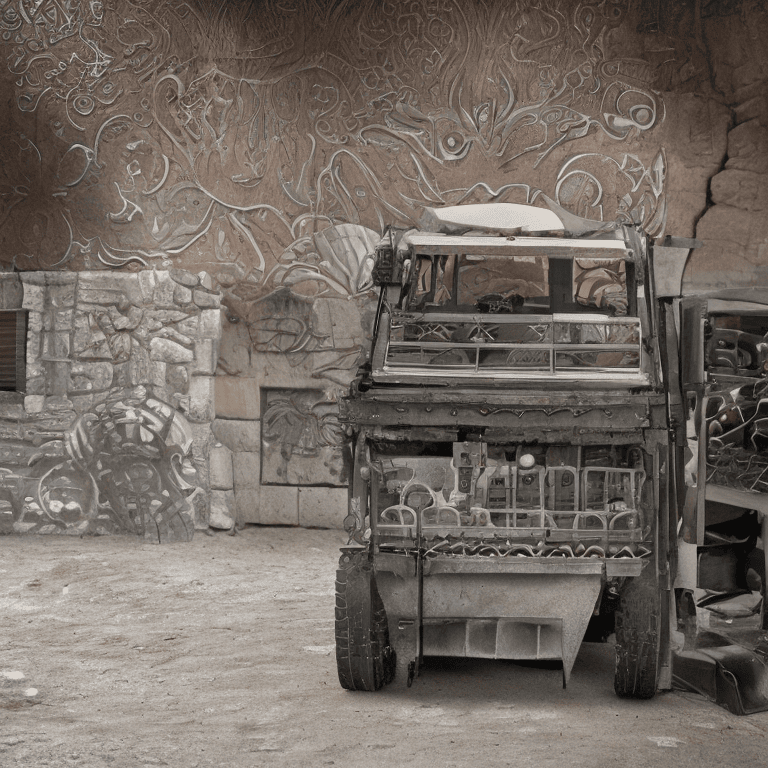}
        &
        \imgcell{21.3}{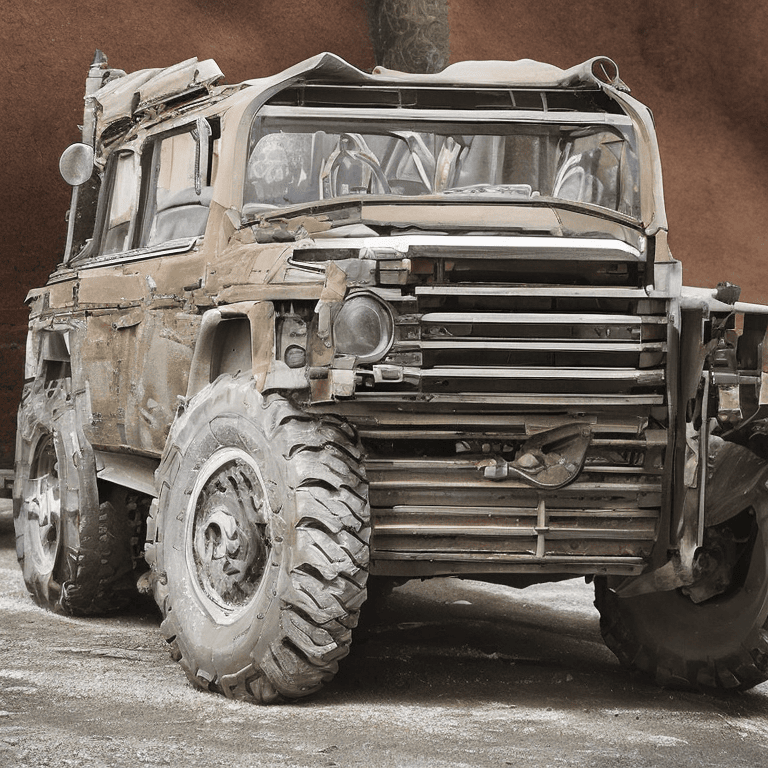}
        &
        \imgcell{22.8}{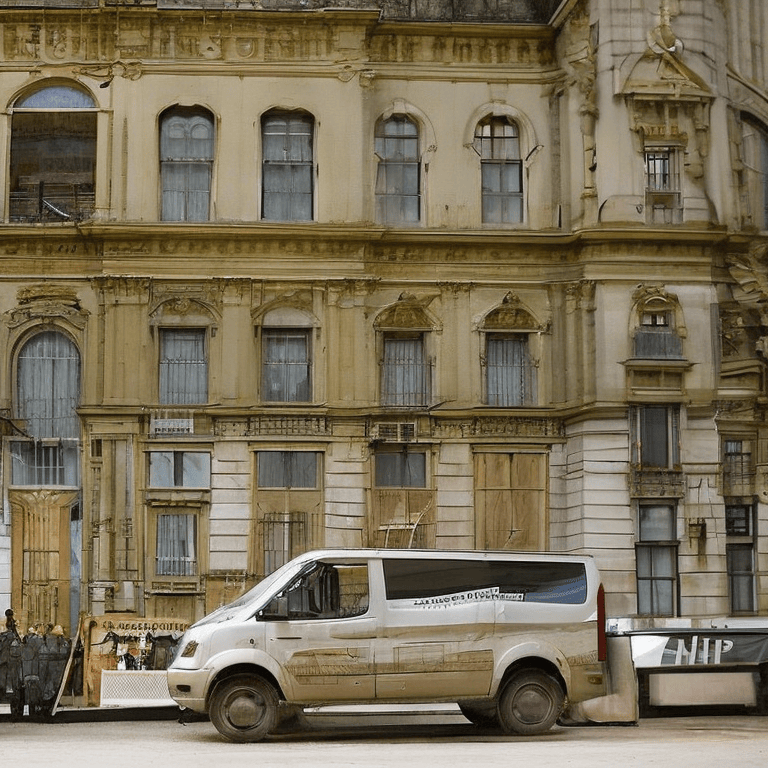}
        \\[2pt]

        \rowlabel{$O^3$ filtered seed latents (BO)}
        &
        \imgcell{20.2}{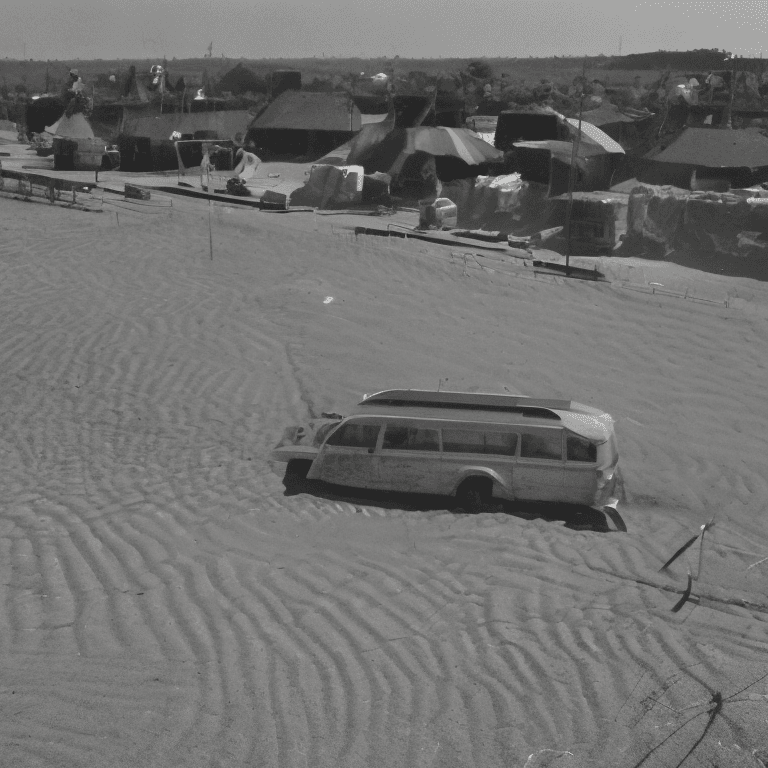}
        &
        \imgcell{21.9}{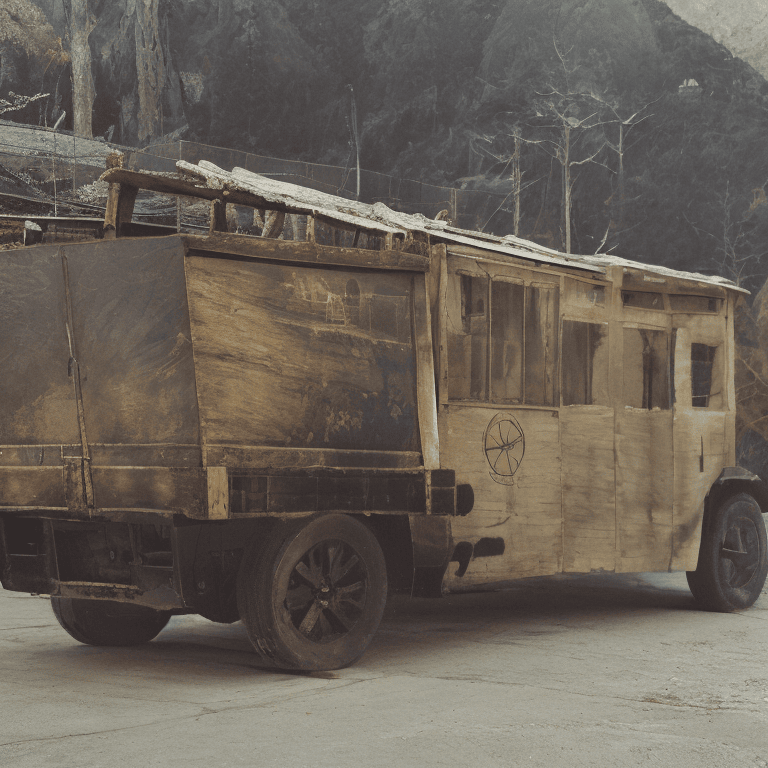}
        &
        \imgcell{23.4}{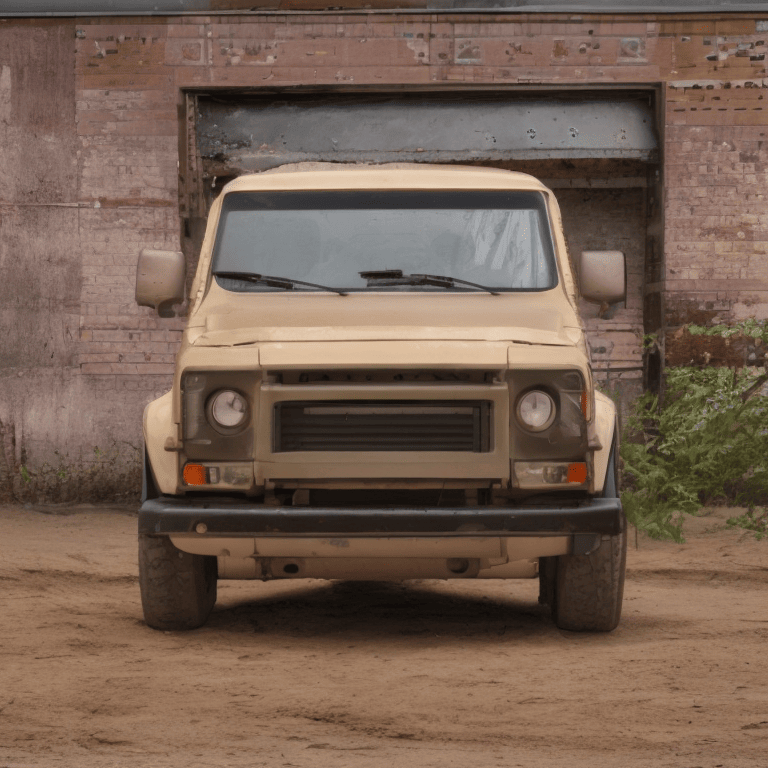}
        \\[2pt]
    \end{tabular}

    \caption{
    \textbf{Best generation found by each search method on the target prompt ``A sand beige, dusty cargo van at a construction site''}, across surrogate-space dimensionalities (columns) and search methods (rows); the number above each image is its PickScore (higher is better).
    }
\label{fig:image_optimisation_best_images}
\end{figure}

\end{document}